\documentclass[reqno]{amsart}

\usepackage[utf8]{inputenc} 
\usepackage[T1]{fontenc}    
\usepackage{url}            
\usepackage{booktabs}       
\usepackage{caption}
\usepackage{amsfonts}       
\usepackage{nicefrac}       
\usepackage{microtype}      
\usepackage[]{xcolor}         

\usepackage{amsmath,amssymb,amsfonts,mathrsfs}
\usepackage{nicefrac}
\usepackage{algorithmic}

\usepackage{algorithm}
\usepackage{graphicx}

\usepackage{multirow}

\usepackage{wrapfig}

\usepackage{booktabs}

\usepackage{tikz}
\usetikzlibrary{positioning, arrows.meta, calc, matrix}

\usepackage{xcolor}

\definecolor{strblue}{HTML}{0F1ED2}
\definecolor{strred}{HTML}{E61E8C}

\usepackage{natbib}

\usepackage{comment}

\usepackage{aliascnt}
\usepackage{hyperref}
\usepackage{cleveref}

\theoremstyle{definition}

\newtheorem{theorem}{Theorem}[section]
\crefname{theorem}{Theorem}{Theorems}
\Crefname{theorem}{Theorem}{Theorems}

\newaliascnt{definition}{theorem}
\newtheorem{definition}[definition]{Definition}
\aliascntresetthe{definition}
\crefname{definition}{Definition}{Definitions}
\Crefname{definition}{Definition}{Definitions}

\newaliascnt{proposition}{theorem}
\newtheorem{proposition}[proposition]{Proposition}
\aliascntresetthe{proposition}
\crefname{proposition}{Proposition}{Propositions}
\Crefname{proposition}{Proposition}{Propositions}

\newaliascnt{lemma}{theorem}
\newtheorem{lemma}[lemma]{Lemma}
\aliascntresetthe{lemma}
\crefname{lemma}{Lemma}{Lemmas}
\Crefname{lemma}{Lemma}{Lemmas}

\newaliascnt{corollary}{theorem}
\newtheorem{corollary}[corollary]{Corollary}
\aliascntresetthe{corollary}
\crefname{corollary}{Corollary}{Corollaries}
\Crefname{corollary}{Corollary}{Corollaries}

\newaliascnt{remark}{theorem}
\newtheorem{remark}[remark]{Remark}
\aliascntresetthe{remark}
\crefname{remark}{Remark}{Remarks}
\Crefname{remark}{Remark}{Remarks}

\newaliascnt{example}{theorem}
\newtheorem{example}[example]{Example}
\aliascntresetthe{example}
\crefname{example}{Example}{Examples}
\Crefname{example}{Example}{Examples}

\newaliascnt{assumption}{theorem}
\newtheorem{assumption}[assumption]{Assumption}
\aliascntresetthe{assumption}
\crefname{assumption}{Assumption}{Assumptions}
\Crefname{assumption}{Assumption}{Assumptions}

\numberwithin{equation}{section}

\crefname{equation}{Equation}{Equations}
\crefname{figure}{Figure}{Figures}
\crefname{table}{Table}{Tables}
\crefname{algorithm}{Algorithm}{Algorithms}

\crefname{section}{}{}
\creflabelformat{section}{\S #2#1#3}
\crefname{subsection}{}{}
\creflabelformat{subsection}{\S #2#1#3}
\crefname{appendix}{Appendix}{Appendixes}

\DeclareMathOperator*{\argmin}{arg\,min} 
\DeclareMathOperator*{\argmax}{arg\,max} 
\DeclareMathOperator{\Conv}{Conv}

\DeclareMathOperator{\inte}{int}
\DeclareMathOperator{\ext}{ext}

\DeclareMathOperator{\diam}{diam}

\usepackage{pifont}

\newcommand{\sub}{\mathrm{sub}}
\newcommand{\optdex}{\mathrm{OS}}
\newcommand{\twoline}[2]{\begin{tabular}{@{}c@{}}#1\\ {}#2\end{tabular}}

\title[Outperformance Inverse Optimization]{Outperformance Inverse Optimization: Learning Objective Functions that Outperform Agent Decisions}

\author{%
  Akira Kitaoka
}
\address{NEC Corporation, 1753 Shimonumabe, Nakahara-ku, Kawasaki, Kanagawa, Japan }
\email{akira-kitaoka@nec.com}


\begin{document}

\begin{abstract}

Inverse optimization estimates the weights of an objective function that explain observed decisions as optimal solutions, and is used in a variety of fields. For mixed-integer linear programs (MILPs), existing methods aim to reproduce the observations as optimal solutions, and thus learn compromise weights when the observations are suboptimal. We propose outperformance inverse optimization, which instead seeks weights that induce, at each state, an optimal solution outperforming the observed action in every component. We give a loss function that can be evaluated with forward-problem oracles alone and is thus applicable to MILPs, together with gradient-based and DC optimization algorithms for minimizing it. For weights inducing a unique outperforming optimal solution at all observations, we prove that the probability of failing to induce such a solution at a new state (the generalization error) is bounded by a quantity inversely proportional to the number of observations, and that this bound is tight in the number of observations up to logarithmic factors. In experiments on synthetic and real data, the proposed methods improve the prediction of solutions outperforming the actions over existing methods.

\end{abstract}

\maketitle

\section{Introduction}
Many processes and systems, from human decision making to natural phenomena, can often be described by mathematical optimization. In practice, however, the true objective function of such a model is in many cases unknown in advance. Therefore, the inverse optimization problem~\citep{ahuja2001inverse,heuberger2004inverse,chan2019inverse,chan2023inverse}, which estimates the objective function from observed data, is of great practical importance.

Inverse optimization problems have been studied extensively \citep{ahuja2001inverse,heuberger2004inverse,chan2019inverse,chan2023inverse} and applied in a variety of fields, including
geophysics \citep{tarantola1998inverse,burton1992instance},
transportation \citep{bertsimas2015data},
power systems \citep{birge2017inverse},
scheduling of television advertisements \citep{suzuki2019TV},
and healthcare \citep{chan2022inverse}.
They have also been developed as a foundation of machine learning methods such as inverse reinforcement learning \citep{ng2000algorithms} and contrastive learning \citep{shi2023understanding}.


A challenge in inverse optimization is that the observed actions are not necessarily optimal solutions of the mathematical optimization under consideration (due to, e.g., measurement noise and the bounded rationality of the decision maker; \citealp{aswani2018inverse}).
For such observations, tractable methods that relax the KKT conditions, the variational inequality, or the strong duality constraint \citep{keshavarz2011imputing,bertsimas2015data,chan2014generalized,chan2019inverse} have been proposed, but all of them presuppose the convexity of the forward problem and cannot be extended to mixed-integer linear programs (MILPs) \citep{barmann2018online}.
On the other hand, methods \citep{barmann2017emulating,barmann2018online,kitaoka2024exact,lin2024conformal,sakaue2025revisiting_online_learning_approach,sakaue2025online,kitaoka2026explicit} that use losses evaluable only with an optimization oracle for the forward problem, such as the suboptimality loss \citep{mohajerin2018data}, are applicable also to MILPs, but since they search for weights that explain the actions as optimal solutions, for suboptimal actions they learn compromise weights.
Moreover, in designing mathematical models, an objective function that reproduces the observed actions may be required, but there are also situations that require one inducing solutions that outperform the observed actions in every indicator.

In this paper, we therefore propose a framework that learns weights inducing solutions that outperform the agent's actions instead of reproducing them. The contributions of this paper are the following four.

\textbf{Formulation and the OS loss}\ \ We formulate the problem of finding weights that give, at each state, an optimal solution outperforming the agent's action in every component as outperformance inverse optimization, and as its surrogate loss we introduce the outperformance suboptimality (OS) loss, which can be evaluated using only optimization oracles for the forward problem and its dominance-constrained version (\Cref{sec:setup,sec:performance_metric}).

\textbf{Minimization methods}\ \ As methods for minimizing the OS loss, we propose gradient-based optimization methods, including the projected subgradient method, and a DC (difference of convex) optimization method that terminates after finitely many iterations if the inner convex subproblems are solved exactly (\Cref{sec:proposed_method}). Moreover, when the feasible region of each state is an M$^\natural$-convex set (including M-convex sets) and the weight space is the probability simplex, we show that the uniform weight makes the OS loss $0$ at all observations and thus attains its global minimum (\Cref{prop:mnatural_uniform_exact}, \Cref{cor:testcut_per_call_N}).

\textbf{Generalization error}\ \ For weights inducing a unique outperforming optimal solution at all observations, we show that the probability of failing to induce such a solution at a new state (the generalization error) is bounded by a quantity inversely proportional to the number of observations $N$, and via a lower bound that this upper bound is tight in $N$ up to log factors (\Cref{sec:generalization}).

\textbf{Experiments}\ \ In five experiments with synthetic and real data, we show that, with the optimization method fixed, replacing the suboptimality loss of existing inverse optimization with the OS loss improves the prediction of solutions outperforming the actions (smaller componentwise shortfalls): indeed, the OS-loss variant wins in 14 of the 15 (optimization method, experiment) pairs (\Cref{sec:numerical_experiments}).

\section{Related work}
\label{sec:related_work}

Details, including classical inverse optimization, are given in \Cref{sec:appendix_related_work}.

\paragraph{Inverse optimization for noisy data}
As data-driven inverse optimization that handles the case where the observed actions are not optimal solutions, methods have been proposed that use as the loss the distance between the observations and the optimal solutions \citep{aswani2018inverse} (NP-hard even when the forward problem is convex), the residual of the KKT conditions \citep{keshavarz2011imputing}, the violation of the variational inequality \citep{bertsimas2015data}, the relaxation of the strong duality constraint \citep{chan2014generalized,chan2019inverse}, and distributionally robust optimization of the suboptimality loss \citep{mohajerin2018data}.
However, since these tractable reformulations use the KKT conditions or duality, they presuppose the convexity of the forward problem and cannot be extended to the integer case \citep{bertsimas2015data,barmann2018online}.

\paragraph{Inverse optimization for MILPs}
Methods that minimize the suboptimality loss using only the forward-problem oracle do not require the convexity of the feasible region and are therefore applicable to MILPs \citep{barmann2017emulating,barmann2018online,kitaoka2024exact,lin2024conformal,sakaue2025revisiting_online_learning_approach,sakaue2025online,kitaoka2026explicit}.
However, since these seek weights that explain the agent's actions as optimal solutions, when the actions are suboptimal the suboptimality loss is positive even at the true weight \citep{sakaue2025revisiting_online_learning_approach,sakaue2025online}.

\paragraph{Generalization error analysis of inverse optimization}
\citet{kitaoka2025inverse} bounded the expected suboptimality loss by $O(\sqrt{d/N})$ for inverse optimization of MILPs that learns both the objective function and the constraints ($d$ is the dimension of the weights and $N$ is the number of observations).
\citet{fatemi2026tight} bounded, for noiseless inverse optimization in which the action is the unique optimal solution under the true weight, the probability that the true optimal solution ceases to be optimal under the learned weight by $O(d/N)$, and showed that this order is tight.
Both presuppose that the agent's actions are optimal solutions.

\paragraph{Position of this paper}
Rather than explaining the agent's actions as optimal solutions, this paper seeks weights that induce optimal solutions outperforming the actions in every component.
Since the OS loss, like the suboptimality loss, can be evaluated using only the forward-problem oracle, it is applicable to nonconvex forward problems including MILPs, and it can learn weights inducing solutions that outperform the actions even when the actions are suboptimal.
In addition, the generalization error analysis (\Cref{sec:generalization}) evaluates, in a setting where the actions may be suboptimal, the probability that a unique optimal solution outperforming the action is not obtained at a new state.

\section{Problem setting}
\label{sec:setup}

We consider an offline learning setting with two players, a \emph{learner} and an \emph{agent}.\footnote{\label{footnote:linear-model} The ``agent'' is sometimes called the ``expert,'' but we do not use this term to avoid confusion with experts in online learning.}
Let $d$ be a positive integer, and let $\mathbb{R}^d$ be the space on which the forward optimization is defined.
We call a nonempty set $\mathcal{S}$ the state set, and for each state $s \in \mathcal{S}$ we denote the feasible region by $ \mathcal{X} (s) (\subseteq \mathbb{R}^d)$.
For a weight $\theta \in \mathbb{R}^d$ and a state $s \in \mathcal{S}$, we write the forward problem (linear optimization) and its optimal solution as
\begin{equation}
    \label{eq:fop}
    x^* (\theta , s) :\in \argmax_{x \in \mathcal{X} (s)} \langle \theta , x \rangle
    .
\end{equation}
Let $\Theta \subseteq \mathbb{R}^d_{\geq 0 } (:= [0, \infty)^d )$ be the weight space.

For $n=1, \ldots , N $, given a state $s^n \in \mathcal{S}$, the agent chooses $x^n \in \mathcal{X}(s^n)$ as its action.
From the observations $\{ (s^n, x^n) \}_{n=1}^N$, we wish to find a weight $\theta^*$ satisfying, at each state $s^n$,
\begin{equation}
    \label{eq:outperformance_IOP}
    x^*(\theta^*, s^n) \geq x^n
    ,
\end{equation}
where $\geq$ means that $\geq$ holds componentwise. Hereafter, when $x\geq x^\prime$, we say that $x$ dominates (outperforms) $x^\prime$. We call solving \cref{eq:outperformance_IOP} outperformance inverse optimization (outperformance IO).
Each component of $x$ is oriented so that larger is better (e.g., times enter with reversed sign), and observation noise is assumed to keep $x^n$ feasible.

Note that the set $\mathcal{X}(s)$ is not necessarily convex.
If $\mathcal{X}(s)$ is the feasible region of a MILP, the solution of any MILP solver, e.g., Gurobi \citep{gurobi} or SCIP via OR-Tools \citep{ortools}, can serve as the oracle $x^*(\theta,s)$.

\section{Performance metrics}
\label{sec:performance_metric}
For $s \in \mathcal{S}$, we define an optimal solution that dominates
$x^\prime \in \mathcal{X}(s)$ as follows (we call this problem the dominance-constrained forward problem):
\begin{equation}
        x^* (\theta ; x^\prime ,s ) : \in \argmax_{x \in \mathcal{X}_{\geq x^\prime} (s)} \langle \theta , x \rangle  ,
        \quad
        \mathcal{X}_{\geq x^\prime} (s) := \{ x \in \mathcal{X} (s) \, | \, x \geq x^\prime \}
        .
    \label{eq:fop_superior}
\end{equation}
As a metric useful for outperformance inverse optimization,
we define the outperformance suboptimality loss (OS loss) $\ell^{\optdex}$ for any $s \in \mathcal{S}$ and data $x^\prime \in \mathcal{X}(s)$ by
\begin{equation}
    \ell^{\optdex} (\theta ; x^\prime , s ) :=
    \max_{x \in \mathcal{X} (s)} \langle \theta , x \rangle
    -
    \max_{x \in \mathcal{X}_{\geq x^\prime} (s) } \langle \theta , x \rangle
    .
    \label{eq:def_optdex}
\end{equation}
By the inclusion $\mathcal{X} (s) \supseteq \mathcal{X}_{\geq x^\prime} (s)$, the OS loss $\ell^{\optdex} $ is nonnegative.
The OS loss $\ell^{\optdex}(\theta;x^\prime,s)$ is $0$ if and only if there exists an optimal solution $x\in\argmax_{x\in\mathcal{X}(s)}\langle\theta,x\rangle$ satisfying $x\geq x^\prime$. In particular, if no point of $\mathcal{X}_{\geq x^\prime}(s)$ is optimal under any $\theta\in\Theta$ (e.g., an efficient but unsupported action; \Cref{sec:appendix_unsupported}), no weight attains \cref{eq:outperformance_IOP}.
However, when the optimal solution is not unique, the optimal solution $x^*(\theta,s)$ returned by the oracle does not necessarily dominate $x^\prime$.
To solve \cref{eq:outperformance_IOP}, we introduce the empirical mean of the OS loss
\begin{equation}
    \bar{\ell}^{\optdex} (\theta )
    :=
    \frac{1}{N} \sum_{n=1}^N
    \ell^{\optdex} (\theta ; x^n , s^n )
    .
    \label{eq:def_mean_optdex}
\end{equation}
The empirical mean of the OS loss $\bar{\ell}^{\optdex}(\theta)$ is $0$ if and only if, for each $n$, there exists an optimal solution $x\in\argmax_{x\in\mathcal{X}(s^n)}\langle\theta,x\rangle$ that dominates $x^n$.
In particular, if $\theta$ satisfies \cref{eq:outperformance_IOP}, then $\bar{\ell}^{\optdex}(\theta)=0$ since $x^*(\theta,s^n)\in\mathcal{X}_{\geq x^n}(s^n)$.
Conversely, even if $\bar{\ell}^{\optdex}(\theta)=0$, when the optimal solution is not unique, $x^*(\theta,s^n)$ returned by the oracle does not necessarily dominate $x^n$, and attaining \cref{eq:outperformance_IOP} is not guaranteed.
Accordingly, since a weight satisfying \cref{eq:outperformance_IOP} attains the minimum value $0$ of the empirical mean of the OS loss $\bar{\ell}^{\optdex}$, we treat \cref{eq:outperformance_IOP} as the following problem using the empirical mean of the OS loss $\bar{\ell}^{\optdex}$:
\begin{equation}
    \text{minimize}\,
    \bar{\ell}^{\optdex} (\theta )
    \text{ such that } \theta \in \Theta
    .
    \label{eq:empirical_outperformance_IOP}
\end{equation}

\begin{remark}
    \label{rem:sub_loss_def}
    The suboptimality loss is defined by $\ell^{\sub}(\theta;x^\prime,s):=\max_{x\in\mathcal{X}(s)}\langle\theta,x\rangle-\langle\theta,x^\prime\rangle$ \citep{mohajerin2018data}.
\end{remark}

To solve outperformance inverse optimization, we introduce the following assumption.
\begin{assumption}
    \label{assu:margin}
    \begin{description}
        \item[(1)] The weight space $\Theta \subset \mathbb{R}^d_{\geq 0}$ is a nonempty bounded polyhedron $\Theta = \{ \theta \in \mathbb{R}^d_{\geq 0} \mid A_\Theta \theta \leq b_\Theta \}$ with $A_\Theta\in\mathbb{R}^{m_\Theta\times d}$ and $b_\Theta\in\mathbb{R}^{m_\Theta}$, and $\mathbf{0}\notin\Theta$ (in the implementation, we take the probability simplex $\Delta^{d-1} := \{ \theta \in \mathbb{R}^d_{\geq 0} \mid \sum_{i=1}^d \theta_i = 1 \}$).
        \item[(2)] For each state $s \in \mathcal{S}$, $\mathcal{X} (s) \subset \mathbb{R}^d$ is a finite union of nonempty bounded polyhedra (we call each bounded polyhedron constituting the union a piece). Moreover, $\mathcal{X}(s)$ is uniformly bounded in $s$, that is, $\sup_{s\in\mathcal{S}}\sup_{x\in\mathcal{X}(s)}\|x\|<\infty$.
        \item[(3)] (Extreme-point oracle) For any $\theta \in \mathbb{R}^d$, the forward-problem oracle $x^* (\cdot , s)$ and the dominance-constrained forward-problem oracle $x^* (\cdot \, ; x^\prime , s)$ return a vertex of a bounded polyhedral piece attaining the maximum.
    \end{description}
\end{assumption}

We denote the sets of all values that the forward-problem oracle $x^* (\cdot , s)$ and the dominance-constrained forward-problem oracle $x^* (\cdot \, ; x^\prime , s)$ can return by $\mathcal{Y} (s)$ and $\mathcal{Y}_{\geq x^\prime} (s)$, respectively.
\begin{remark}[On \Cref{assu:margin}]
    \label{rem:assu_margin}
    By (2), $\mathcal{X} (s)$ is nonempty and compact, and for $x^\prime \in \mathcal{X} (s)$ the dominance-constrained feasible region $\mathcal{X}_{\geq x^\prime} (s) := \{ x \in \mathcal{X} (s) \mid x \geq x^\prime \}$ (\cref{eq:fop_superior}) is also a finite union of nonempty bounded polyhedra.

    (3) is a natural assumption as the standard behavior of LP and MILP solvers based on the simplex method and branch-and-bound.
    Moreover, $\mathcal{Y} (s)$ and $\mathcal{Y}_{\geq x^\prime} (s)$ are finite subsets of $\mathcal{X}(s)$, since they consist of vertices of the bounded polyhedral pieces.
\end{remark}

\section{Proposed methods}
\label{sec:proposed_method}

\subsection{Gradient-based optimization methods}
\label{subsec:gradient_method}
\label{sec:gradient_based_method}

\begin{proposition}
    \label{prop:gredient_OA_loss}
    Assume \Cref{assu:margin}(2). Let $s \in \mathcal{S}$ and
    $x^\prime \in \mathcal{X}(s)$. Then
    $g (\theta; x^\prime, s) := x^* (\theta , s) - x^* (\theta ; x^\prime ,s )$ is
    a gradient of $\ell^{\optdex}  (\theta ; x^\prime , s )$ with respect to $\theta$ in the sense of weak derivatives (derivatives in the sense of Schwartz distributions; \citealp[cf.][]{kunzinger2019theory}).\footnote{For a continuously differentiable function, the weak derivative coincides with the ordinary gradient. Moreover, by Rademacher's theorem~\citep[cf.][Theorem~10.8~(ii)]{villani2009optimal}, $\ell^{\optdex}$ is differentiable almost everywhere, and $g=\nabla_{\theta}\ell^{\optdex}$ holds at such points (\Cref{sec:weak_derivative_OA_loss}).}
\end{proposition}
For the proof, see \Cref{sec:weak_derivative_OA_loss}.
By \Cref{prop:gredient_OA_loss}, a gradient of the empirical mean of the OS loss $\bar{\ell}^{\optdex}$ in the sense of weak derivatives is given by $\frac{1}{N} \sum_{n=1}^N g (\theta; x^n, s^n )$.
Therefore,
gradient-based optimization methods can be applied to solve \cref{eq:empirical_outperformance_IOP}, although they may stall at a local minimum or a stationary point without reaching the minimum value.
In particular, in its simplest form, the projected subgradient method for the OS loss is the update $\theta^{k+1}\leftarrow\Pi_\Theta\bigl(\theta^k-\frac{\alpha_k}{N}\sum_{n=1}^N(x^*(\theta^k,s^n)-x^*(\theta^k;x^n,s^n))\bigr)$ ($\Pi_\Theta$ is the Euclidean projection onto $\Theta$ and $\alpha_k>0$ is the learning rate at iteration $k$), which differs from that for the suboptimality loss \citep{kitaoka2024exact} only in that the observed action $x^n$ is replaced by $x^*(\theta^k;x^n,s^n)$ (\Cref{alg:psg}).
In this paper, as gradient-based optimization methods, we adopt the projected subgradient method (PSG) \citep[cf.][]{boyd2003subgradient}, the Online Newton Step (ONS) \citep{hazan2007logarithmic}, and MetaGrad \citep{van2016metagrad}. For the Online Newton Step and MetaGrad, we adapt the implementation that \citet{sakaue2025online} used for online inverse linear optimization. Although these are online methods, we update them with the average of the gradients over all states at each iteration (full-batch). Implementation details are given in \Cref{sec:appendix_experiments}.

\subsection{DC optimization}
\label{sec:dc_opt}

We call a tuple of dominating feasible solutions $y^\prime=(y^{\prime n})_{n=1}^N$ ($y^{\prime n}\in\mathcal{X}_{\geq x^n}(s^n)$) a branch (not to be confused with branching in branch-and-bound). Fixing a branch replaces the second term of \cref{eq:def_optdex} (concave in $\theta$) with the linear function $\langle\theta,y^{\prime n}\rangle$, which yields the convex upper bound on the empirical mean of the OS loss
\begin{equation}
    \bar{\ell}^{\sub}(\theta;y^\prime)
    :=\frac{1}{N}\sum_{n=1}^N\Bigl[\max_{x\in\mathcal{X}(s^n)}\langle\theta,x\rangle-\langle\theta,y^{\prime n}\rangle\Bigr]
    \;\ge\;\bar{\ell}^{\optdex}(\theta)
    .
    \label{eq:def_sub_loss_main}
\end{equation}
We call $\min_{\theta\in\Theta}\bar{\ell}^{\sub}(\theta;y^\prime)$ the convex subproblem of the branch $y^\prime$.
We denote the set of all branches by $\mathcal{X}^{\mathrm{all}}_{\geq}:=\prod_{n=1}^N\mathcal{X}_{\geq x^n}(s^n)$ and the set of those consisting of values that the dominance-constrained forward-problem oracle can return by $\mathcal{Y}^{\mathrm{all}}_{\geq}:=\prod_{n=1}^N\mathcal{Y}_{\geq x^n}(s^n)$, and we let $m_{\sub}(y^\prime):=\min_{\theta\in\Theta}\bar{\ell}^{\sub}(\theta;y^\prime)$ be the minimum value of the convex subproblem of a branch $y^\prime\in\mathcal{X}^{\mathrm{all}}_{\geq}$ (the minimum is attained since $\Theta$ is compact and $\bar{\ell}^{\sub}(\cdot\,;y^\prime)$ is continuous).

\begin{proposition}[Decomposition into branches]
    \label{prop:dc_branch_decomp}
    Under \Cref{assu:margin}, for any $\theta\in\mathbb{R}^d$, we have
    \begin{equation}
        \bar{\ell}^{\optdex}(\theta)
        =
        \min_{y^\prime\in\mathcal{X}^{\mathrm{all}}_{\geq}} \bar{\ell}^{\sub}(\theta;y^\prime)
        =
        \min_{y^\prime\in\mathcal{Y}^{\mathrm{all}}_{\geq}} \bar{\ell}^{\sub}(\theta;y^\prime) .
        \label{eq:dc_pointwise_min}
    \end{equation}
    Therefore, for the finite set $\mathcal{Y}^{\mathrm{all}}_{\geq}$, we have
    \begin{equation}
        \min_{\theta\in\Theta}\bar{\ell}^{\optdex}(\theta)
        =
        \min_{y^\prime\in\mathcal{Y}^{\mathrm{all}}_{\geq}} m_{\sub}(y^\prime) .
        \label{eq:dc_min_min}
    \end{equation}
\end{proposition}

The proof is given in \Cref{subsec:dc_decomp}.

\Cref{prop:dc_branch_decomp} specializes to our problem the result of Sections 4.1--4.2 of \citet{pham1997convex} that a global solution of a polyhedral DC program reduces to finitely many convex programs (the DC structure of the empirical mean of the OS loss is described in \Cref{rem:OA_loss_is_DC}, and the correspondence in \Cref{subsec:dc_literature}). Our DC optimization uses this structure directly. By \cref{eq:dc_pointwise_min}, \cref{eq:empirical_outperformance_IOP} is equivalent to the joint minimization of $\bar{\ell}^{\sub}(\theta;y^\prime)$ over $(\theta,y^\prime)\in\Theta\times\mathcal{Y}^{\mathrm{all}}_{\geq}$, and \Cref{alg:dc_outer} in the appendix (\Cref{sec:appendix_dc_opt}) performs this joint minimization alternately in $y^\prime$ and $\theta$. The minimization in $y^\prime$ amounts to updating the branch by calling the dominance-constrained forward-problem oracle at each state, and the minimization in $\theta$ amounts to solving the convex subproblem of the branch. This alternating minimization is nothing but the simplified DCA (a simplified version of the DC Algorithm) of \citet{pham1997convex} (\Cref{rem:OA_loss_is_DC,subsec:dc_literature}).
Moreover, if we take the agent's actions themselves as the branch, $\bar{\ell}^{\sub}(\theta;(x^n)_{n=1}^N)$ is the empirical mean of the suboptimality loss \citep{mohajerin2018data}. Hence the first equality of \cref{eq:dc_pointwise_min} means that the empirical mean of the OS loss is the pointwise minimum of the empirical means of the suboptimality loss obtained by replacing the agent's actions with feasible solutions dominating them.

In this paper, we solve \cref{eq:empirical_outperformance_IOP} by a double-loop method with this simplified DCA as the outer loop and a solver such as the cutting-plane method \citep{kelley1960cutting} (the inner solver) as the inner loop. Details such as the pseudocode are given in \Cref{sec:appendix_dc_opt}.

If the inner convex subproblems are solved exactly, this double-loop method terminates after finitely many iterations (\Cref{prop:dc_finite_convergence}). If the optimal solutions of the dominance-constrained forward problems are unique at a point where \Cref{alg:dc_outer} stops because the branch is no longer updated (a fixed point of DCA), that point is a local minimizer of the empirical mean of the OS loss (\Cref{prop:dc_local_min}; the set of weights at which the optimal solution is not unique has measure $0$, \Cref{prop:argmax_unique_ae}). However, reaching the global minimum is not guaranteed. Since the OS loss is nonnegative, if the loss value becomes $0$, it can be confirmed that the point is a global minimizer (\Cref{prop:dc_global_certificate}). In particular, when the feasible region of each state is an M$^\natural$-convex set (including M-convex sets) and $\Theta=\Delta^{d-1}$, the uniform weight makes the empirical mean of the OS loss $0$ and attains the global minimum (\Cref{prop:mnatural_uniform_exact}, \Cref{cor:testcut_per_call_N}).

\section{Generalization error analysis}
\label{sec:generalization}

In this section, we evaluate whether weights learned from the observations induce a unique optimal solution outperforming the agent's action also at states that were not observed, that is, the generalization of \cref{eq:outperformance_IOP}.
We show that the generalization error (\Cref{def:gen_error}) of weights that give such optimal solutions at all observations is uniformly bounded by order $d/N$ up to log factors for fixed $\kappa$ (\Cref{assu:gen}) (in general, the bound contains a term $d\log\kappa/N$), and that this upper bound is tight up to log factors.
The proofs are given in \Cref{sec:appendix_generalization}.

Suppose that the pair $(s,x)$ of a state $s$ and the agent's action $x$ at that state follows a distribution $\mathcal{P}$, that the observations $\{(s^n,x^n)\}_{n=1}^N$ are i.i.d.\ samples from $\mathcal{P}$, and we denote the probability with respect to the distribution of the observations by $\mathbb{P}$ (see \Cref{subsec:appendix_gen_preliminaries} for details of the setting).
We make the following assumption.

\begin{assumption}
    \label{assu:gen}
    \begin{description}
        \item[(1)] $\kappa:=\operatorname{ess\,sup}_{s}|\mathcal{Y}(s)|$ satisfies $2\le\kappa<\infty$.
        \item[(2)] (Realizability) There exists $\theta^*\in\Theta$ such that, a.s.\ for $(s,x)\sim\mathcal{P}$, the optimal solution of the forward problem \cref{eq:fop} at the state $s$ is unique and that optimal solution $x^*(\theta^*,s)$ dominates the agent's action $x$ ($x^*(\theta^*,s)\geq x$).
    \end{description}
\end{assumption}

We define the generalization error as follows.

\begin{definition}
    \label{def:gen_error}
    For a weight $\theta\in\Theta$, a state $s\in\mathcal{S}$, and data $x^\prime\in\mathbb{R}^d$, we define the \emph{outperformance error} by
    \begin{equation*}
        \ell^{01}(\theta;x^\prime,s):=
        \begin{cases}
            0 & \text{if }\argmax_{y\in\mathcal{X}(s)}\langle\theta,y\rangle=\{v\}\text{ for some }v\geq x^\prime,\\
            1 & \text{otherwise},
        \end{cases}
    \end{equation*}
    and define the \emph{generalization error} as the probability $\mathrm{GE}(\theta):=\mathcal{P}\bigl(\ell^{01}(\theta;x,s)=1\bigr)$ with respect to $(s,x)\sim\mathcal{P}$.
    We say that $\theta$ is \emph{consistent} with the observations if $\ell^{01}(\theta;x^n,s^n)=0$ for all $n$.
\end{definition}

The upper bound is as follows.

\begin{theorem}
    \label{thm:gen_upper_general}
    Assume \Cref{assu:margin}(2)(3) and \Cref{assu:gen}(1), and let $N\ge d$.
    For any $\beta\in(0,1)$, with probability at least $1-\beta$, any weight $\hat{\theta}^N\in\Theta$ consistent with the observations (it may depend on the observations) satisfies
    \begin{equation*}
        \mathrm{GE}(\hat{\theta}^N)\le\frac{2}{N}\left(d\log_2\frac{2\mathrm{e}N\kappa^2}{d}+\log_2\frac{2}{\beta}\right)
        ,
    \end{equation*}
    where $\mathrm{e}$ is Napier's constant.
\end{theorem}

\begin{example}[Integer programming]
    \label{ex:gen_upper_integer}
    In addition to the assumptions of \Cref{thm:gen_upper_general}, suppose that there exist positive integers $M_1,\ldots,M_d$ such that $\mathcal{X}(s)\subseteq\prod_{i=1}^d\{0,1,\ldots,M_i\}$ holds a.s.
    Since $\kappa\le\prod_{i=1}^d(M_i+1)$ by $\mathcal{Y}(s)\subseteq\mathcal{X}(s)$ (\Cref{rem:assu_margin}), \Cref{thm:gen_upper_general} implies that, for any $\beta\in(0,1)$, with probability at least $1-\beta$, any weight $\hat{\theta}^N\in\Theta$ consistent with the observations satisfies
    \begin{equation*}
        \mathrm{GE}(\hat{\theta}^N)\le\frac{2}{N}\left(d\log_2\frac{2\mathrm{e}N}{d}+2d\sum_{i=1}^d\log_2(M_i+1)+\log_2\frac{2}{\beta}\right)
        .
    \end{equation*}
    In particular, for 0-1 integer programming ($M_i=1$), the second term in the parentheses on the right-hand side is $2d^2$.
\end{example}

As in \Cref{ex:gen_upper_integer}, $\kappa$ can depend exponentially on $d$, in which case the upper bound of \Cref{thm:gen_upper_general} contains a term of order $d^2/N$.

Since \Cref{thm:gen_upper_general} does not depend on a specific estimator, it applies directly to the weight reached by the methods of \Cref{sec:proposed_method} whenever that weight has OS loss $0$ and a unique optimal solution at all observations.

The lower bound is as follows. It is obtained by embedding the instance that \citet{fatemi2026tight} constructed to show that their upper bound cannot be improved (Theorem~3.1 of that paper; the case where the agent's action is the unique optimal solution under the true weight) into our setting, in which the weights are restricted to the probability simplex.

\begin{theorem}
    \label{thm:gen_lower}
    For any integer $d_0\ge2$, with $d=2d_0+1$, there exists a tuple of a state set $\mathcal{S}$, a distribution $\mathcal{P}$, feasible regions $\mathcal{X}(s)$, and the weight space $\Theta=\Delta^{d-1}$ that satisfies \Cref{assu:margin} and \Cref{assu:gen} with $\kappa=3$ such that, for any $\beta\in(0,1/2]$ and $N\ge\max\{d_0,\ln(1/\beta)\}$, there exists a weight $\hat{\theta}^N\in\Theta$ consistent with the observations (depending on the observations) that satisfies, with probability at least $\beta$,
    \begin{equation*}
        \mathrm{GE}(\hat{\theta}^N)>\frac{d_0+\ln(1/\beta)}{4N}=\frac{d-1+2\ln(1/\beta)}{8N}
        .
    \end{equation*}
\end{theorem}

Since \Cref{thm:gen_upper_general} is an upper bound holding with probability at least $1-\beta$ and \Cref{thm:gen_lower} is a lower bound holding with probability at least $\beta$, we compare them by applying \Cref{thm:gen_upper_general} with $\beta/2$. Then, for $\beta\in(0,1/2]$ and $N\ge\max\{d,\ln(1/\beta)\}$, both events hold simultaneously with probability at least $\beta/2$, so the right-hand side of \Cref{thm:gen_upper_general} with $\beta$ replaced by $\beta/2$, for $\kappa=3$ and $d=2d_0+1$, is larger than that of \Cref{thm:gen_lower}. This replacement only changes the constant inside the logarithm, so tightness up to log factors is preserved. That is, \Cref{thm:gen_lower} guarantees the tightness of \Cref{thm:gen_upper_general}.

\section{Numerical experiments}
\label{sec:numerical_experiments}

In this section, we compare outperformance inverse optimization (minimization of the empirical mean of the OS loss \cref{eq:def_mean_optdex}) with existing inverse optimization (minimization of the suboptimality loss) in terms of the degree to which \cref{eq:outperformance_IOP} is attained.
The five experiments are the knapsack problem and the single-machine scheduling problem with release times using synthetic data, the movie selection problem with viewing costs using MovieLens100K \citep{harper2015movielens}, and the weighted multi-cover problem (selection of reading lists and reference lists) using Cora \citep{sen2008collective} and ogbn-arxiv \citep{hu2020open}.
The results of all methods, implementation details, hyperparameters, and the experimental environment are given in \Cref{sec:appendix_experiments}.

\subsection{Experimental setup}
\label{subsec:experimental_setup}

\paragraph{Compared methods}
In the main text, we compare seven methods. Sub- and OS- indicate minimizing the empirical mean of the suboptimality loss $\bar{\ell}^{\sub}(\theta;(x^n)_{n=1}^N)$ (\cref{eq:def_sub_loss_main}) and the empirical mean of the OS loss, respectively, and the following name indicates the optimization method (PSG: projected subgradient method \citep[cf.][]{boyd2003subgradient}, ONS \citep{hazan2007logarithmic}, MetaGrad \citep{van2016metagrad}, (FB): the full-batch version using the average of the subgradients (weak derivatives for the OS loss) over all states at each iteration).
The comparison of the three pairs consisting of the existing methods Sub-PSG \citep{kitaoka2024exact}, Sub-ONS(FB), and Sub-MetaGrad(FB) \citep{sakaue2025online} and the proposed methods OS-PSG, OS-ONS(FB), and OS-MetaGrad(FB), each obtained by replacing only the loss function, represents the effect of the difference in the loss function. We call the OS- method of each pair the OS-loss variant and the Sub- method the suboptimality-loss variant. The remaining method, OS-DC, is the proposed method based on the DC optimization of \Cref{sec:dc_opt}. As the inner solver we use the proximal bundle method with cutting-plane finishing (\Cref{subsec:rule_cpfinish}), and after DC optimization stops, we apply the projected subgradient method with a small step size to the empirical mean of the OS loss, starting from the obtained weight (subgradient finishing; OS-DC[PB+CPF]+SGF in \Cref{tab:methods}).
For all methods, the initial point is the uniform weight $\theta=\mathbf{1}/d$ and the weight space is the probability simplex $\Theta=\Delta^{d-1}$.
The list of all 13 methods including these (\Cref{tab:methods}), and the detailed settings and hyperparameters, are given in \Cref{sec:appendix_experiments}.

\paragraph{Evaluation metrics}
The main evaluation metric is the outperformance prediction loss $\ell^{\mathrm{OP}}(\theta):=\frac{1}{N}\sum_{n=1}^N\sum_{i=1}^d\max\{x_i^n-x_i^*(\theta,s^n),0\}$ (\cref{eq:def_opl}). That is, $\ell^{\mathrm{OP}}$ is the average over states of the sum of the componentwise shortfalls of the oracle's solution relative to the agent's action.
We have $\ell^{\mathrm{OP}}(\theta)=0$ if and only if $\theta$ is a solution of outperformance inverse optimization \cref{eq:outperformance_IOP}.
In addition, we record the attainment rate of $\ell^{\mathrm{OP}}=0$ (the fraction of episodes---the unit in which one set of observations $\{(s^n,x^n)\}_{n=1}^N$ is generated or extracted and all methods are run on it---that reached $\ell^{\mathrm{OP}}=0$) and the empirical mean of the OS loss $\bar{\ell}^{\optdex}$.

\paragraph{Measure of computational cost}
We measure the computational cost by the number of calls to the forward-problem oracles.
Calls to $x^*(\cdot,s)$ and to $x^*(\cdot\,;x^\prime,s)$ are both counted as one, and calls to $x^*(\cdot\,;x^\prime,s)$, which the existing methods do not solve, are counted only on the side of the proposed methods.
Letting $T$ be the number of iterations of OS-PSG (Sub-PSG, which makes half as many oracle calls per iteration, runs $2T$ iterations) and $N$ the number of states, we impose a common oracle-call budget $B:=2TN$ on all methods, evaluate each method by the weight with the smallest $\ell^{\mathrm{OP}}$ among the weights visited so far (best-so-far), and compare the values at the point where the number of calls reaches $B$.
Details of how the calls are counted and of the evaluation procedure are given in \Cref{subsec:appendix_pseudocode_methods}.

\subsection{Overview of the experiments}
\label{subsec:experiment_overview}

In the main text, among the five experiments, we present the results of the knapsack problem with synthetic data and the weighted multi-cover problem with real data (Cora \citep{sen2008collective}, ogbn-arxiv \citep{hu2020open}).
The results of the single-machine scheduling problem and the movie selection problem with viewing costs are given in \Cref{subsec:appendix_other_results}, and the mathematical models and instances of all experiments in \Cref{subsec:appendix_problems}.
The knapsack problem uses the bounded integer knapsack problem as the forward problem, and the agent's action is obtained by decreasing each component of the optimal solution under the true weight $\theta^*$ by $1$ with probability $0.2$. Hence a weight $\theta^*$ satisfying \cref{eq:outperformance_IOP} exists.
The weighted multi-cover problem models, as a weighted multi-cover of words, the situation in which a researcher selects a limited number of papers, and the agent's action is the depth to which the set of papers actually in citation relations covers each word (the depth to which a word is covered is the number of selected papers containing that word, capped at an upper bound $r$).
Except for the single-machine scheduling problem ($N=10$, $d\in\{4,6,8\}$, $B=2000$), we set $N=20$, $d\in\{20,40,60,80,100\}$, and $T=100$ ($B=4000$); in each episode, all methods use the same instance.

\subsection{Experimental results}
\label{subsec:experiment_results}

\begin{table}[tbp]
    \centering
    \caption{Comparison of the three pairs differing only in the loss function (OS-PSG vs.\ Sub-PSG, OS-ONS(FB) vs.\ Sub-ONS(FB), OS-MetaGrad(FB) vs.\ Sub-MetaGrad(FB)). The numbers of episodes on the same instances in which $\ell^{\mathrm{OP}}$ at the point where the number of calls reaches $B$ is smaller/equal/larger for the OS-loss variant (wins/ties/losses). The larger of wins and losses is shown in bold.}
    \label{tab:loss_only_pairs}
    \setlength{\tabcolsep}{2.5pt}
    \begin{tabular}{l|ccc|c}
        \hline
        Problem & PSG & ONS(FB) & MetaGrad(FB) & Episodes \\
        \hline
        Knapsack & \textbf{241}/9/0 & 75/27/\textbf{148} & \textbf{230}/4/16 & 250 \\
        Single machine & \textbf{24}/51/0 & \textbf{30}/42/3 & \textbf{28}/47/0 & 75 \\
        Movie (MovieLens) & \textbf{97}/2/1 & \textbf{45}/23/32 & \textbf{90}/9/1 & 100 \\
        Multi-cover (Cora) & \textbf{92}/2/6 & \textbf{59}/19/22 & \textbf{72}/11/17 & 100 \\
        Multi-cover (ogbn-arxiv) & \textbf{81}/15/4 & \textbf{74}/22/4 & \textbf{63}/30/7 & 100 \\
        \hline
    \end{tabular}
\end{table}

\Cref{tab:loss_only_pairs} shows, for all five experiments, the wins and losses of the three pairs differing only in the loss function; \Cref{tab:knapsack_main} shows the attainment rates for the knapsack problem; and \Cref{fig:opl_vs_calls} shows the evolution of $\ell^{\mathrm{OP}}$ for the weighted multi-cover problem (that for the knapsack problem is shown in \Cref{fig:opl_vs_calls_knapsack} in the appendix).
Overall, when the optimization method is fixed and only the suboptimality loss is replaced with the OS loss, the prediction of solutions outperforming the agent's actions improves. Indeed, the OS-loss variant won in 14 of the 15 pairs.
A similar tendency is observed in the comparison replacing only the loss function also for the single-machine scheduling problem and the movie selection problem with viewing costs in the appendix (\Cref{subsec:appendix_other_results}).

\begin{table}[tbp]
    \centering
    \caption{Knapsack problem: the fraction [\%] of episodes that attained $\ell^{\mathrm{OP}}=0$ at the point where the number of calls reaches $B=4000$. 50 episodes for each $d$. For each optimization method, the OS loss (OS) and the suboptimality loss (Sub) are shown side by side, and the higher one in each pair is shown in bold. The DC optimization uses only the OS loss. The median of $\ell^{\mathrm{OP}}$ for each dimension is shown in \Cref{tab:knapsack_median}.}
    \label{tab:knapsack_main}
    \begin{tabular}{c|cc|cc|cc|c}
        \hline
        & \multicolumn{2}{c|}{PSG} & \multicolumn{2}{c|}{ONS(FB)} & \multicolumn{2}{c|}{MetaGrad(FB)} & OS-DC \\
        $d$ & OS & Sub & OS & Sub & OS & Sub & \\
        \hline
        20 & \textbf{98} & 16 & \textbf{34} & 0 & \textbf{86} & 6 & 90 \\
        40 & \textbf{96} & 0 & \textbf{2} & 0 & \textbf{52} & 0 & 88 \\
        60 & \textbf{92} & 0 & 0 & 0 & \textbf{46} & 0 & 86 \\
        80 & \textbf{90} & 0 & 0 & 0 & \textbf{8} & 0 & 82 \\
        100 & \textbf{76} & 0 & 0 & 0 & \textbf{12} & 0 & 84 \\
        \hline
    \end{tabular}
\end{table}

\begin{figure}[tbp]
    \centering
    \includegraphics[width=\linewidth]{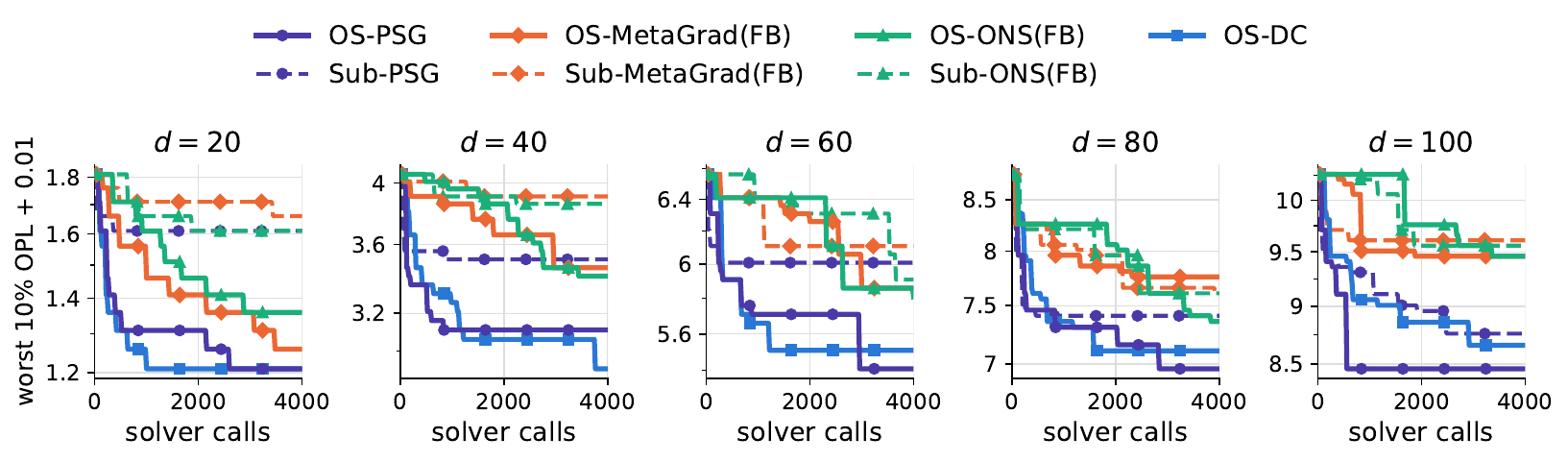}\\
    (a) Weighted multi-cover problem (Cora)\\[2mm]
    \includegraphics[width=\linewidth]{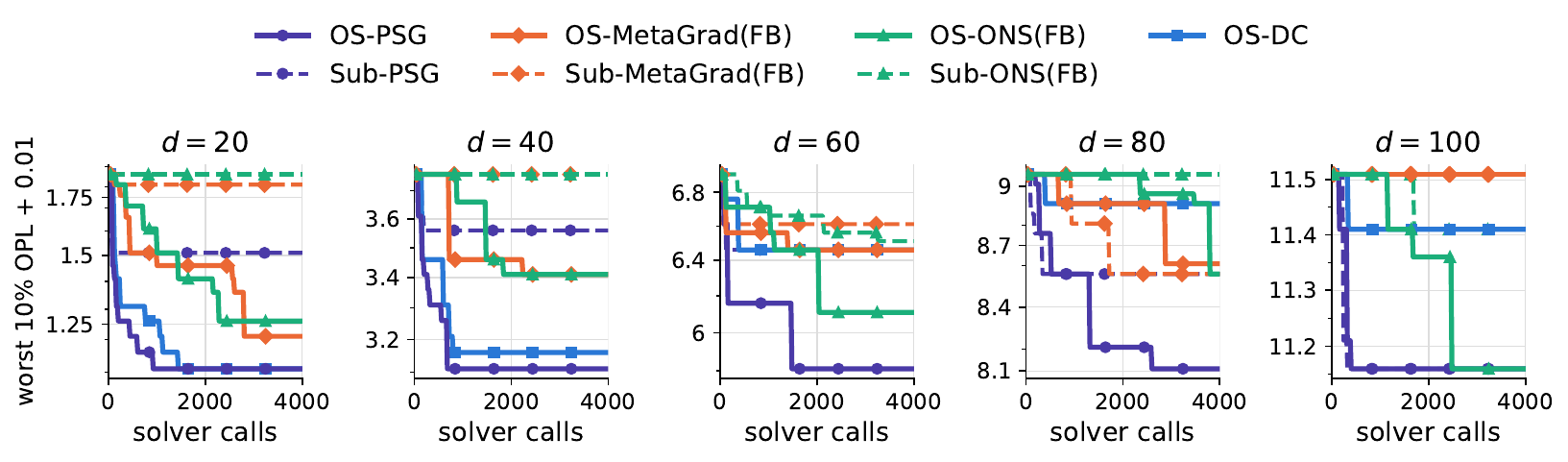}\\
    (b) Weighted multi-cover problem (ogbn-arxiv)
    \caption{Relation between the worst 10\% (the 90th percentile across episodes) of $\ell^{\mathrm{OP}}$ (best-so-far) and the number of oracle calls. The vertical axis is $\ell^{\mathrm{OP}}+0.01$ on a logarithmic scale, and $0.01$ represents $\ell^{\mathrm{OP}}=0$. Since the actions are integer-valued, $\ell^{\mathrm{OP}}$ takes discrete values in steps of $1/N$, and the percentiles are chosen from the observed values without interpolation. The colors indicate the optimization methods (PSG, MetaGrad(FB), ONS(FB), DC optimization); solid lines indicate the OS loss (proposed) and dashed lines the suboptimality loss (existing).}
    \label{fig:opl_vs_calls}
\end{figure}

\paragraph{Knapsack problem}
In this problem, where the agent's actions are degraded versions of optimal solutions, the improvement is the most pronounced.
OS-PSG and OS-DC attained \cref{eq:outperformance_IOP} in at least 76\% of the episodes in all dimensions, whereas the attainment rate of the methods with the suboptimality loss is at most 16\% (\Cref{tab:knapsack_main}; the attainment rates of Sub-ONS and Sub-MetaGrad, which update state by state (\Cref{tab:methods}), are also at most 9\% when aggregated over all dimensions; \Cref{tab:full_knapsack}).
Existing inverse optimization searches for weights that make the degraded actions optimal and therefore moves toward compromise weights that fit the degradation, whereas the OS loss measures the gap to the best feasible solution dominating the action and is thus presumably less affected by it.
However, only in the ONS(FB) pair did the suboptimality-loss variant win (75 wins and 148 losses; \Cref{tab:loss_only_pairs}). In this pair, the attainment rate is $0$\% for both losses for $d\ge60$ (\Cref{tab:knapsack_main}), presumably because the full-batch ONS does not converge within the oracle-call budget $B$.

\paragraph{Weighted multi-cover problem}
In this problem, which uses the actions of real data, on both Cora and ogbn-arxiv the OS-loss variant won in all three pairs (\Cref{tab:loss_only_pairs}).
In particular, in the PSG pair, the OS-loss variant won in 92 of 100 episodes on Cora and 81 of 100 on ogbn-arxiv.

\section{Conclusion}
\label{sec:conclusion}

In this paper, we proposed outperformance inverse optimization. Instead of explaining the actions as optimal solutions, it seeks weights that induce, at each state, an optimal solution outperforming the action in every component.
Its surrogate, the OS loss, requires only oracles for the forward problem and its dominance-constrained version, and thus applies to nonconvex forward problems including MILPs. Its empirical mean vanishes if and only if optimal solutions dominating the actions exist.
To minimize the OS loss, we gave gradient-based optimization methods and a DC optimization that terminates finitely if the inner convex subproblems are solved exactly. When each feasible region is M$^\natural$-convex (including M-convex) and $\Theta=\Delta^{d-1}$, the uniform weight attains the global minimum $0$ of the OS loss. We also showed that the upper bound on the generalization error is tight up to log factors for fixed $\kappa$.
In the experiments, replacing only the loss function made the OS-loss variant win in 14 of 15 pairs, with particularly large gains when the actions were degraded optimal solutions.

We list two directions for future work.
The first is to clarify when the DC optimization for the OS loss reaches a global minimizer and how it can be made to reach one.
The second is to clarify whether the dependence of the upper bound on the generalization error on $\kappa$ (which yields a term of order $d^2/N$ for integer programming; \Cref{ex:gen_upper_integer}) is essential, and whether an upper bound without log factors can be obtained for a specific estimator.

\bibliographystyle{apalike}
\bibliography{outperformance_IO}

\newpage
\appendix

\section{Related work}
\label{sec:appendix_related_work}

\paragraph{Classical inverse optimization}
Classical inverse optimization is the problem of finding the objective weight closest to a prior estimate such that a given solution becomes optimal; when the forward problem is an LP, it can be solved via LP duality \citep{ahuja2001inverse}.
As extensions to MILP, \citet{wang2009cutting} proposed a cutting-plane method that uses the forward problem as an oracle, and \citet{bulut2021complexity} studied the computational complexity of inverse MILP.
All of these assume that the given solution is optimal under some weight.
When this assumption fails, classical inverse optimization returns trivial solutions such as the zero vector in the linear case \citep{chan2018trade_off_preservation}.
For an overview of inverse optimization, see the survey by \citet{chan2023inverse}.

\paragraph{Details of inverse optimization for noisy data}
The method that minimizes the distance between the observations and the optimal solutions (the prediction loss) \citep{aswani2018inverse} is statistically consistent, but it leads to a bilevel program that is NP-hard even when the forward problem is convex, and it has to rely on enumerating the parameters.
The tractable reformulations based on the residual of the KKT conditions \citep{keshavarz2011imputing}, the violation of variational inequalities \citep{bertsimas2015data}, the relaxation of the strong duality constraint \citep{chan2014generalized,chan2019inverse}, and distributionally robust optimization of the suboptimality loss \citep{mohajerin2018data} all assume convexity (conic representability) of the forward problem.
Methods that estimate the constraints of the feasible region \citep{chan2020inverse,ghobadi2020inferring,ren2025inverse} also mainly target LPs.

\paragraph{Details of inverse optimization for MILP}
Among the methods that use only the forward-problem oracle, in the offline setting, \citet{lin2024conformal} proposed point estimation by a cutting-plane method, and \citet{kitaoka2024exact,kitaoka2026explicit} obtained exact solutions in finitely many iterations by gradient methods.
The augmented suboptimality loss \citep{zattoni2024learning} is evaluated by solving an augmented forward problem with a distance term, and its convex reformulation over mixed-integer sets enumerates the integer points; except for binary decision variables, for which the augmented forward problem reduces to the forward problem with modified weights, it cannot be evaluated with the forward-problem oracle alone. Moreover, like the suboptimality loss, it aims to reproduce the observations as optimal solutions. We therefore do not compare with it in our experiments.
As a method that learns the constraints in addition to the objective, \citet{kitaoka2025inverse} proposed two-stage learning of the constraints and the objective.
In the online setting, \citet{barmann2017emulating,barmann2018online} and \citet{sakaue2025revisiting_online_learning_approach,sakaue2025online} give regret upper bounds, and also give guarantees in the offline setting via online-to-batch conversion (a conversion that outputs, e.g., the average of the iterates of online learning and derives an upper bound on the offline loss from a regret upper bound).
Among these, \citet{sakaue2025online} give a regret upper bound that adapts to the cumulative suboptimality of the agent's actions.
Moreover, although it does not minimize the suboptimality loss, \citet{besbes2025contextual} proposed a method that handles offline and online learning in a setting that does not assume the structure of the forward problem.
The theoretical guarantees of \citet{kitaoka2024exact,kitaoka2026explicit,kitaoka2025inverse} are limited to the setting where the agent's actions are optimal.

\paragraph{Details of generalization error analysis for inverse optimization}
\citet{kitaoka2025inverse} observed that the ``distance'' between parameters that naturally arises in the generalization error analysis of inverse optimization is in general not a metric but a pseudometric, and extended the Dudley-type bound via covering numbers of sub-Gaussian processes to pseudometric spaces for losses that are Lipschitz with respect to a pseudometric.
Applying this to inverse optimization for MILP that learns the objective and the constraints, they bounded the expectation of the suboptimality loss at the learned weight by $O(\sqrt{d/N})$, where $d$ is the dimension of the features.
\citet{fatemi2026tight} treated noiseless inverse optimization, in which the agent's action is the unique optimal solution under the true weight, regarded the consistency constraint that each observation imposes on the weight as a constraint of a scenario program, and applied the result of \citet{campi2008exact}.
For the estimator that minimizes a convex function under the consistency constraints, they bounded the probability that the true optimal solution is not contained in the set of optimal solutions under the learned weight (the set-level mismatch) by $2(d+\log(1/\beta))/N$ with probability at least $1-\beta$ (\citealp[Proposition~2.1]{fatemi2026tight}), and, under an assumption of covariance diversity or for the incenter estimator, bounded the probability that the action selected under the learned weight by a fixed measurable rule that chooses one optimal solution when there are several does not coincide with the true optimal solution (the action-level mismatch) at the same order (\citealp[Theorems~2.2 and~2.3]{fatemi2026tight}).
Furthermore, for the class of estimators that minimize a linear objective, they constructed an instance showing that these orders are tight (\citealp[Theorem~3.1]{fatemi2026tight}).
\citet[Appendix~A.4]{fatemi2026tight} also point out that, for an arbitrary consistent estimator, only a bound with a $\log$ factor based on the VC dimension is available.
\Cref{thm:gen_upper_general} of this paper gives a uniform bound of this type for outperformance inverse optimization, in which the agent's actions may be suboptimal, and \Cref{thm:gen_lower} is obtained by embedding the instance of Theorem~3.1 of \citet{fatemi2026tight} into the setting of this paper.

\section{An action that cannot be outperformed}
\label{sec:appendix_unsupported}

We give a concrete example in which no weight makes the OS loss $0$.
A point $x^\prime\in\mathcal{X}(s)$ is called \emph{efficient} if there is no $x\in\mathcal{X}(s)$ with $x\geq x^\prime$ and $x\neq x^\prime$. An efficient point is called \emph{supported} if it is an optimal solution of $\max_{x\in\mathcal{X}(s)}\langle\theta,x\rangle$ for some $\theta\in\mathbb{R}^d_{\geq0}\setminus\{0\}$, and \emph{unsupported} otherwise.
If $x^\prime$ is efficient, then $\mathcal{X}_{\geq x^\prime}(s)=\{x^\prime\}$, so $\ell^{\optdex}(\theta;x^\prime,s)=0$ if and only if $x^\prime$ is an optimal solution under $\theta$. Hence, if $x^\prime$ is unsupported and $0\notin\Theta$, we have $\ell^{\optdex}(\theta;x^\prime,s)>0$ for every $\theta\in\Theta$.

Let $d=2$, $\Theta=\Delta^1$, $\mathcal{X}(s)=\{(3,0),(0,3),(1,1)\}$, and $x^\prime=(1,1)$. Since $\mathcal{X}(s)$ is a union of three bounded polyhedra each consisting of a single point, it satisfies \Cref{assu:margin}(2). Neither $(3,0)$ nor $(0,3)$ dominates $x^\prime$, so $x^\prime$ is efficient and $\mathcal{X}_{\geq x^\prime}(s)=\{x^\prime\}$. For $\theta=(\theta_1,\theta_2)\in\Delta^1$, since $\max\{\theta_1,\theta_2\}\ge1/2$ and $\theta_1+\theta_2=1$, we have
\[
    \ell^{\optdex}(\theta;x^\prime,s)
    =\max\{3\theta_1,3\theta_2,\theta_1+\theta_2\}-(\theta_1+\theta_2)
    =3\max\{\theta_1,\theta_2\}-1
    \ge\frac{1}{2} ,
\]
with equality if and only if $\theta=(1/2,1/2)$. Hence $x^\prime$ is unsupported, and the OS loss at this state is at least $1/2$ for every $\theta\in\Theta$. In particular, if $(s,x^\prime)$ is one of $N$ observations, the empirical mean of the OS loss is at least $1/(2N)$ for every $\theta\in\Theta$, and \cref{eq:outperformance_IOP} cannot be attained.
Thus, even an action that is not dominated by any feasible solution can be impossible to outperform by an optimal solution under a single linear objective.

\section{Weak derivative of the OS loss}
\label{sec:weak_derivative_OA_loss}
\begin{proposition}[Weak derivative in the sense of Schwartz distributions]
\label{prop:max_linear_weak_derivative_distribution_F}
Let $W\subset\mathbb{R}^d$ be an open set and let $\mathcal{Z}\subset\mathbb{R}^d$ be nonempty and compact.
Define
\[
F(\theta):=\max_{x\in\mathcal{Z}} \langle \theta ,  x \rangle .
\]
Moreover, for each $\theta\in W$, take
\[
x^*(\theta)\in\argmax_{x\in\mathcal{Z}} \langle \theta ,  x \rangle .
\]
Then, for each $i=1,\ldots,d$, in the sense of Schwartz distributions $\mathcal{D}'(W)$~\citep[cf.][]{kunzinger2019theory}, we have
\[
\partial_i F = x_i^*\quad\text{in }\mathcal{D}'(W) .
\]
That is, for any test function $\psi\in C_c^\infty(W)$, we have
\[
\langle \partial_i F,\psi\rangle
:=
-\int_{W}F(\theta)\,\partial_i\psi(\theta)\,d\theta
=
\int_{W}x_i^*(\theta)\,\psi(\theta)\,d\theta .
\]
\end{proposition}

\begin{proof}
In what follows, let $L_F:=\max_{x\in\mathcal{Z}}\|x\|$ (since $\mathcal{Z}$ is nonempty and compact, $0\le L_F<\infty$).

\textbf{Step 0. $F$ is a finite-valued convex function on $\mathbb{R}^d$ with Lipschitz constant $L_F$.}
Since $\mathcal{Z}$ is nonempty and compact and $x\mapsto\langle\theta,x\rangle$ is continuous, the maximum is attained at each $\theta\in\mathbb{R}^d$, and $F$ is well defined as a finite-valued function on $\mathbb{R}^d$. Since $F$ is the supremum of the family of linear functions $\{\langle\cdot,x\rangle\}_{x\in\mathcal{Z}}$, it is convex. Moreover, for any $\theta,\theta'\in\mathbb{R}^d$ and $x\in\mathcal{Z}$, we have
\[
\langle\theta',x\rangle=\langle\theta,x\rangle+\langle\theta'-\theta,x\rangle\le F(\theta)+L_F\|\theta'-\theta\| ,
\]
so taking the maximum over $x\in\mathcal{Z}$ yields $F(\theta')\le F(\theta)+L_F\|\theta'-\theta\|$. Exchanging the roles of $\theta$ and $\theta'$,
\begin{equation}
\label{eq:F_lipschitz}
|F(\theta')-F(\theta)|\le L_F\|\theta'-\theta\|\qquad(\forall\theta,\theta'\in\mathbb{R}^d)
\end{equation}
follows. In particular, $F$ is continuous and $F\in L^1_{\mathrm{loc}}(W)$.

\textbf{Step 1. $F$ is differentiable a.e., and $\nabla F(\theta)=x^*(\theta)$ at such points.}
By \cref{eq:F_lipschitz} and Rademacher's theorem~\citep[cf.][Theorem~10.8~(ii)]{villani2009optimal}, $F$ is differentiable in the classical sense almost everywhere on $W$.
Take a point of differentiability $\theta\in W$. Since $x^*(\theta)\in\argmax_{x\in\mathcal{Z}}\langle\theta,x\rangle$, for any $\theta'\in\mathbb{R}^d$,
\[
F(\theta')\ge\langle\theta',x^*(\theta)\rangle
=\langle\theta,x^*(\theta)\rangle+\langle\theta'-\theta,x^*(\theta)\rangle
=F(\theta)+\langle\theta'-\theta,x^*(\theta)\rangle ,
\]
that is, $x^*(\theta)$ is a subgradient of $F$ at $\theta$:
\begin{equation}
\label{eq:support_theta}
F(\theta')-F(\theta)\ge\langle\theta'-\theta,x^*(\theta)\rangle\qquad(\forall\theta'\in\mathbb{R}^d) .
\end{equation}
Take any $h\in\mathbb{R}^d$. Since $W$ is open, $\theta+th\in W$ for sufficiently small $t>0$; setting $\theta'=\theta+th$ in \cref{eq:support_theta} and dividing by $t>0$ gives
\[
\frac{F(\theta+th)-F(\theta)}{t}\ge\langle h,x^*(\theta)\rangle .
\]
By differentiability at $\theta$, letting $t\downarrow0$ yields
\[
\langle h,\nabla F(\theta)\rangle\ge\langle h,x^*(\theta)\rangle .
\]
Replacing $h$ with $-h$ gives the reverse inequality, so
\[
\langle h,\nabla F(\theta)\rangle=\langle h,x^*(\theta)\rangle\qquad(\forall h\in\mathbb{R}^d) ,
\]
that is, $\nabla F(\theta)=x^*(\theta)$. Hence we have
\begin{equation}
\label{eq:grad_eq_xstar_ae}
\nabla F(\theta)=x^*(\theta)\quad\text{a.e. }\theta\in W .
\end{equation}
Note that $\theta\mapsto\nabla F(\theta)$ is Lebesgue measurable (as the pointwise limit of difference quotients on the set of points of differentiability), and by the completeness of the Lebesgue measure and \cref{eq:grad_eq_xstar_ae}, $x^*$ is also Lebesgue measurable. Moreover, since $\|x^*(\theta)\|\le L_F$, we have $x^*\in L^\infty(W;\mathbb{R}^d)\subset L^1_{\mathrm{loc}}(W;\mathbb{R}^d)$, and the integral on the right-hand side of the claim converges absolutely.

\textbf{Step 2. Identification of the distributional derivative via convergence of difference quotients.}
Take any $i\in\{1,\ldots,d\}$ and $\psi\in C_c^\infty(W)$, and let $\mathcal{K}:=\operatorname{supp}\psi$ (a compact subset of $W$). Since $W$ is open,
\[
\delta:=\operatorname{dist}(\mathcal{K},\mathbb{R}^d\setminus W)>0
\]
(when $W=\mathbb{R}^d$, read $\delta:=+\infty$).
Let $\mathcal{K}_\delta:=\{\theta\in\mathbb{R}^d \, | \, \operatorname{dist}(\theta,\mathcal{K})\le\delta/2\}$. Then $\mathcal{K}_\delta$ is a compact subset of $W$, and if $0<|t|<\delta/2$, then $\mathcal{K}+te_i\subset \mathcal{K}_\delta\subset W$ (where $e_i$ is the $i$th standard basis vector).

First, for $0<|t|<\delta/2$, since the support of $\psi$ is contained in $\mathcal{K}$, the change of variables $\theta\mapsto\theta-te_i$ gives
\[
\int_{W}F(\theta+te_i)\,\psi(\theta)\,d\theta
=
\int_{\mathcal{K}+te_i}F(\theta)\,\psi(\theta-te_i)\,d\theta
=
\int_{W}F(\theta)\,\psi(\theta-te_i)\,d\theta ,
\]
and therefore we have
\begin{equation}
\label{eq:difference_quotient_ibp}
\int_{W}\frac{F(\theta+te_i)-F(\theta)}{t}\,\psi(\theta)\,d\theta
=
\int_{W}F(\theta)\,\frac{\psi(\theta-te_i)-\psi(\theta)}{t}\,d\theta .
\end{equation}

For the right-hand side of \cref{eq:difference_quotient_ibp}, since $\psi\in C_c^\infty(W)$, the convergence
$\bigl(\psi(\cdot-te_i)-\psi(\cdot)\bigr)/t\to-\partial_i\psi$ ($t\to0$) is uniform, and for $0<|t|<\delta/2$ the support of the integrand is contained in $\mathcal{K}_\delta$. Since $F$ is continuous (Step 0), it is bounded on the compact set $\mathcal{K}_\delta$, and therefore
\[
\lim_{t\to0}\int_{W}F(\theta)\,\frac{\psi(\theta-te_i)-\psi(\theta)}{t}\,d\theta
=
-\int_{W}F(\theta)\,\partial_i\psi(\theta)\,d\theta .
\]

For the left-hand side of \cref{eq:difference_quotient_ibp}, by Step 1, for almost every $\theta\in W$,
\[
\lim_{t\to0}\frac{F(\theta+te_i)-F(\theta)}{t}=\partial_i F(\theta)=x_i^*(\theta) ,
\]
and by \cref{eq:F_lipschitz},
\[
\left|\frac{F(\theta+te_i)-F(\theta)}{t}\,\psi(\theta)\right|\le L_F\,|\psi(\theta)| \in L^1(W) ,
\]
so the integrand is dominated by an integrable function independent of $t$. Hence, by Lebesgue's dominated convergence theorem,
\[
\lim_{t\to0}\int_{W}\frac{F(\theta+te_i)-F(\theta)}{t}\,\psi(\theta)\,d\theta
=
\int_{W}x_i^*(\theta)\,\psi(\theta)\,d\theta .
\]

Combining these by letting $t\to0$ in \cref{eq:difference_quotient_ibp}, for any $\psi\in C_c^\infty(W)$, we have
\[
\langle\partial_i F,\psi\rangle
=
-\int_{W}F(\theta)\,\partial_i\psi(\theta)\,d\theta
=
\int_{W}x_i^*(\theta)\,\psi(\theta)\,d\theta .
\]
That is, $\partial_i F=x_i^*$ in $\mathcal{D}'(W)$.
\end{proof}

\begin{proof}[Proof of \Cref{prop:gredient_OA_loss}]
    Apply \Cref{prop:max_linear_weak_derivative_distribution_F} with $W=\mathbb{R}^d$ and $\mathcal{Z}=\mathcal{X}(s)$, and with $W=\mathbb{R}^d$ and $\mathcal{Z}=\mathcal{X}_{\geq x^\prime}(s)$ (both are nonempty and compact), and take the difference of the two terms in \cref{eq:def_optdex}.
\end{proof}

\section{Pseudocode of the rules for selecting evaluation points in the inner solver: Step 12 of \Cref{alg:dc_inner}}
\label{sec:appendix_inner_rules}

In this section, as concrete forms of Step 12 of \Cref{alg:dc_inner}, we give the pseudocode of the cutting-plane method \citep{kelley1960cutting}, the proximal bundle method \citep{kiwiel1990proximity}, and the level method \citep{lemarechal1995new}. Furthermore, in \Cref{subsec:rule_cpfinish}, we give the pseudocode of the cutting-plane finishing of \Cref{theo:cpfinish_exact}, which starts with one of these and switches to the cutting-plane method for finishing. Below, $\mathrm{UB}$ denotes the minimum of the values of the convex subproblem evaluated so far in \Cref{alg:dc_inner} (an upper bound), $\mathrm{LB}$ denotes the minimum value of the model of cuts (a lower bound), and $\epsilon_{\mathrm{in}}$ denotes the inner tolerance. The cut $g_j$, the cut set $\mathcal{G}$, the solution pool $\mathcal{V}$, the maximum number of iterations $J$, the point $\hat{\theta}$ with the smallest evaluated value, and $\mathrm{converged}$ are defined in \Cref{alg:dc_inner} (\Cref{subsec:dc_algorithm}), and inherited cuts and cut inheritance are defined in \Cref{subsec:dc_cut_pool}. The selection step refers to Step 12 of \Cref{alg:dc_inner}.
Each is written as a procedure that is called at iteration $j$ of \Cref{alg:dc_inner} immediately after the cut $g_j$ is added to $\mathcal{G}$ and $\mathrm{UB},\mathrm{LB}$ are updated (after passing the stopping test), and that returns $\theta^{j+1}$ (only the proximal bundle method also returns the quantities carried over to the next call).

We fix the notation. Let $\mathcal{G}_j$ be the cut set at the selection step of iteration $j$ (including the inherited cuts and $g_0,\ldots,g_j$), and define the cutting-plane model
\begin{equation}
    \hat{\ell}_j(\theta):=\max_{g\in\mathcal{G}_j}\langle\theta,g\rangle .
    \label{eq:def_cp_model}
\end{equation}
By \Cref{prop:dc_cut_validity}, $\hat{\ell}_j\le\bar{\ell}^{\sub}(\cdot\,;y^\prime)$, and at an evaluated point $\theta^{i}$ ($i\le j$), since $g_i\in\mathcal{G}_j$, we have $\hat{\ell}_j(\theta^{i})=\bar{\ell}^{\sub}(\theta^{i};y^\prime)=\langle\theta^{i},g_i\rangle$. The $\mathrm{LB}$ of \Cref{alg:dc_inner} equals $\min_{\theta\in\Theta}\hat{\ell}_j(\theta)$. Let $\|\cdot\|$ denote the Euclidean norm.
The subproblems of the three rules differ only in the objective over the common polyhedron
\begin{equation}
    P_j:=\bigl\{(\theta,\tau)\in\mathbb{R}^d\times\mathbb{R} \ \bigm|\ \theta\in\Theta,\ \langle\theta,g\rangle\le\tau\ (\forall g\in\mathcal{G}_j)\bigr\}
    \label{eq:def_cp_epigraph}
\end{equation}
(the epigraph of the model $\hat{\ell}_j$ over $\Theta$); since $\Theta$ is a polyhedron by \Cref{assu:margin}(1), they are an LP for the cutting-plane method and a convex quadratic program (QP) for the proximal bundle method and the level method.

\subsection{Cutting-plane method}
\label{subsec:rule_kelley}

The cutting-plane method \citep{kelley1960cutting} takes a minimizer of the model $\hat{\ell}_j$ over $\Theta$ as the next evaluation point (\Cref{alg:rule_kelley}).
Since this minimization is exactly the LP that \Cref{alg:dc_inner} solves to compute $\mathrm{LB}$, no additional solve is needed, and the $\theta$ component of the optimal solution of that LP can be reused.
When the LP has multiple optimal solutions, any of them may be chosen (the proof of \Cref{lem:dc_inner_finite} uses only optimality).
Finite exactness and an upper bound on the number of iterations under this rule are given in \Cref{lem:dc_inner_finite}, and its use as the rule after the switch in cutting-plane finishing (\Cref{subsec:rule_cpfinish}) is given in \Cref{theo:cpfinish_exact}.

\begin{algorithm}[th]
    \caption{Rule for selecting evaluation points: cutting-plane method}
    \label{alg:rule_kelley}
    \begin{algorithmic}[1]
        \REQUIRE {Cut set} $\mathcal{G}_j$
        \STATE $(\theta^{\mathrm{LP}},\tau^{\mathrm{LP}})\gets$ optimal solution and optimal value of the LP $\min\{\,\tau \mid (\theta,\tau)\in P_j\,\}$
        \STATE $\theta^{j+1}\gets\theta^{\mathrm{LP}}$
        \RETURN $\theta^{j+1}$
    \end{algorithmic}
\end{algorithm}

\subsection{Proximal bundle method}
\label{subsec:rule_bundle}

The proximal bundle method \citep{kiwiel1990proximity} maintains a stability center $\theta^{\mathrm{c}}\in\Theta$ and takes as the candidate point the optimal solution of
\begin{equation}
    \min\Bigl\{\,\tau+\frac{u}{2}\|\theta-\theta^{\mathrm{c}}\|^2 \ \Bigm|\ (\theta,\tau)\in P_j\,\Bigr\} ,
    \label{eq:bundle_qp}
\end{equation}
which adds the proximal term $\frac{u}{2}\|\theta-\theta^{\mathrm{c}}\|^2$ ($u>0$ is the proximal weight) to the model (\Cref{alg:rule_bundle}).
Since the value $\bar{\ell}^{\sub}(\theta^{j+1};y^\prime)$ of the candidate point $\theta^{j+1}$ is evaluated at the next iteration of \Cref{alg:dc_inner}, the test of whether to update the center (the descent test) is performed at the beginning of the selection step of the next iteration.
The predicted decrease $\nu:=f^{\mathrm{c}}-\hat{\ell}_j(\theta^{j+1})$ (where $f^{\mathrm{c}}:=\bar{\ell}^{\sub}(\theta^{\mathrm{c}};y^\prime)$ is the value at the center) is carried over to the next call together with the candidate point; there, if the actual decrease $f^{\mathrm{c}}-\bar{\ell}^{\sub}(\theta^{j+1};y^\prime)$ is at least $\rho\,\nu$ ($\rho\in(0,1)$ is the descent parameter), the center is moved to the candidate point (serious step). Otherwise, the center is kept and only the cut of the candidate point is added (null step).
Since $(\theta^{\mathrm{c}},\tau)=(\theta^{\mathrm{c}},f^{\mathrm{c}})\in P_j$ is a feasible solution of \cref{eq:bundle_qp} with objective value $f^{\mathrm{c}}$, we have $\hat{\ell}_j(\theta^{j+1})\le f^{\mathrm{c}}$, that is, always $\nu\ge0$.

We make two remarks specific to this framework.
First, the $\hat{\theta}$ of \Cref{alg:dc_inner} (the point with the smallest evaluated value) and the center $\theta^{\mathrm{c}}$ differ in general (a null-step point may update $\mathrm{UB}$). Always $\mathrm{UB}\le f^{\mathrm{c}}$.
Second, the stopping test specific to the proximal bundle method ($\nu\le\epsilon$, where $\epsilon>0$ is a tolerance) is replaced by the stopping test $\mathrm{UB}-\mathrm{LB}\le\epsilon_{\mathrm{in}}$ of \Cref{alg:dc_inner}. Indeed, when the selection step is reached, $\mathrm{UB}-\mathrm{LB}>\epsilon_{\mathrm{in}}\ge0$, and in particular $\mathrm{LB}<\mathrm{UB}\le f^{\mathrm{c}}=\hat{\ell}_j(\theta^{\mathrm{c}})$, so $\theta^{\mathrm{c}}$ is not a minimizer of the model $\hat{\ell}_j$ over $\Theta$. Since $\theta^{\mathrm{c}}$ being a minimizer of the model with the proximal term \cref{eq:bundle_qp} is equivalent to its being a minimizer of $\hat{\ell}_j$ (because the gradient of the proximal term vanishes at $\theta^{\mathrm{c}}$), we have $\theta^{j+1}\ne\theta^{\mathrm{c}}$, and since the optimal value of \cref{eq:bundle_qp} is at most $f^{\mathrm{c}}$, $\hat{\ell}_j(\theta^{j+1})<f^{\mathrm{c}}$, that is, $\nu>0$ holds. Conversely, in the situation where $\nu=0$ ($\theta^{\mathrm{c}}$ is a minimizer of the model), $\mathrm{LB}=f^{\mathrm{c}}\ge\mathrm{UB}$, so \Cref{alg:dc_inner} has already stopped at the stopping test of the same iteration.
Moreover, since cuts are never discarded owing to cut inheritance, the bundle compression (aggregate subgradient) of \citet{kiwiel1990proximity} is not used.
The proximal weight $u$ is updated according to the proximity control rule of \citet{kiwiel1990proximity} ($u$ is not decreased at null steps, and the lower bound $u_{\min}>0$ is maintained; keeping $u$ constant is also allowed).

\begin{algorithm}[th]
    \caption{Rule for selecting evaluation points: proximal bundle method}
    \label{alg:rule_bundle}
    \begin{algorithmic}[1]
        \REQUIRE {Cut set} $\mathcal{G}_j$, current evaluation point $\theta^{j}$ and its value $\bar{\ell}^{\sub}(\theta^{j};y^\prime)=\langle\theta^{j},g_j\rangle$, descent parameter $\rho\in(0,1)$, initial proximal weight $u_0$ and its lower bound $u_{\min}$ ($u_0\ge u_{\min}>0$)
        \REQUIRE Quantities carried over from the call at iteration $j-1$ (undefined for $j=0$): stability center $\theta^{\mathrm{c}}\in\Theta$, its value $f^{\mathrm{c}}=\bar{\ell}^{\sub}(\theta^{\mathrm{c}};y^\prime)$, proximal weight $u\ge u_{\min}$, predicted decrease $\nu=f^{\mathrm{c}}-\hat{\ell}_{j-1}(\theta^{j})>0$
        \IF{$j=0$}
            \STATE $\theta^{\mathrm{c}}\gets\theta^{0}$, $f^{\mathrm{c}}\gets\langle\theta^{0},g_0\rangle$, $u\gets u_0$ (initialization of the center)
        \ELSIF{$\langle\theta^{j},g_j\rangle\le f^{\mathrm{c}}-\rho\,\nu$}
            \STATE $\theta^{\mathrm{c}}\gets\theta^{j}$, $f^{\mathrm{c}}\gets\langle\theta^{j},g_j\rangle$ (serious step: update of the center)
            \STATE $u\gets$ update by the proximity control rule (\citealp{kiwiel1990proximity})
        \ELSE
            \STATE $u\gets$ update by the proximity control rule (without decreasing $u$) (null step: the center is kept)
        \ENDIF
        \STATE $(\theta^{j+1},\tau_{j+1})\gets$ optimal solution of the QP \cref{eq:bundle_qp} (minimization of the model with the proximal term; $\tau_{j+1}=\hat{\ell}_j(\theta^{j+1})$)
        \STATE $\nu\gets f^{\mathrm{c}}-\tau_{j+1}$ (predicted decrease; $\nu>0$)
        \RETURN $\theta^{j+1}$, and $\theta^{\mathrm{c}},f^{\mathrm{c}},u,\nu$ carried over to the call at iteration $j+1$
    \end{algorithmic}
\end{algorithm}

\subsection{Level method}
\label{subsec:rule_level}

The level method \citep{lemarechal1995new} sets a level between the lower bound $\mathrm{LB}$ and the upper bound $\mathrm{UB}$,
\begin{equation}
    \mathrm{lev}_j:=\mathrm{LB}+\lambda\,(\mathrm{UB}-\mathrm{LB})
    \qquad(\lambda\in(0,1)\text{ is the level parameter}),
    \label{eq:def_level}
\end{equation}
and takes the projection of the current evaluation point $\theta^{j}$ onto the level set $\{\theta\in\Theta \mid \hat{\ell}_j(\theta)\le\mathrm{lev}_j\}$ of the model as the next evaluation point (\Cref{alg:rule_level}).
The level set is nonempty because the optimal solution $\theta^{\mathrm{LP}}$ of the LP attaining $\mathrm{LB}$ satisfies $\hat{\ell}_j(\theta^{\mathrm{LP}})=\mathrm{LB}\le\mathrm{lev}_j$, and the projection is the unique optimal solution of
\begin{equation}
    \min\Bigl\{\,\frac{1}{2}\|\theta-\theta^{j}\|^2 \ \Bigm|\ \theta\in\Theta,\ \langle\theta,g\rangle\le\mathrm{lev}_j\ (\forall g\in\mathcal{G}_j)\,\Bigr\}
    \label{eq:level_qp}
\end{equation}
(the feasible region is the slice of $P_j$ at $\tau=\mathrm{lev}_j$).
Since $\mathrm{UB},\mathrm{LB}$ are the quantities maintained by \Cref{alg:dc_inner}, this rule carries no quantities over across calls.
The number of iterations required to reach a certified gap $\mathrm{UB}-\mathrm{LB}\le\epsilon$ (where $\epsilon>0$ is a target accuracy) is discussed in \Cref{rem:level_complexity} (see also the same place for a remark on using inherited cuts).

\begin{algorithm}[th]
    \caption{Rule for selecting evaluation points: level method}
    \label{alg:rule_level}
    \begin{algorithmic}[1]
        \REQUIRE {Cut set} $\mathcal{G}_j$, current evaluation point $\theta^{j}$, upper bound $\mathrm{UB}$, lower bound $\mathrm{LB}$, level parameter $\lambda\in(0,1)$
        \STATE $\mathrm{lev}_j\gets\mathrm{LB}+\lambda\,(\mathrm{UB}-\mathrm{LB})$ (setting the level; \cref{eq:def_level})
        \STATE $\theta^{j+1}\gets$ optimal solution of the QP \cref{eq:level_qp} (projection of $\theta^{j}$ onto the level set $\{\theta\in\Theta \mid \hat{\ell}_j(\theta)\le\mathrm{lev}_j\}$; nonempty)
        \RETURN $\theta^{j+1}$
    \end{algorithmic}
\end{algorithm}

\subsection{Cutting-plane finishing}
\label{subsec:rule_cpfinish}

The cutting-plane finishing of \Cref{theo:cpfinish_exact} can also be written as a selection rule of Step 12 (\Cref{alg:rule_cpfinish}).
Under the guarantee of \Cref{theo:cpfinish_exact}, \Cref{alg:dc_inner} is run with tolerance $0$ and $J=\infty$, and the switching threshold $\epsilon_{\mathrm{sw}}>0$ is held by the rule (since $\epsilon_{\mathrm{sw}}$ is a different quantity from the tolerance $\epsilon_{\mathrm{in}}$ of \Cref{alg:dc_inner}, we use a separate symbol).
The rule distinguishes cases by comparing $\mathrm{UB}-\mathrm{LB}$ with $\epsilon_{\mathrm{sw}}$: if $\mathrm{UB}-\mathrm{LB}>\epsilon_{\mathrm{sw}}$, it applies a selection rule $\mathsf{R}$ (e.g., \Cref{alg:rule_kelley}, \ref{alg:rule_bundle}, \ref{alg:rule_level}); otherwise, it applies the cutting-plane method (\Cref{alg:rule_kelley}).
We call the iterations at which the former condition holds Phase 1, and those at which the latter holds Phase 2 (they correspond to the two phases of \Cref{theo:cpfinish_exact}).

This case distinction gives a two-phase procedure because $\mathrm{UB}-\mathrm{LB}$ is nonincreasing in the iteration $j$. Indeed, by Steps 5--7 of \Cref{alg:dc_inner}, $\mathrm{UB}$ is nonincreasing as the minimum of the evaluated values, and since the cut set $\mathcal{G}$ only grows in Step 4, $\mathrm{LB}=\min_{\theta\in\Theta}\hat{\ell}_j(\theta)$ in Step 8 is nondecreasing. Therefore, once $\mathrm{UB}-\mathrm{LB}\le\epsilon_{\mathrm{sw}}$ holds, it continues to hold thereafter, and the switch of phases occurs only once and never goes back.

The guarantee is as stated in \Cref{theo:cpfinish_exact}. If we take $\epsilon_{\mathrm{sw}}<\delta(y^\prime)$ ($\delta(y^\prime)$ is the submodel gap, \cref{eq:def_submodel_gap}), then at the time of the switch the lower bound is determined exactly as $\mathrm{LB}=m_{\sub}(y^\prime)$ ($m_{\sub}(y^\prime)$ is the minimum value of the convex subproblem of the branch $y^\prime$, \cref{eq:def_branch_min}) (\Cref{lem:lb_snap}), and since the cuts only increase thereafter, this value is maintained. In Phase 2, the evaluation point is a minimizer of the model, so $\mathrm{UB}=\mathrm{LB}=m_{\sub}(y^\prime)$ is reached in finitely many iterations, the stopping test with tolerance $0$ is satisfied, and an exact minimizer is returned.
Note that since the cut set $\mathcal{G}$ and the solution pool $\mathcal{V}$ are carried over unchanged within \Cref{alg:dc_inner} across the phases, the switch requires no additional calls to the forward-problem oracle. When the proximal bundle method is used as the rule $\mathsf{R}$, the quantities carried over by \Cref{alg:rule_bundle} may be discarded at the switch.

\begin{algorithm}[th]
    \caption{Rule for selecting evaluation points: cutting-plane finishing (\Cref{theo:cpfinish_exact})}
    \label{alg:rule_cpfinish}
    \begin{algorithmic}[1]
        \REQUIRE {Cut set} $\mathcal{G}_j$, current evaluation point $\theta^{j}$, upper bound $\mathrm{UB}$, lower bound $\mathrm{LB}$, switching threshold $\epsilon_{\mathrm{sw}}>0$, termination tolerance $\epsilon_{\mathrm{term}}\ge0$, selection rule $\mathsf{R}$ (e.g., \Cref{alg:rule_kelley}, \ref{alg:rule_bundle}, \ref{alg:rule_level})
        \IF{$\mathrm{UB}-\mathrm{LB}>\epsilon_{\mathrm{sw}}$}
            \STATE $\theta^{j+1}\gets$ the point obtained by applying the rule $\mathsf{R}$
        \ELSE
            \STATE (finishing by the cutting-plane method) $\theta^{j+1}\gets$ the point obtained by applying \Cref{alg:rule_kelley}
            \IF{$\|\theta^{j+1}-\theta^{i}\|_2\le\epsilon_{\mathrm{term}}$ for some evaluated point $\theta^{i}$ ($i\le j$)}
                \STATE stop \Cref{alg:dc_inner} with $\mathrm{converged}=\mathrm{TRUE}$ without evaluating $\theta^{j+1}$
            \ENDIF
        \ENDIF
        \RETURN $\theta^{j+1}$
    \end{algorithmic}
\end{algorithm}

In Phase 2, if the minimizer of the LP coincides with an evaluated point within $\epsilon_{\mathrm{term}}$, the inner loop is stopped with $\mathrm{converged}=\mathrm{TRUE}$ without evaluating that point (in this case $\mathrm{UB}-\mathrm{LB}=0$ does not necessarily hold). This is because re-evaluating the same point yields no new cut, and what remains in $\mathrm{UB}-\mathrm{LB}$ is a difference of the order of the rounding error of the loss values, so we avoid continuing the iterations due to rounding errors.
When $\epsilon_{\mathrm{term}}=0$ and the LPs are solved exactly, this test is never triggered: if $\theta^{j+1}=\theta^{i}$ for an evaluated point $\theta^{i}$ ($i\le j$), then the cut $g_i$ obtained at $\theta^{i}$ belongs to $\mathcal{G}$, so $\mathrm{LB}=\max_{g\in\mathcal{G}}\langle\theta^{j+1},g\rangle\ge\langle\theta^{i},g_i\rangle=\bar{\ell}^{\sub}(\theta^{i};y^\prime)\ge\mathrm{UB}$, and \Cref{alg:dc_inner} has already stopped by $\mathrm{UB}-\mathrm{LB}\le\epsilon_{\mathrm{in}}$. \Cref{theo:cpfinish_exact} guarantees an exact minimizer under $\epsilon_{\mathrm{term}}=0$ (the positive $\epsilon_{\mathrm{term}}$ used in the experiments is a safeguard against rounding errors; \Cref{subsec:appendix_implementation}).

When \Cref{alg:dc_outer} is run with \Cref{alg:rule_cpfinish} as the rule of Step 12, as long as we take $\epsilon_{\mathrm{sw}}<\delta_{\mathrm{in}}:=\min_{y^\prime\in\mathcal{Y}^{\mathrm{all}}_{\geq}}\delta(y^\prime)$ and $\epsilon_{\mathrm{term}}=0$, and as long as Phase 1 ends after finitely many iterations under the rule $\mathsf{R}$ (which always holds if $\mathsf{R}$ is the cutting-plane method), the inner loop is solved exactly at every outer iteration, so the assumption of \Cref{prop:dc_finite_convergence} is satisfied (\Cref{rem:gap_outlook}).

\section{DC optimization}
\label{sec:appendix_dc_opt}

This section describes the DC optimization method that solves \cref{eq:empirical_outperformance_IOP} on the basis of the decomposition into branches in \Cref{prop:dc_branch_decomp} and the DC structure in \Cref{rem:OA_loss_is_DC}.
Its skeleton is a double loop: the outer loop is simplified DCA \citep{pham1997convex}, which linearizes the concave part and reduces the problem to convex subproblems, and the inner loop solves the convex subproblems by a cutting-plane-type solver, typified by the cutting-plane method \citep{kelley1960cutting}. In addition, the method has a mechanism that accumulates the feasible solutions returned by the forward-problem oracle as global lower-bounding cuts and reuses them across iterations and runs (cut inheritance). Cut inheritance instantiates the idea of the proximal bundle method for DC programming of \citet{oliveira2019proximal} in accordance with the structure of our problem.
Below, \Cref{subsec:dc_notation} gives the notation and assumptions, \Cref{subsec:dc_decomp} the relation between branches and DCA, \Cref{subsec:dc_algorithm} the pseudocode, \Cref{subsec:dc_cut_pool} cut inheritance, \Cref{subsec:dc_literature} the underlying literature, \Cref{sec:theory_in_dc_opt} finite convergence and local optimality, and \Cref{subsec:dc_theory} the theoretical analysis.

\subsection{Notation and assumptions}
\label{subsec:dc_notation}

Throughout this section, we assume \Cref{assu:margin}. In the appendices below, a point $u$ of a convex set $C$ is an \emph{extreme point} of $C$ if $u=\lambda v+(1-\lambda)w$ with $v,w\in C$ and $\lambda\in(0,1)$ implies $v=w=u$.
For a state $s\in\mathcal{S}$ and $x^\prime\in\mathcal{X}(s)$, we set
\begin{equation}
    F_1(\theta;s):=\max_{x\in\mathcal{X}(s)}\langle\theta,x\rangle ,
    \qquad
    F_2(\theta;x^\prime,s):=\max_{x\in\mathcal{X}_{\geq x^\prime}(s)}\langle\theta,x\rangle .
    \label{eq:def_F1F2}
\end{equation}
Since $x^\prime\in\mathcal{X}_{\geq x^\prime}(s)$, the set $\mathcal{X}_{\geq x^\prime}(s)$ is nonempty, and it is compact as the intersection of the compact set $\mathcal{X}(s)$ and the closed set $\{x\in\mathbb{R}^d \mid x\ge x^\prime\}$; hence both maxima in \cref{eq:def_F1F2} are attained. By definition (\cref{eq:def_optdex}),
\[
    \ell^{\optdex}(\theta;x^\prime,s)=F_1(\theta;s)-F_2(\theta;x^\prime,s) .
\]

By the extreme-point oracle assumption (\Cref{assu:margin}(3)), $\mathcal{Y}(s^n)\subseteq\mathcal{X}(s^n)$ and $\mathcal{Y}_{\geq x^n}(s^n)\subseteq\mathcal{X}_{\geq x^n}(s^n)$ are finite sets. Furthermore, since the value returned by the oracle attains the maximum for every $\theta$, we have
\begin{equation}
    F_1(\theta;s^n)=\max_{y\in\mathcal{Y}(s^n)}\langle\theta,y\rangle ,
    \qquad
    F_2(\theta;x^n,s^n)=\max_{y\in\mathcal{Y}_{\geq x^n}(s^n)}\langle\theta,y\rangle .
    \label{eq:F1F2_finite_max}
\end{equation}
That is, $F_1$ and $F_2$ are maxima of finitely many linear functions (convex, piecewise linear, and positively homogeneous of degree $1$).

Among the branches introduced in the main text (\Cref{sec:dc_opt}), we deal with the tuples $y^\prime=(y^{\prime n})_{n=1}^N\in\mathcal{Y}^{\mathrm{all}}_{\geq}=\prod_{n=1}^N\mathcal{Y}_{\geq x^n}(s^n)$ consisting of values that the oracle can return. For $y\in\mathbb{R}^d$ and a branch $y^\prime\in\mathcal{Y}^{\mathrm{all}}_{\geq}$, we set
\begin{align}
    \ell^{\sub}(\theta;y,s) &:= F_1(\theta;s)-\langle\theta,y\rangle ,
    \label{eq:def_sub_loss}
    \\
    \bar{\ell}^{\sub}(\theta;y^\prime) &:= \frac{1}{N}\sum_{n=1}^N \ell^{\sub}(\theta;y^{\prime n},s^n) ,
    \qquad
    m_{\sub}(y^\prime) := \min_{\theta\in\Theta}\bar{\ell}^{\sub}(\theta;y^\prime) .
    \label{eq:def_branch_min}
\end{align}
We call $m_{\sub}(y^\prime)$ the \emph{branch minimum} of the branch $y^\prime$. The function $\ell^{\sub}$ has the same form as the suboptimality loss \citep{mohajerin2018data}. By \cref{eq:F1F2_finite_max}, each $\bar{\ell}^{\sub}(\cdot\,;y^\prime)$ is convex and piecewise linear, and by the compactness of $\Theta$, the minimum in \cref{eq:def_branch_min} is attained.

\subsection{\texorpdfstring{Relation between branches and DCA}{Relation between branches and DCA}}
\label{subsec:dc_decomp}

\begin{remark}[Relation to the DC structure]
    \label{rem:OA_loss_is_DC}
    The empirical mean of the OS loss $\bar{\ell}^{\optdex}$ is a DC function of $\theta\in\Theta$, that is, a function that can be written in the form convex function $-$ convex function. Indeed, by definition,
    \[
        \bar{\ell}^{\optdex}(\theta)
        =
        \frac{1}{N}\sum_{n=1}^N\max_{x\in\mathcal{X}(s^n)}\langle\theta,x\rangle
        -
        \frac{1}{N}\sum_{n=1}^N\max_{x\in\mathcal{X}_{\geq x^n}(s^n)}\langle\theta,x\rangle ,
    \]
    and both terms on the right-hand side are convex functions of $\theta$.
    By this observation, \cref{eq:empirical_outperformance_IOP} can be regarded as a DC optimization problem and connected to the literature on DC programming. The operation of linearizing the concave part $-\frac{1}{N}\sum_{n=1}^N\max_{x\in\mathcal{X}_{\geq x^n}(s^n)}\langle\theta,x\rangle$ at the dominating vertex (the solution returned by the dominance-constrained forward-problem oracle) $x^*(\theta;x^n,s^n)$ is nothing but the minimization with respect to $y^\prime$ in \Cref{prop:dc_branch_decomp}.
\end{remark}

\begin{proof}[Proof of \Cref{prop:dc_branch_decomp}]
    By \Cref{assu:margin}(2), $\mathcal{X}_{\geq x^n}(s^n)$ is nonempty and compact for each $n$, so the maximum in the second term of \cref{eq:def_optdex} is attained, and $-\max_{x\in\mathcal{X}_{\geq x^n}(s^n)}\langle\theta,x\rangle=\min_{y\in\mathcal{X}_{\geq x^n}(s^n)}\bigl(-\langle\theta,y\rangle\bigr)$. Since the minimizations for different $n$ can be carried out independently,
    \[
        \bar{\ell}^{\optdex}(\theta)
        =\frac{1}{N}\sum_{n=1}^N\Bigl[\max_{x\in\mathcal{X}(s^n)}\langle\theta,x\rangle+\min_{y\in\mathcal{X}_{\geq x^n}(s^n)}\bigl(-\langle\theta,y\rangle\bigr)\Bigr]
        =\min_{y^\prime\in\mathcal{X}^{\mathrm{all}}_{\geq}}\bar{\ell}^{\sub}(\theta;y^\prime) ,
    \]
    which gives the first equality in \cref{eq:dc_pointwise_min}. By \Cref{assu:margin}(3), for every $\theta$ the value returned by the oracle $x^*(\theta;x^n,s^n)\in\mathcal{Y}_{\geq x^n}(s^n)\subseteq\mathcal{X}_{\geq x^n}(s^n)$ attains the maximum, so $\mathcal{X}_{\geq x^n}(s^n)$ may be replaced by $\mathcal{Y}_{\geq x^n}(s^n)$ in the argument above, which gives the second equality. Since $\mathcal{Y}^{\mathrm{all}}_{\geq}$ is a finite set (\Cref{rem:assu_margin}) and each $m_{\sub}(y^\prime)$ is attained, \cref{eq:dc_min_min} follows by exchanging the $\min$ over $\theta$ and over $y^\prime$.
\end{proof}

The decomposition into branches was given in \Cref{prop:dc_branch_decomp} of the main text. Here we supplement the relation between DCA and branches.

DCA can be regarded as an iterative method that selects branches adaptively without enumerating them explicitly. If the concave part $-\frac{1}{N}\sum_{n=1}^N F_2(\,\cdot\,;x^n,s^n)$ is linearized at the current point $\theta^k$ by a subgradient, namely the dominating vertices $y^{(k),n}:=x^*(\theta^k;x^n,s^n)$, then the resulting convex subproblem is exactly $\min_{\theta\in\Theta}\bar{\ell}^{\sub}(\theta;y^{(k)})$, that is, the minimization of the empirical mean of the suboptimality loss with the branch $y^{(k)}=(y^{(k),n})_n$ fixed. A subgradient oracle for the convex subproblem is obtained immediately from the forward-problem oracle as
\[
    \partial_\theta \bar{\ell}^{\sub}(\theta;y^{(k)}) \ni \frac{1}{N}\sum_{n=1}^N\bigl[x^*(\theta,s^n)-y^{(k),n}\bigr]
\]
(as in \citealp{barmann2017emulating}). Evaluated at $\theta=\theta^k$, it coincides with the gradient in the sense of weak derivatives $\frac{1}{N}\sum_{n=1}^N g(\theta^k;x^n,s^n)$ of \Cref{prop:gredient_OA_loss}. That is, whereas the gradient-based optimization methods of \Cref{subsec:gradient_method} update the branch at every step, DCA fixes the branch, solves the convex subproblem to completion, and then updates the branch.

\subsection{Algorithm}
\label{subsec:dc_algorithm}

The pseudocode is given in \Cref{alg:dc_inner,alg:dc_outer}.

\begin{algorithm}[th]
    \caption{Inner solver based on a cutting-plane model for the convex subproblem (with cut inheritance)}
    \label{alg:dc_inner}
    \begin{algorithmic}[1]
        \REQUIRE Branch $y^\prime=(y^{\prime n})_{n=1}^N$, initial point $\theta^{0}\in\Theta$, solution pool $\mathcal{V}_0$, tolerance $\epsilon_{\mathrm{in}}\ge0$, maximum number of iterations $J$
        \STATE $\mathcal{G}\gets\{\,\frac{1}{N}\sum_{n=1}^N (v^n - y^{\prime n}) \mid v\in\mathcal{V}_0\,\}$, $\mathrm{UB}\gets+\infty$ (initialization with inherited cuts)
        \FOR{$j=0,1,\ldots,J-1$}
            \STATE Compute $v_j^n\gets x^*(\theta^{j} , s^n)$ for each $n$, and set $\mathcal{V}_{j+1}\gets\mathcal{V}_j\cup\{(v_j^n)_n\}$ (update of the pool)
            \STATE $g_j\gets\frac{1}{N}\sum_{n=1}^N (v_j^n-y^{\prime n})$, $\mathcal{G}\gets\mathcal{G}\cup\{g_j\}$ ($\bar{\ell}^{\sub}(\theta^{j};y^\prime)=\langle\theta^{j},g_j\rangle$)
            \IF{$\langle\theta^{j},g_j\rangle<\mathrm{UB}$}
                \STATE $\mathrm{UB}\gets\langle\theta^{j},g_j\rangle$, $\hat{\theta}\gets\theta^{j}$
            \ENDIF
            \STATE $\mathrm{LB}\gets\min_{\theta\in\Theta}\max_{g\in\mathcal{G}}\langle\theta,g\rangle$
            \IF{$\mathrm{UB}-\mathrm{LB}\le\epsilon_{\mathrm{in}}$}
                \RETURN $(\hat{\theta},\ \mathrm{converged}=\mathrm{TRUE},\ \mathcal{V}_{j+1})$
            \ENDIF
            \STATE Choose the next evaluation point $\theta^{j+1}\in\Theta$ according to the selection rule of the inner solver (see \Cref{sec:appendix_inner_rules} for details).
        \ENDFOR
        \RETURN $(\hat{\theta},\ \mathrm{converged}=\mathrm{FALSE},\ \mathcal{V}_{J})$
    \end{algorithmic}
\end{algorithm}

\begin{algorithm}[th]
    \caption{DC optimization of the empirical mean of the OS loss (simplified DCA $+$ cut inheritance)}
    \label{alg:dc_outer}
    \begin{algorithmic}[1]
        \REQUIRE Data $\{(s^n,x^n)\}_{n=1}^N$, oracles $x^*(\cdot,s^n)$ and $x^*(\cdot\,;x^n,s^n)$, initial point $\theta^0\in\Theta$, thresholds $\epsilon_{\mathrm{gap}},\epsilon_{\mathrm{tie}}\ge0$, maximum number of iterations $K$
        \STATE $\mathcal{V}^0\gets\emptyset$, and $y^{\prime n}\gets x^*(\theta^0;x^n,s^n)$ for each $n$ (initialization of the branch)
        \FOR{$k=0,1,\ldots,K-1$}
            \STATE Apply \Cref{alg:dc_inner} to $(y^\prime,\theta^k,\mathcal{V}^k)$ and obtain $(\theta^{k+1},\mathrm{converged},\widetilde{\mathcal{V}}^{k})$\label{line:dc_inner_call}
            \STATE For each $n$, $v^n\gets x^*(\theta^{k+1},s^n)$ and $y^{\prime\prime n}\gets x^*(\theta^{k+1};x^n,s^n)$; $\mathcal{V}^{k+1}\gets\widetilde{\mathcal{V}}^{k}\cup\{(v^n)_n\}$ (update of the pool)
            \STATE $\ell_k\gets\frac{1}{N}\sum_{n=1}^N\langle\theta^{k+1},v^n-y^{\prime\prime n}\rangle$ ($=\bar{\ell}^{\optdex}(\theta^{k+1})$)
            \STATE $\Delta_k\gets\frac{1}{N}\sum_{n=1}^N\langle\theta^{k+1},y^{\prime\prime n}-y^{\prime n}\rangle$
            \IF{$\ell_k\le\epsilon_{\mathrm{gap}}$}
                \STATE \textbf{break} (certificate stopping; certificate of global optimality; \Cref{prop:dc_global_certificate})
            \ENDIF
            \IF{$\Delta_k\le\epsilon_{\mathrm{tie}}$ and $\mathrm{converged}=\mathrm{TRUE}$}
                \STATE \textbf{break} (tie stopping; fixed point of DCA)
            \ENDIF
            \STATE $y^{\prime n}\gets y^{\prime\prime n}$ ($\forall n$) (update of the branch)
        \ENDFOR
        \RETURN {In the case of target-attainment stopping (see \cref{eq:def_opl}), that point; otherwise,} among the points $\theta^{k+1}$ of the outer iterations so far, the one with the smallest $\ell_k$ (if several points attain the minimum, the one from the earliest iteration)
    \end{algorithmic}
\end{algorithm}

\begin{remark}[Reasons for adopting the inner solver and cut inheritance]
    \label{rem:dc_design_rationale}
    There are two reasons for using the cutting-plane method in the inner loop.
    First, the unit of computational cost in our problem is a call to the forward-problem oracle (solving a MILP for each state), and the value and a subgradient of the objective function of the convex subproblem are obtained simultaneously from a single oracle evaluation (without additional solves). The cutting-plane method keeps all the information obtained in this way as the cut set $\mathcal{G}$, and thus makes full use of the information per oracle evaluation. Moreover, since the convex subproblem is convex and piecewise linear (\cref{eq:def_sub_loss_main}), minimizing the model requires only an LP, and an exact minimizer is reached in finitely many iterations (\Cref{lem:dc_inner_finite}). Furthermore, with cutting-plane finishing (\Cref{subsec:rule_cpfinish}), which switches the selection rule for evaluation points only once during the run, the number of iterations until reaching it can be bounded in terms of the number $T_1$ of iterations of Phase 1 per call (\Cref{prop:total_inner_iterations}), $M$ ($M:=\prod_{n=1}^N|\mathcal{Y}(s^n)|$; \Cref{lem:dc_inner_finite}), and $|\mathcal{Y}^{\mathrm{all}}_{\geq}|$ (\Cref{theo:cpfinish_exact}, \Cref{prop:total_inner_iterations}, and \Cref{rem:gap_outlook}). Moreover, in the integer case with the level method as the rule $\mathsf{R}$ for selecting evaluation points before the switch in cutting-plane finishing (\Cref{alg:rule_cpfinish}), this number is given as an explicit function of the problem size, $M$, and $|\mathcal{Y}^{\mathrm{all}}_{\geq}|$ whenever the iteration bound of \Cref{rem:level_complexity} holds (\Cref{cor:total_integer}).
    Second, the cutting-plane method returns the gap between the upper and lower bounds, $\mathrm{UB}-\mathrm{LB}$, as a certificate of convergence. Tie stopping in \Cref{alg:dc_outer} may stop erroneously at an accidental tie (a spurious fixed point) when the inner loop returns an unconverged approximate solution, and it is therefore adopted only when the convergence of the inner loop has been confirmed. This decision requires a certificate. When a gradient-based method without a lower-bounding model is used as the inner solver, convergence is decided from the value of the convex subproblem and the subgradient at the weight (e.g., for ONS and MetaGrad, when the value is at most $\epsilon_{\mathrm{in}}$ or the subgradient is $0$; \Cref{alg:dc_inner_online}).
    The reason for accumulating and reusing cuts (cut inheritance) is as follows. Cuts are obtained for free as by-products of oracle evaluations, and their validity depends only on the feasibility of the solutions, not on the choice of the branch (\Cref{prop:dc_cut_validity}). Hence they can be safely reused across branch updates and outer iterations. The effect can be quantified: the total number of inner iterations executed in $K$ outer iterations decreases from $K(M+1)$ without sharing the pool to $M+2K-1$ (\Cref{lem:dc_inner_finite}; $M$ is an upper bound on the total number of tuples of solutions that the oracle can return).
    Note that \Cref{alg:dc_inner} does not fix the selection rule for the next evaluation point $\theta^{j+1}$. Since $\mathrm{UB}$ is the minimum of the values actually evaluated and $\mathrm{LB}$ is a lower bound obtained from the model of cuts, the validity of the stopping criterion $\mathrm{UB}-\mathrm{LB}\le\epsilon_{\mathrm{in}}$ (the $\epsilon_{\mathrm{in}}$-optimality of $\hat{\theta}$) is preserved however $\theta^{j+1}$ is chosen. The pseudocode of each rule is given in \Cref{sec:appendix_inner_rules}.
\end{remark}

The outer loop (\Cref{alg:dc_outer}) is simplified DCA, which repeats fixing the branch by the dominating vertices (linearization), solving the convex subproblem in the inner loop, and updating the branch.
The inner loop (\Cref{alg:dc_inner}) is a method based on a cutting-plane model for the convex subproblem $\min_{\theta\in\Theta}\bar{\ell}^{\sub}(\theta;y^\prime)$ (if the selection rule for evaluation points is taken to be the minimizer of the model, it becomes the cutting-plane method \citep{kelley1960cutting}). The linearization of $\bar{\ell}^{\sub}(\cdot\,;y^\prime)$ at a point $\theta^{j}$ is, with $v_j^n:=x^*(\theta^{j},s^n)$ and $g_j:=\frac{1}{N}\sum_{n=1}^N(v_j^n-y^{\prime n})$,
\begin{equation}
    \bar{\ell}^{\sub}(\theta;y^\prime)
    \ \ge\ \bar{\ell}^{\sub}(\theta^{j};y^\prime)+\langle g_j , \theta-\theta^{j}\rangle
    \ =\ \langle \theta , g_j \rangle ,
    \label{eq:dc_cut_form}
\end{equation}
which is a homogeneous linear lower bound (cut) with intercept $0$ (by $\bar{\ell}^{\sub}(\theta^{j};y^\prime)=\langle\theta^{j},g_j\rangle$). The problem of minimizing the model of cuts (the master problem) is the minimization over $\Theta$ of the maximum of the cuts, that is, a linear program (LP).

There are three stopping criteria.
First, certificate stopping (stopping by $\ell_k\le\epsilon_{\mathrm{gap}}$ in \Cref{alg:dc_outer}): since the OS loss is nonnegative (\Cref{sec:performance_metric}), its value itself is an upper bound on the gap to the global optimal value, and if $\ell_k\le\epsilon_{\mathrm{gap}}$, then $\theta^{k+1}$ is $\epsilon_{\mathrm{gap}}$-globally optimal (\Cref{prop:dc_global_certificate}). This stopping does not depend on the convergence of the inner solver and is always valid.
Second, tie stopping (stopping by $\Delta_k\le\epsilon_{\mathrm{tie}}$): $\Delta_k=0$ means that the branch $y^\prime$ used for the linearization is also optimal as a linearization of the concave part at $\theta^{k+1}$ (a fixed point of DCA has been reached). However, a $\theta^{k+1}$ returned while the inner loop has not converged is an approximate solution of the convex subproblem, and accidental ties (spurious fixed points) can occur for combinatorial forward problems; hence, in the implementation, tie stopping is adopted only when the inner solver reports convergence.
Third, target-attainment stopping: the outperformance prediction loss
\begin{equation}
    \ell^{\mathrm{OP}}(\theta):=\frac{1}{N}\sum_{n=1}^N\sum_{i=1}^d\max\{x_i^n-x_i^*(\theta,s^n),\,0\}
    \label{eq:def_opl}
\end{equation}
is nonnegative, and $\ell^{\mathrm{OP}}(\theta)=0$ is equivalent to $x^*(\theta,s^n)\ge x^n$ ($\forall n$), that is, to the attainment of \cref{eq:outperformance_IOP}. The implementation computes $\ell^{\mathrm{OP}}$ every time it evaluates $x^*(\cdot,s^n)$, including at the inner evaluation points, and as soon as it falls below a threshold (default $0$), it outputs that $\theta$ and immediately stops the whole procedure.

We describe the main points of the implementation. Besides the cutting-plane method (the default), the inner solver can be replaced by the projected subgradient method, the proximal bundle method \citep{kiwiel1990proximity}, the level method \citep{lemarechal1995new}, the ellipsoid method, the analytic center cutting-plane method, the center of gravity method, the randomized center of gravity method, ONS, or MetaGrad (\Cref{alg:dc_inner_online}), and every solver that has a lower-bounding model (all except the projected subgradient method, ONS, and MetaGrad) can use inherited cuts (corresponding to the initialization of $\mathcal{G}$ in \Cref{alg:dc_inner}). The initial point of the inner loop is $\theta^{0}=\theta^k$ (warm start) (when ONS or MetaGrad is used as the inner solver, it resumes from the state of the learner at the end of the previous outer iteration), and the return value is projected onto the probability simplex to clean up numerical errors. The default values of the implementation are $J=20$ (the inner loop is an inexact solve that allows truncation), $K=100$, $\epsilon_{\mathrm{tie}}=10^{-8}$, and $\epsilon_{\mathrm{gap}}=10^{-4}$ (the values used in the experiments are given in \Cref{subsec:appendix_implementation}). Computational cost is accounted for by the number of calls to the forward-problem oracle rather than by the number of iterations (calls to the forward-problem oracle occur only in computing $v_j$ in \Cref{alg:dc_inner} and $v,y^{\prime\prime}$ and the initialization of the branch in \Cref{alg:dc_outer}; the additional calls when ONS or MetaGrad is used as the inner solver are shown in \Cref{alg:dc_inner_online}).

\subsection{Cut inheritance}
\label{subsec:dc_cut_pool}

In DC programming, the idea of keeping the cuts (cutting-plane model) of the convex component across iterations was proposed by \citet{oliveira2019proximal} as a proximal bundle method for DC programming. That paper takes as its issue that DCA and the proximal linearized method (PLM) must solve a nonsmooth convex subproblem exactly at each iteration, and replaces the convex component $f_1$ in the subproblem by a cutting-plane model $\breve{f}_1^k$ constructed from an information bundle $\mathcal{B}_1^k$ that accumulates the linearizations obtained in past iterations. Whereas the linearization of the concave component $-f_2$ is updated only when the stability center is updated, the model of $f_1$ is kept across iterations, so the convex subproblem need not be solved exactly. Cut inheritance, described in this subsection, instantiates this idea in accordance with the structure of our problem that the convex component $F_1$ is a support function.

In the initialization of $\mathcal{G}$ in \Cref{alg:dc_inner}, tuples of solutions of the forward-problem oracle obtained in the past, $v=(v^n)_n$ (the pool $\mathcal{V}$), are converted into cuts for the current branch $y^\prime$ and used as initial cuts. We call the cuts made in this way from the tuples in the pool $\mathcal{V}$ inherited cuts, and this mechanism cut inheritance. Their validity follows from the next proposition.

\begin{proposition}[Validity of cuts]
    \label{prop:dc_cut_validity}
    Let $v=(v^n)_{n=1}^N$ be any tuple satisfying $v^n\in\mathcal{X}(s^n)$ ($n=1,\ldots,N$). Then, for any branch $y^\prime\in\mathcal{Y}^{\mathrm{all}}_{\geq}$, we have
    \[
        \bar{\ell}^{\sub}(\theta;y^\prime)
        \ \ge\ \left\langle \theta ,\, \frac{1}{N}\sum_{n=1}^N (v^n-y^{\prime n}) \right\rangle
        \qquad(\forall\theta\in\mathbb{R}^d) .
    \]
    Furthermore, if each $v^n$ attains $F_1(\theta^{0};s^n)$ at some $\theta^{0}$, then equality holds at $\theta=\theta^{0}$.
\end{proposition}

\begin{proof}
    Since $v^n\in\mathcal{X}(s^n)$, we have $F_1(\theta;s^n)\ge\langle\theta,v^n\rangle$ for every $\theta$. Hence
    \begin{align*}
        \bar{\ell}^{\sub}(\theta;y^\prime)
        &=\frac{1}{N}\sum_{n=1}^N\bigl[F_1(\theta;s^n)-\langle\theta,y^{\prime n}\rangle\bigr]
        \ge\frac{1}{N}\sum_{n=1}^N\langle\theta,v^n-y^{\prime n}\rangle\\
        &=\left\langle\theta,\frac{1}{N}\sum_{n=1}^N(v^n-y^{\prime n})\right\rangle .
    \end{align*}
    The equality condition follows by substituting $F_1(\theta^{0};s^n)=\langle\theta^{0},v^n\rangle$ ($\forall n$).
\end{proof}

As \Cref{prop:dc_cut_validity} shows, these cuts have the following properties.
First, the validity of a cut depends only on the feasibility of $v$ ($v^n\in\mathcal{X}(s^n)$), and not on at which $\theta$, under which branch, or by which solver $v$ was obtained. Hence the same pool $\mathcal{V}$ can be shared across branch updates, DC iterations, types of inner solvers, and different runs of DC optimization (as long as the sequence of states $\{s^n\}$ is the same).
Second, since $v$ need not even be optimal (it is valid as a lower bound as long as it is feasible), the cuts are robust to truncation and approximate solution of the forward-problem solver (although the tightness of the cuts depends on optimality).
Third, $v$ has already been obtained as a by-product of evaluating $F_1$, and generating the cuts requires no additional forward-problem solves.
These are consequences of $F_1$ being a support function (positively homogeneous of degree $1$), in contrast to cuts of a general convex function, which have nonzero intercepts depending on the linearization point.

\subsection{Underlying literature}
\label{subsec:dc_literature}

Every component of the skeleton of our method is a known method. We make the correspondence explicit.
The outer loop is, among the variants of DCA formulated by \citet{pham1997convex} for DC programming, simplified DCA, which linearizes the concave part by a subgradient (here, the dominating vertices). Our setting fits directly into the polyhedral DC programs of Section~4 of that paper. Indeed, in
\begin{gather*}
    (\tilde{\mathrm{P}})\colon\ \inf\{\phi(\theta)-\tilde{h}(\theta) \mid \theta\in\mathbb{R}^d\} ,\\
    \tilde{h}(\theta)=\max\{\langle a^i,\theta\rangle-\alpha^i \mid i\in I\}\quad(I\text{ is a finite set})
\end{gather*}
of that paper, it suffices to set
\[
    \phi(\theta)=\frac{1}{N}\sum_{n=1}^N F_1(\theta;s^n)+\chi_\Theta(\theta) ,
    \qquad
    \tilde{h}(\theta)=\frac{1}{N}\sum_{n=1}^N F_2(\theta;x^n,s^n)
\]
(where $\chi_\Theta$ is the indicator function of $\Theta$). By \cref{eq:F1F2_finite_max},
\[
    \tilde{h}(\theta)=\max_{y^\prime\in\mathcal{Y}^{\mathrm{all}}_{\geq}}\left\langle\theta,\frac{1}{N}\sum_{n=1}^N y^{\prime n}\right\rangle ,
\]
so $\tilde{h}$ is a polyhedral convex function with $a^{y^\prime}=\frac{1}{N}\sum_{n=1}^N y^{\prime n}$, $\alpha^{y^\prime}=0$, and $I=\mathcal{Y}^{\mathrm{all}}_{\geq}$ (and $\phi$ is polyhedral convex as well), and the convex subproblem $(\mathrm{P}_i)$ ($\inf_\theta\{\phi(\theta)-(\langle a^i,\theta\rangle-\alpha^i)\}$) of that paper coincides exactly with the convex subproblem of the branch $y^\prime$ $m_{\sub}(y^\prime)=\min_{\theta\in\Theta}\bar{\ell}^{\sub}(\theta;y^\prime)$ of this section. Hence \cref{eq:dc_min_min} of \Cref{prop:dc_branch_decomp} is an instantiation, for our problem, of the equality (where $\alpha$ is the optimal value of the DC program $(\tilde{\mathrm{P}})$: $\inf_\theta\{\phi(\theta)-\tilde{h}(\theta)\}$ of that paper) $\alpha=\inf_{i\in I}\inf_\theta\{\phi(\theta)-(\langle a^i,\theta\rangle-\alpha^i)\}$ in Section~4.1 of that paper and of the remark at the beginning of Section~4.2 that ``solving the polyhedral DC program $(\tilde{\mathrm{P}})$ globally reduces to solving $|I|$ convex programs $(\mathrm{P}_i)$.''
The monotone decrease of the outer loop (\Cref{prop:dc_descent}) corresponds to Theorem 5(i) of that paper (although \Cref{prop:dc_descent} differs in that it allows inexact solution of the inner problem). The relation between finite convergence (\Cref{prop:dc_finite_convergence}) and Theorem 6 of that paper was discussed in \Cref{rem:dca_finite_theorem6}. For later developments of DCA, see \citet{lethi2018dc}.
The default inner loop is the cutting-plane method \citep{kelley1960cutting}; the facts that the convex subproblem reduces to the minimization of the suboptimality loss \citep{mohajerin2018data} and that its subgradient is obtained from the forward-problem oracle share the same structure as the subgradient computation of \citet{barmann2017emulating}. Among the interchangeable inner solvers, the proximal bundle method is due to \citet{kiwiel1990proximity} and the level method to \citet{lemarechal1995new}.
The idea of keeping the cuts (model) of the convex part across updates of the linearization of the concave part was proposed by \citet{oliveira2019proximal}, as described in \Cref{subsec:dc_cut_pool}. Related prior work includes the double bundle method \citep{joki2018double}, which has a bundle for each of the two DC components, and BEM-DC \citep{gaudioso2023bundle}, which assumes polyhedral DC and distributes cuts to both components; the latter is the closest to our method. The differences from these are that (i) the cuts are made not from a subgradient oracle but from the feasible solutions themselves returned by the forward-problem oracle, and, by the support-function structure, they are global lower bounds with intercept $0$ that do not depend on the branch (\Cref{prop:dc_cut_validity}); (ii) since their validity depends only on feasibility, they are robust to truncation; (iii) they are kept not only within a single run but across types of inner solvers and runs (in the experiments of this paper, the pool was shared within a single run of DC optimization); and (iv) the method is applied to outperformance inverse optimization (OS loss minimization).

\subsection{Finite convergence of DCA and local optimality}
\label{sec:theory_in_dc_opt}

This subsection describes the finite convergence of the DC optimization method of \Cref{sec:dc_opt} and local optimality. We use the notation of \Cref{sec:dc_opt} (in particular, the finite set of branches $\mathcal{Y}^{\mathrm{all}}_{\geq}$ and the minimum value of the convex subproblem $m_{\sub}(y^\prime)$). We also write the branch at iteration $k$ of \Cref{alg:dc_outer} as $y^{(k)}:=(x^*(\theta^k;x^n,s^n))_{n=1}^N$. Then, if the inner convex subproblems are solved exactly, DC optimization stops in finitely many iterations.

\begin{proposition}[Finite convergence of exact DCA]
    \label{prop:dc_finite_convergence}
    Under \Cref{assu:margin}, suppose that in line~\ref{line:dc_inner_call} of \Cref{alg:dc_outer} (the application of \Cref{alg:dc_inner}), the inner convex subproblem is solved exactly, that is, $\theta^{k+1}\in\argmin_{\theta\in\Theta}\bar{\ell}^{\sub}(\theta;y^{(k)})$ (by \Cref{lem:dc_inner_finite}, this is possible with finitely many oracle calls, in which case we may set $\mathrm{converged}=\mathrm{TRUE}$). With $\epsilon_{\mathrm{tie}}=0$ and $K=\infty$, \Cref{alg:dc_outer} stops within at most $|\mathcal{Y}^{\mathrm{all}}_{\geq}|$ outer iterations (by certificate stopping or tie stopping). In the case of tie stopping, the point $\theta^{k+1}$ of the iteration at which it stopped satisfies
    \begin{equation}
        \bar{\ell}^{\optdex}(\theta^{k+1})
        =m_{\sub}(y^{(k)})
        =\min_{\theta\in\Theta}\bar{\ell}^{\sub}(\theta;y^{(k)}) .
        \label{eq:dc_fixed_point}
    \end{equation}
    That is, $\theta^{k+1}$ is an exact minimizer of the convex subproblem for its own branch (a fixed point of DCA). Since the output of \Cref{alg:dc_outer} is the point with the smallest $\ell_k$, its empirical mean of the OS loss is at most $m_{\sub}(y^{(k)})$.
\end{proposition}

The proof is given in \Cref{subsec:dc_theory}.

\begin{proposition}[Local optimality of tie-stopping points]
    \label{prop:dc_local_min}
    In the setting of \Cref{prop:dc_finite_convergence}, suppose that \Cref{alg:dc_outer} stops by tie stopping at iteration $k$. If $\argmax_{x\in\mathcal{X}_{\geq x^n}(s^n)}\langle\theta^{k+1},x\rangle$ is a singleton for every $n$, then $\theta^{k+1}$ is a local minimizer of $\bar{\ell}^{\optdex}$ over $\Theta$.
\end{proposition}

\begin{proof}
    Let $y^\prime:=y^{(k)}$. Since $y^{\prime\prime n}:=x^*(\theta^{k+1};x^n,s^n)$ is an optimal solution of $\max_{x\in\mathcal{X}_{\geq x^n}(s^n)}\langle\theta^{k+1},x\rangle$ and $y^{\prime n}\in\mathcal{X}_{\geq x^n}(s^n)$, each term $\langle\theta^{k+1},y^{\prime\prime n}-y^{\prime n}\rangle$ of $\Delta_k$ in \Cref{alg:dc_outer} is nonnegative. Under tie stopping, $\Delta_k\le\epsilon_{\mathrm{tie}}=0$, so each term is $0$, and $y^{\prime n}$ is also an optimal solution at $\theta^{k+1}$. By assumption, this optimal solution is unique, so $\argmax_{x\in\mathcal{X}_{\geq x^n}(s^n)}\langle\theta^{k+1},x\rangle=\{y^{\prime n}\}$.
    Let $V_n$ be the set of extreme points of the pieces of $\mathcal{X}_{\geq x^n}(s^n)$ (the nonempty intersections $P\cap\{x\in\mathbb{R}^d\mid x\geq x^n\}$ with the pieces $P$ of $\mathcal{X}(s^n)$). Then $V_n$ is finite, $y^{\prime n}\in V_n$ by \Cref{assu:margin}(3), and $\max_{x\in\mathcal{X}_{\geq x^n}(s^n)}\langle\theta,x\rangle=\max_{v\in V_n}\langle\theta,v\rangle$ for every $\theta\in\mathbb{R}^d$. By the uniqueness of the optimal solution, $\langle\theta^{k+1},y^{\prime n}\rangle>\langle\theta^{k+1},v\rangle$ for every $v\in V_n\setminus\{y^{\prime n}\}$. Since $V_n$ and the number of $n$ are finite, there exists a neighborhood $U$ of $\theta^{k+1}$ on which $\langle\theta,y^{\prime n}\rangle>\langle\theta,v\rangle$ for every $n$ and every $v\in V_n\setminus\{y^{\prime n}\}$. Hence $\max_{x\in\mathcal{X}_{\geq x^n}(s^n)}\langle\theta,x\rangle=\langle\theta,y^{\prime n}\rangle$ for $\theta\in U$, and by \cref{eq:def_optdex,eq:def_sub_loss_main}, $\bar{\ell}^{\optdex}(\theta)=\bar{\ell}^{\sub}(\theta;y^\prime)$ on $U\cap\Theta$.
    By \Cref{prop:dc_finite_convergence}, $\theta^{k+1}$ is a minimizer of $\bar{\ell}^{\sub}(\cdot\,;y^\prime)$ over $\Theta$, and by \cref{eq:dc_fixed_point}, $\bar{\ell}^{\optdex}(\theta^{k+1})=\bar{\ell}^{\sub}(\theta^{k+1};y^\prime)$. Therefore, for every $\theta\in U\cap\Theta$, $\bar{\ell}^{\optdex}(\theta)=\bar{\ell}^{\sub}(\theta;y^\prime)\ge\bar{\ell}^{\sub}(\theta^{k+1};y^\prime)=\bar{\ell}^{\optdex}(\theta^{k+1})$.
\end{proof}

\begin{remark}[Relation to a known finite convergence theorem]
    \label{rem:dca_finite_theorem6}
    \citet{pham1997convex} proved the finite convergence of simplified DCA (Theorem 6 of that paper) for DC programs $\inf_{\theta}\{\phi(\theta)-\tilde{h}(\theta)\}$ ($\phi$ is a convex function) whose concave part is explicitly written as a polyhedral convex function
    \[
        \tilde{h}(\theta)=\max_{i\in I}\bigl(\langle a^i,\theta\rangle-\alpha^i\bigr)
        \qquad
        (I\text{ is a finite index set, }a^i\in\mathbb{R}^d,\ \alpha^i\in\mathbb{R}) .
    \]
    The empirical mean of the OS loss can also be written in this form $\phi-\tilde{h}$, and the index $i$ of the linear pieces corresponds to the branch $y^\prime\in\mathcal{Y}^{\mathrm{all}}_{\geq}$ (\Cref{subsec:dc_literature}).
    However, that theorem assumes the natural choice of subgradients, that is, a choice in the relative interior of $\partial\tilde{h}(\theta)$ that puts positive weights on all active linear pieces (Lemma~3 of that paper), and implementing it requires enumerating all linear pieces active at each point $\theta$, that is, all dominating vertices that are optimal solutions of $\max_{x\in\mathcal{X}_{\geq x^n}(s^n)}\langle\theta,x\rangle$ for each state $s^n$. Since the forward-problem oracle returns only one optimal solution, this enumeration is difficult in general, and applying that theorem directly to our setting is difficult. Therefore, in this paper, we prove finiteness directly using the tie stopping rule in the setting where the oracle returns only one optimal solution (\Cref{prop:dc_finite_convergence}).
\end{remark}

\begin{remark}[Guarantee of global optimality]
    \label{rem:dc_limitations}
    The point obtained by tie stopping is a minimizer for the branch $y^{(k)}$. Global optimality is guaranteed by certificate stopping ($\ell_k\le\epsilon_{\mathrm{gap}}$; \Cref{prop:dc_global_certificate}); in particular, in settings where \cref{eq:outperformance_IOP} is realizable (the global minimum value is $0$), reaching the global minimum can be confirmed by a loss value of $0$.
    Note that, when $\Theta=\Delta^{d-1}$, the set of weights at which the optimal solution is not unique has measure $0$ on $\Delta^{d-1}$ (\Cref{prop:argmax_unique_ae}).
    %
    %

\end{remark}

\begin{proposition}[{\citealp[Lemma~3.3]{kitaoka2024exact}}]
    \label{prop:argmax_unique_ae}
    Assume \Cref{assu:margin}(2) and let $\Theta=\Delta^{d-1}$. Then, for any states $s^1,\ldots,s^N\in\mathcal{S}$ and $x^n\in\mathcal{X}(s^n)$ ($n=1,\ldots,N$), for almost every $\theta\in\Delta^{d-1}$ (with respect to the measure on $\Delta^{d-1}$ induced by the Lebesgue measure), both $\argmax_{x\in\mathcal{X}(s^n)}\langle\theta,x\rangle$ and $\argmax_{x\in\mathcal{X}_{\geq x^n}(s^n)}\langle\theta,x\rangle$ are singletons for every $n$.
\end{proposition}

\citet[Lemma~3.3]{kitaoka2024exact} shows that, when the feasible region of each state is a finite union of bounded closed convex polytopes, for almost every $\theta\in\Delta^{d-1}$, the feature of the optimal solution (in that paper, the image $f(x,s)$ of a feasible solution $x$ under the feature map $f$ is called the feature; our setting corresponds to $f(x,s)=x$) is uniquely determined for every $n$. Its proof shows that the maximizer of $\langle\theta,\cdot\rangle$ over the finite set of extreme points of the pieces is unique. If the maximizer among the extreme points is unique, the optimal face of each piece is a single point; hence, applying this with the identity feature map to $\mathcal{X}(s^n)$ and $\mathcal{X}_{\geq x^n}(s^n)$, both of which are finite unions of nonempty bounded polyhedra by \Cref{rem:assu_margin}, yields \Cref{prop:argmax_unique_ae}. Although that lemma is stated under the assumption that the observations are optimal solutions for a true weight, its proof does not use this assumption.

\subsection{Theoretical analysis}
\label{subsec:dc_theory}

Complete proofs are given for all propositions in this subsection. We first state the validity of certificate stopping.

\begin{proposition}[Certificate of global optimality]
    \label{prop:dc_global_certificate}
    For any $\theta\in\Theta$ and $\epsilon\ge0$, if $\bar{\ell}^{\optdex}(\theta)\le\epsilon$, then we have
    \[
        \bar{\ell}^{\optdex}(\theta)
        \le
        \min_{\theta^\prime\in\Theta}\bar{\ell}^{\optdex}(\theta^\prime)+\epsilon .
    \]
    In particular, if $\epsilon=0$, then $\theta$ is a global minimizer of \cref{eq:empirical_outperformance_IOP}.
\end{proposition}

\begin{proof}
    By the inclusion $\mathcal{X}_{\geq x^n}(s^n)\subseteq\mathcal{X}(s^n)$, each term satisfies $\ell^{\optdex}(\theta^\prime;x^n,s^n)\ge0$, and hence $\min_{\theta^\prime\in\Theta}\bar{\ell}^{\optdex}(\theta^\prime)\ge0$. Therefore $\bar{\ell}^{\optdex}(\theta)\le\epsilon\le\min_{\theta^\prime}\bar{\ell}^{\optdex}(\theta^\prime)+\epsilon$.
\end{proof}

Next, we show that the inner cutting-plane method reaches an exact minimizer in finitely many iterations in the ideal setting (no truncation and tolerance $0$).

\begin{lemma}[Finite exactness of the inner loop]
    \label{lem:dc_inner_finite}
    Under \Cref{assu:margin}, let $\epsilon_{\mathrm{in}}=0$ and $J=\infty$, suppose that the LP in \Cref{alg:dc_inner} is solved exactly, and let the selection rule for evaluation points be the cutting-plane method ($\theta^{j+1}$ is a minimizer attaining $\mathrm{LB}$).
    Set $M:=\prod_{n=1}^N|\mathcal{Y}(s^n)|$ (an upper bound on the total number of values that a tuple of oracle solutions $(v^n)_n$ can take).
    \begin{description}
        \item[(1)] When \Cref{alg:dc_outer} calls \Cref{alg:dc_inner} once, the call stops within at most $M+1$ inner iterations (iterations of the loop over $j$ in \Cref{alg:dc_inner}), and the returned $\hat{\theta}$ is an exact minimizer of $\bar{\ell}^{\sub}(\cdot\,;y^\prime)$ over $\Theta$. This bound does not depend on the contents of the pool $\mathcal{V}$ at the time of the call.
        \item[(2)] When \Cref{alg:dc_outer} is run for $K$ outer iterations (iterations of the loop over $k$), the total number of inner iterations executed during the run is at most $M+2K-1$.
    \end{description}
\end{lemma}

\begin{proof}
    Let $\mathcal{G}_j$ be the cut set at the end of iteration $j$, and let $\mathrm{LB}_{j+1}=\min_{\theta\in\Theta}\max_{g\in\mathcal{G}_j}\langle\theta,g\rangle$ be the optimal value of the LP. By \Cref{prop:dc_cut_validity}, all cuts in $\mathcal{G}_j$ (including the inherited cuts) are linear lower bounds of $\bar{\ell}^{\sub}(\cdot\,;y^\prime)$, so $\mathrm{LB}_{j+1}\le m_{\sub}(y^\prime)$, and since $\mathcal{G}_j$ is monotonically increasing, $\mathrm{LB}_{j+1}$ is nondecreasing in $j$. Moreover, since $\mathrm{UB}$ is the minimum of the values actually evaluated, $\bar{\ell}^{\sub}(\theta^{i};y^\prime)=\langle\theta^{i},g_i\rangle$, we always have $\mathrm{UB}\ge m_{\sub}(y^\prime)$. Therefore, at stopping ($\mathrm{UB}\le\mathrm{LB}$), $\mathrm{UB}=m_{\sub}(y^\prime)$, and $\hat{\theta}$ is a minimizer.

    We show finiteness. Let $\mathcal{V}_0$ be the pool at the time of the call. By the initialization of $\mathcal{G}$ in \Cref{alg:dc_inner}, the cut of each tuple in $\mathcal{V}_0$ (for the current branch $y^\prime$) is already contained in $\mathcal{G}$ at the first inner iteration. The tuple of oracle solutions $v_j$ takes values in the finite set $\prod_n\mathcal{Y}(s^n)$. Suppose that at some inner iteration $j\ge1$, $v_j$ coincides with an element of $\mathcal{V}_0\cup\{v_0,\ldots,v_{j-1}\}$. Then its cut equals $g_j$ and belongs to $\mathcal{G}_{j-1}$. Since $\theta^{j}$ ($j\ge1$) is an optimal solution of the LP over $\mathcal{G}_{j-1}$,
    \[
        \bar{\ell}^{\sub}(\theta^{j};y^\prime)=\langle\theta^{j},g_j\rangle
        \le\max_{g\in\mathcal{G}_{j-1}}\langle\theta^{j},g\rangle=\mathrm{LB}_{j} .
    \]
    Hence, after the update at inner iteration $j$, $\mathrm{UB}\le\mathrm{LB}_j\le\mathrm{LB}_{j+1}$, and the call stops at this iteration. By contraposition, at an inner iteration $j\ge1$ that does not stop, $v_j$ is always a tuple newly added to $\mathcal{V}$. At $j=0$, the initial point $\theta^{0}=\theta^k$ is a warm start and is not necessarily an optimal solution of the LP, so $v_0$ may coincide with an element of $\mathcal{V}_0$ without stopping; this can happen only if $|\mathcal{V}_0|\ge1$. Since $|\mathcal{V}|\le M$, the inner iterations that do not stop occur at most $(M-|\mathcal{V}_0|)+\min\{1,|\mathcal{V}_0|\}\le M$ times, and adding the stopping iteration, the number of inner iterations is at most $M+1$. This bound does not depend on $\mathcal{V}_0$. This proves (1).

    We show (2). In \Cref{alg:dc_outer}, the pool $\mathcal{V}$ increases monotonically across outer iterations and is never reinitialized. By the argument of (1), every inner iteration that does not stop either adds one new tuple to $\mathcal{V}$ or is an iteration $j=0$ whose tuple $v_0$ already belongs to the pool. Throughout the whole run, the former occur at most $M$ times. The latter occur at most once per call, and not in the first call, since $\mathcal{V}^0=\emptyset$; hence they occur at most $K-1$ times. In addition, each call has one stopping inner iteration, and the number of calls equals the number $K$ of outer iterations; hence the total number of inner iterations is at most $M+(K-1)+K=M+2K-1$.
\end{proof}

\Cref{lem:dc_inner_finite}(2) expresses the quantitative effect of cut inheritance (\Cref{prop:dc_cut_validity}). Without sharing the pool, the bound $M+1$ of (1) is needed independently for each call, and the bound on the total is only $K(M+1)$. With a shared pool, a tuple obtained once is reused as a cut in subsequent calls, and an inner iteration $j\ge1$ that returns a duplicate tuple causes stopping on the spot, so the total drops to $M+2K-1$.

For the outer loop, we first show, as a property that holds even if the inner loop is inexact, the monotone decrease of the OS loss.

\begin{proposition}[Descent property]
    \label{prop:dc_descent}
    Write the branch at iteration $k$ of \Cref{alg:dc_outer} (the $y^\prime$ when \Cref{alg:dc_inner} is applied) as $y^{(k)}:=(x^*(\theta^k;x^n,s^n))_{n=1}^N$. If the inner solver returns $\theta^{k+1}$ satisfying
    \begin{equation}
        \bar{\ell}^{\sub}(\theta^{k+1};y^{(k)})\le\bar{\ell}^{\sub}(\theta^{k};y^{(k)})
        \label{eq:dc_inner_no_worse}
    \end{equation}
    (since \Cref{alg:dc_inner} first evaluates the initial point $\theta^{0}=\theta^k$ and returns the best of the evaluated points, it satisfies \cref{eq:dc_inner_no_worse} even when truncated), then we have
    \[
        0\le
        \bar{\ell}^{\optdex}(\theta^{k+1})
        \le
        \bar{\ell}^{\optdex}(\theta^{k}) .
    \]
\end{proposition}

\begin{proof}
    Since $y^{(k),n}$ attains $F_2(\theta^k;x^n,s^n)$, we have $\bar{\ell}^{\sub}(\theta^k;y^{(k)})=\bar{\ell}^{\optdex}(\theta^k)$. On the other hand, since $y^{(k),n}\in\mathcal{X}_{\geq x^n}(s^n)$, we have $F_2(\theta^{k+1};x^n,s^n)\ge\langle\theta^{k+1},y^{(k),n}\rangle$, so
    \[
        \bar{\ell}^{\optdex}(\theta^{k+1})
        =\frac{1}{N}\sum_{n=1}^N\bigl[F_1(\theta^{k+1};s^n)-F_2(\theta^{k+1};x^n,s^n)\bigr]
        \le\bar{\ell}^{\sub}(\theta^{k+1};y^{(k)}) .
    \]
    Combined with \cref{eq:dc_inner_no_worse}, this gives the claim (nonnegativity holds as in the proof of \Cref{prop:dc_global_certificate}).
\end{proof}

Finally, we prove the finite convergence of the outer loop when the inner convex subproblems are solved exactly (\Cref{prop:dc_finite_convergence}).

\begin{proof}[Proof of \Cref{prop:dc_finite_convergence}]
    At iteration $k$, $y^{\prime\prime n}=y^{(k+1),n}$. First, since $y^{(k),n}\in\mathcal{X}_{\geq x^n}(s^n)$ and $y^{(k+1),n}$ attains $F_2(\theta^{k+1};x^n,s^n)$, each term of $\Delta_k$ satisfies
    \[
        \langle\theta^{k+1},y^{(k+1),n}-y^{(k),n}\rangle
        =F_2(\theta^{k+1};x^n,s^n)-\langle\theta^{k+1},y^{(k),n}\rangle\ge0 ;
    \]
    hence $\Delta_k\ge0$, and moreover we have
    \begin{equation}
        \bar{\ell}^{\optdex}(\theta^{k+1})
        =\bar{\ell}^{\sub}(\theta^{k+1};y^{(k)})-\Delta_k
        =m_{\sub}(y^{(k)})-\Delta_k .
        \label{eq:dc_loss_identity}
    \end{equation}
    Here, the second equality follows from the exactness of the inner loop. Also, since $\bar{\ell}^{\sub}(\theta^{k+1};y^{(k+1)})=\bar{\ell}^{\optdex}(\theta^{k+1})$,
    \begin{equation}
        m_{\sub}(y^{(k+1)})\le\bar{\ell}^{\sub}(\theta^{k+1};y^{(k+1)})=m_{\sub}(y^{(k)})-\Delta_k .
        \label{eq:dc_m_decrease}
    \end{equation}
    If neither stopping occurs at iteration $k$, then $\Delta_k>0$, and by \cref{eq:dc_m_decrease}, $m_{\sub}(y^{(k+1)})<m_{\sub}(y^{(k)})$, i.e., the branch minimum strictly decreases. Since $m_{\sub}$ takes values in the finite set $\{m_{\sub}(y^\prime)\mid y^\prime\in\mathcal{Y}^{\mathrm{all}}_{\geq}\}$, strict decreases can occur at most $|\mathcal{Y}^{\mathrm{all}}_{\geq}|-1$ times, and one of the stoppings occurs within at most $|\mathcal{Y}^{\mathrm{all}}_{\geq}|$ iterations. At tie stopping, $\Delta_k=0$, so \cref{eq:dc_loss_identity} gives \cref{eq:dc_fixed_point}.
\end{proof}

\section{Exact solution of the inner convex subproblems based on gaps}
\label{sec:appendix_gap_exactness}

In this section, we extend the finite exactness of \Cref{lem:dc_inner_finite} (the guarantee when the rule for selecting evaluation points is the cutting-plane method) to stabilized selection rules such as the level method and the proximal bundle method. The key is to transplant to the inner convex subproblems the idea used in \citet{kitaoka2024exact,kitaoka2026explicit}: strengthen an approximation guarantee to exact attainment through a positive gap possessed by a piecewise-linear structure, and make the number of iterations explicit.
Throughout this section, we assume \Cref{assu:margin} and fix a branch $y^\prime\in\mathcal{Y}^{\mathrm{all}}_{\geq}$.

\subsection{Submodel gap}
\label{subsec:submodel_gap}

Let $\mathcal{Y}^{\mathrm{all}}:=\prod_{n=1}^N\mathcal{Y}(s^n)$ be the set of all tuples of solutions that the oracle can return (a finite set; $|\mathcal{Y}^{\mathrm{all}}|=M$ is the $M$ of \Cref{lem:dc_inner_finite}), and write the cut of a tuple $v\in\mathcal{Y}^{\mathrm{all}}$ as $g_v:=\frac{1}{N}\sum_{n=1}^N(v^n-y^{\prime n})$ ($g_v$ depends on the fixed branch $y^\prime$, which we suppress in the notation). We call $g_v$ a vertex cut. By \cref{eq:F1F2_finite_max}, we have
\begin{equation}
    \bar{\ell}^{\sub}(\theta;y^\prime)=\max_{v\in\mathcal{Y}^{\mathrm{all}}}\langle\theta,g_v\rangle
    \qquad(\forall\theta\in\mathbb{R}^d) .
    \label{eq:subloss_as_finite_max}
\end{equation}
For a nonempty $W\subseteq\mathcal{Y}^{\mathrm{all}}$, we define the minimum value of the submodel by
\begin{equation}
    m_W(y^\prime):=\min_{\theta\in\Theta}\max_{v\in W}\langle\theta,g_v\rangle .
    \label{eq:def_submodel_min}
\end{equation}
If $W\subseteq W^\prime$, then $m_W(y^\prime)\le m_{W^\prime}(y^\prime)$, and by \cref{eq:subloss_as_finite_max}, $m_{\mathcal{Y}^{\mathrm{all}}}(y^\prime)=m_{\sub}(y^\prime)$. We define the submodel gap by
\begin{equation}
    \delta(y^\prime):=m_{\sub}(y^\prime)-\max\bigl\{\,m_W(y^\prime)\ \bigm|\ \emptyset\neq W\subseteq\mathcal{Y}^{\mathrm{all}},\ m_W(y^\prime)<m_{\sub}(y^\prime)\,\bigr\}
    \label{eq:def_submodel_gap}
\end{equation}
(when the set over which the $\max$ on the right-hand side is taken is empty, we read $\delta(y^\prime):=+\infty$). $\delta(y^\prime)$ is not an assumption but a complexity constant that is automatically positive by finiteness.

\begin{lemma}[Positivity of the gap and discreteness of the lower bound]
    \label{lem:submodel_gap}

    (i) $\delta(y^\prime)>0$.

    (ii) When \Cref{alg:dc_inner} is run with the branch $y^\prime$ and a pool $\mathcal{V}_0$ consisting only of tuples of solutions returned by the oracle, the value $\mathrm{LB}$ computed at each iteration equals $m_W(y^\prime)$ for some nonempty $W\subseteq\mathcal{Y}^{\mathrm{all}}$. In particular, the values that $\mathrm{LB}$ can take belong to the finite set $\{m_W(y^\prime) \mid \emptyset\neq W\subseteq\mathcal{Y}^{\mathrm{all}}\}$, and we always have $\mathrm{LB}\le m_{\sub}(y^\prime)$.
\end{lemma}

\begin{proof}
    (i) Since $\mathcal{Y}^{\mathrm{all}}$ is finite, it has finitely many subsets, so the largest of the submodel minimum values with $m_W(y^\prime)<m_{\sub}(y^\prime)$ is attained and is strictly smaller than $m_{\sub}(y^\prime)$.
    (ii) The cut set $\mathcal{G}$ of \Cref{alg:dc_inner} is always the set of all cuts $\{g_v \mid v\in W\}$ of the set $W$ consisting of ``the tuples of $\mathcal{V}_0$'' together with ``the tuples evaluated in that call''. Since every tuple added to the pool in \Cref{alg:dc_inner,alg:dc_outer} is a return value of the oracle, $W\subseteq\mathcal{Y}^{\mathrm{all}}$, and $\mathrm{LB}=\min_{\theta\in\Theta}\max_{v\in W}\langle\theta,g_v\rangle=m_W(y^\prime)$. The inequality $\mathrm{LB}\le m_{\sub}(y^\prime)$ follows from monotonicity with respect to $W\subseteq\mathcal{Y}^{\mathrm{all}}$.
\end{proof}

\subsection{Exact agreement of the lower bound and finite exactness of cutting-plane finishing}
\label{subsec:lb_snap_cpfinish}

\begin{lemma}[Exact agreement of the lower bound]
    \label{lem:lb_snap}
    In the setting of \Cref{lem:submodel_gap}(ii), let $0\le\epsilon_{\mathrm{in}}<\delta(y^\prime)$. When \Cref{alg:dc_inner} stops with $\mathrm{UB}-\mathrm{LB}\le\epsilon_{\mathrm{in}}$ ($\mathrm{converged}=\mathrm{TRUE}$), regardless of the rule for selecting evaluation points, we have
    \[
        \mathrm{LB}=m_{\sub}(y^\prime) ,
        \qquad
        \bar{\ell}^{\sub}(\hat{\theta};y^\prime)=\mathrm{UB}\le m_{\sub}(y^\prime)+\epsilon_{\mathrm{in}} .
    \]
    That is, the minimum value of the convex subproblem itself is identified exactly.
\end{lemma}

\begin{proof}
    Since $\mathrm{UB}$ is the minimum of the true values of $\bar{\ell}^{\sub}(\cdot\,;y^\prime)$ at the evaluation points, $\mathrm{UB}\ge m_{\sub}(y^\prime)$. At stopping, $\mathrm{LB}\ge\mathrm{UB}-\epsilon_{\mathrm{in}}\ge m_{\sub}(y^\prime)-\epsilon_{\mathrm{in}}>m_{\sub}(y^\prime)-\delta(y^\prime)$. On the other hand, by \Cref{lem:submodel_gap}(ii), $\mathrm{LB}=m_W(y^\prime)\le m_{\sub}(y^\prime)$, and by \cref{eq:def_submodel_gap}, every value with $m_W(y^\prime)<m_{\sub}(y^\prime)$ is at most $m_{\sub}(y^\prime)-\delta(y^\prime)$; hence $\mathrm{LB}=m_{\sub}(y^\prime)$.
\end{proof}

What becomes exact in \Cref{lem:lb_snap} is the lower bound (the identification of the minimum value), and the returned point $\hat{\theta}$ is $\epsilon_{\mathrm{in}}$-optimal. To obtain an exact minimizer, we use the following cutting-plane finishing: in Phase 1, run with an arbitrary rule for selecting evaluation points (e.g., the level method) until $\mathrm{UB}-\mathrm{LB}\le{\epsilon_{\mathrm{sw}}}$ (where $\epsilon_{\mathrm{sw}}>0$ is the switching threshold); in the subsequent Phase 2, switch the selection rule to the cutting-plane method (the minimizer of the model) and continue with tolerance $0$.

\begin{theorem}[Finite exactness of cutting-plane finishing]
    \label{theo:cpfinish_exact}
    In the setting of \Cref{lem:submodel_gap}(ii), let the tolerance of \Cref{alg:dc_inner} be $0$, $0<\epsilon_{\mathrm{sw}}<\delta(y^\prime)$, $\epsilon_{\mathrm{term}}=0$ in \Cref{alg:rule_cpfinish}, and $J=\infty$, and assume that the LPs are solved exactly. If Phase 1 reaches $\mathrm{UB}-\mathrm{LB}\le{\epsilon_{\mathrm{sw}}}$ in finitely many iterations (which holds for the cutting-plane method by \Cref{lem:dc_inner_finite} and, for the level method, whenever the iteration bound of \Cref{rem:level_complexity} holds, e.g., for the standard level method without inherited cuts), then cutting-plane finishing stops in finitely many iterations, and the returned $\hat{\theta}$ is an exact minimizer satisfying
    \[
        \bar{\ell}^{\sub}(\hat{\theta};y^\prime)=m_{\sub}(y^\prime)
        =\min_{\theta\in\Theta}\bar{\ell}^{\sub}(\theta;y^\prime) .
    \]
\end{theorem}

\begin{proof}
    At the end of Phase 1, since $\mathrm{UB}-\mathrm{LB}\le\epsilon_{\mathrm{sw}}<\delta(y^\prime)$, applying the proof of \Cref{lem:lb_snap} with $\epsilon_{\mathrm{in}}$ replaced by $\epsilon_{\mathrm{sw}}$ gives $\mathrm{LB}=m_{\sub}(y^\prime)$. In Phase 2, cuts are only added, so $\mathrm{LB}$ is nondecreasing, and by \Cref{lem:submodel_gap}(ii) it is always at most $m_{\sub}(y^\prime)$; hence $\mathrm{LB}=m_{\sub}(y^\prime)$ is maintained.
    Each iteration of Phase 2 evaluates at the model minimizer $\theta^{j}$ (a point attaining $\mathrm{LB}$). If the tuple $v_j$ obtained by the evaluation already belongs to the set $W$ corresponding to $\mathcal{G}$, then
    \[
        \bar{\ell}^{\sub}(\theta^{j};y^\prime)=\langle\theta^{j},g_{v_j}\rangle
        \le\max_{g\in\mathcal{G}}\langle\theta^{j},g\rangle=\mathrm{LB}=m_{\sub}(y^\prime) ,
    \]
    and since $\bar{\ell}^{\sub}(\cdot\,;y^\prime)\ge m_{\sub}(y^\prime)$ always holds (on $\Theta$), $\bar{\ell}^{\sub}(\theta^{j};y^\prime)=m_{\sub}(y^\prime)$. At this point the algorithm stops with $\mathrm{UB}=m_{\sub}(y^\prime)=\mathrm{LB}$, and $\hat{\theta}=\theta^{j}$ is an exact minimizer (it likewise stops if $\bar{\ell}^{\sub}(\theta^{j};y^\prime)=m_{\sub}(y^\prime)$ occurs with a new tuple). In every iteration that does not stop, a new tuple is added to the pool, and since $\mathcal{Y}^{\mathrm{all}}$ is finite, Phase 2 stops in finitely many iterations.
\end{proof}

\begin{remark}[Explicit number of iterations with the level method]
    \label{rem:level_complexity}
    For a convex Lipschitz function on a compact convex set, the number of iterations the level method needs to reach a certified gap $\mathrm{UB}-\mathrm{LB}\le\epsilon$ is bounded by $c(\lambda)\,\bigl(L_g(y^\prime)\,D_\Theta/\epsilon\bigr)^2$ ($\lambda\in(0,1)$ is the level parameter and $c(\lambda)$ is a constant depending only on it; \citealp{lemarechal1995new}). Here $L_g(y^\prime):=\max_{v\in\mathcal{Y}^{\mathrm{all}}}\|g_v\|$ is an upper bound on the Lipschitz constant of $\bar{\ell}^{\sub}(\cdot\,;y^\prime)$, and $D_\Theta:=\diam(\Theta)$. Therefore, choosing $\epsilon_{\mathrm{sw}}<\delta(y^\prime)$, the number of iterations of Phase 1 is of order $c(\lambda)\,\bigl(L_g(y^\prime)\,D_\Theta/\delta(y^\prime)\bigr)^2$ and is bounded explicitly, independently of the bound $M+1$ on the number of tuples in \Cref{lem:dc_inner_finite}. Note that this bound is a guarantee for the standard level method whose model is built only from tangents at the evaluation points. Using inherited cuts does not affect the validity of the stopping test (\Cref{lem:lb_snap}). Evaluating the number of iterations with inherited cuts is future work.
\end{remark}

\begin{lemma}[Total number of finishing iterations]
    \label{lem:phase2_total}
    Under \Cref{assu:margin}, call \Cref{alg:dc_inner} at line~\ref{line:dc_inner_call} of \Cref{alg:dc_outer} with tolerance $0$ and $J=\infty$, and let the rule of its Step 12 be \Cref{alg:rule_cpfinish}. Assume that the LPs of \Cref{alg:dc_inner} are solved exactly, and let $M:=\prod_{n=1}^N|\mathcal{Y}(s^n)|$.
    An inner iteration $j\ (\ge1)$ is called a finishing iteration if its evaluation point $\theta^{j}$ is chosen in the cutting-plane branch of \Cref{alg:rule_cpfinish} (the side $\mathrm{UB}-\mathrm{LB}\le\epsilon_{\mathrm{sw}}$).
    Then the total number of finishing iterations executed while \Cref{alg:dc_outer} runs up to $K$ outer iterations is at most $M+K$. This bound does not depend on the value of the switching threshold $\epsilon_{\mathrm{sw}}$.
\end{lemma}

\begin{proof}
    We use the notation of the proof of \Cref{lem:dc_inner_finite}: let $\mathcal{G}_j$ be the cut set at the end of inner iteration $j$, and $\mathrm{LB}_{j+1}=\min_{\theta\in\Theta}\max_{g\in\mathcal{G}_j}\langle\theta,g\rangle$.
    In a finishing iteration $j$, $\theta^{j}$ is the output of \Cref{alg:rule_kelley}, that is, a minimizer on $\Theta$ of the model of $\mathcal{G}_{j-1}$, so $\max_{g\in\mathcal{G}_{j-1}}\langle\theta^{j},g\rangle=\mathrm{LB}_j$.
    If the tuple $v_j$ returned by the oracle already belongs to the pool $\mathcal{V}$ at that time, then its cut (for the current branch $y^\prime$) equals $g_j$ and belongs to $\mathcal{G}_{j-1}$, so
    \[
        \bar{\ell}^{\sub}(\theta^{j};y^\prime)=\langle\theta^{j},g_j\rangle
        \le\max_{g\in\mathcal{G}_{j-1}}\langle\theta^{j},g\rangle=\mathrm{LB}_j\le\mathrm{LB}_{j+1} .
    \]
    After $\mathrm{UB}$ is updated, $\mathrm{UB}\le\bar{\ell}^{\sub}(\theta^{j};y^\prime)\le\mathrm{LB}_{j+1}$, so the stopping test with tolerance $0$ is satisfied and the call stops at this iteration.
    By contraposition, in a finishing iteration that does not stop, $v_j$ is necessarily a tuple newly added to the pool.
    In \Cref{alg:dc_outer}, the pool $\mathcal{V}$ increases monotonically across outer iterations and is never reinitialized (as in the proof of \Cref{lem:dc_inner_finite}(2)). Therefore, throughout the whole run, the number of finishing iterations that do not stop is at most $|\mathcal{V}|\le M$ (iterations of Phase 1 also consume the pool, but this only acts to decrease the number of finishing iterations).
    Each call has at most one stopping iteration, and the number of calls equals the number $K$ of outer iterations, so the total number of finishing iterations is at most $M+K$.
    The threshold $\epsilon_{\mathrm{sw}}$ does not appear in the above argument.
\end{proof}

\begin{proposition}[Total number of inner iterations over the whole run]
    \label{prop:total_inner_iterations}
    Under \Cref{assu:margin}, let the rule of Step 12 of \Cref{alg:dc_inner} be \Cref{alg:rule_cpfinish} (with switching threshold $\epsilon_{\mathrm{sw}}$, termination tolerance $\epsilon_{\mathrm{term}}=0$, and Phase-1 rule $\mathsf{R}$), and call \Cref{alg:dc_inner} with tolerance $0$ and $J=\infty$. Assume that the LPs are solved exactly, and let
    \[
        \delta_{\mathrm{in}}:=\min_{y^\prime\in\mathcal{Y}^{\mathrm{all}}_{\geq}}\delta(y^\prime)>0 ,
        \qquad
        0<\epsilon_{\mathrm{sw}}<\delta_{\mathrm{in}}
    \]
    (since $\mathcal{Y}^{\mathrm{all}}_{\geq}$ is finite, $\delta_{\mathrm{in}}>0$). Assume further that for every branch, the rule $\mathsf{R}$ reaches $\mathrm{UB}-\mathrm{LB}\le\epsilon_{\mathrm{sw}}$ in at most $T_1$ inner iterations. Then the following hold.
    \begin{description}
        \item[(1)] \Cref{alg:dc_outer} stops in at most $|\mathcal{Y}^{\mathrm{all}}_{\geq}|$ outer iterations, and each call returns an exact minimizer of $\bar{\ell}^{\sub}(\cdot\,;y^\prime)$.
        \item[(2)] The total number of inner iterations executed over the whole run is at most
            \[
                \underbrace{|\mathcal{Y}^{\mathrm{all}}_{\geq}|\,T_1}_{\text{Phase 1}}
                \;+\;
                \underbrace{M+|\mathcal{Y}^{\mathrm{all}}_{\geq}|}_{\text{finishing}}
            \]
            ($M:=\prod_{n=1}^N|\mathcal{Y}(s^n)|$).
        \item[(3)] The number of calls to the forward-problem oracle is at most $N\bigl(|\mathcal{Y}^{\mathrm{all}}_{\geq}|\,T_1+M+|\mathcal{Y}^{\mathrm{all}}_{\geq}|\bigr)+2N|\mathcal{Y}^{\mathrm{all}}_{\geq}|{}+N$.
    \end{description}
    When the level method (\Cref{alg:rule_level}) is taken as the rule $\mathsf{R}$, $T_1=c(\lambda)\bigl(L_g\,D_\Theta/\epsilon_{\mathrm{sw}}\bigr)^2$ ($L_g:=\max_{y^\prime\in\mathcal{Y}^{\mathrm{all}}_{\geq}}L_g(y^\prime)$, $D_\Theta:=\diam(\Theta)$; \Cref{rem:level_complexity}. This $T_1$ is the value when the bound of \Cref{rem:level_complexity} holds).
\end{proposition}

\begin{proof}
    We show (1). Since $\epsilon_{\mathrm{sw}}<\delta_{\mathrm{in}}\le\delta(y^\prime)$ holds for every branch $y^\prime\in\mathcal{Y}^{\mathrm{all}}_{\geq}$, each call satisfies the assumptions of \Cref{theo:cpfinish_exact}, returns an exact minimizer, and gives $\mathrm{converged}=\mathrm{TRUE}$. This is the premise of \Cref{prop:dc_finite_convergence} (exact solution of the inner problems), and by that proposition, \Cref{alg:dc_outer} stops in at most $|\mathcal{Y}^{\mathrm{all}}_{\geq}|$ outer iterations.

    We show (2). Split the inner iterations of one call into those up to the iteration at which $\mathrm{UB}-\mathrm{LB}\le\epsilon_{\mathrm{sw}}$ first holds and those after it. Since $\mathrm{UB}-\mathrm{LB}$ is nonincreasing in the iterations (\Cref{subsec:rule_cpfinish}), this split yields two intervals, a first half and a second half, and all evaluation points of the second half are chosen in the cutting-plane branch of \Cref{alg:rule_cpfinish}. That is, all iterations of the second half are finishing iterations in the sense of \Cref{lem:phase2_total}.
    By assumption, the number of iterations of the first half is at most $T_1$ in each call, and by (1) the number of calls is at most $|\mathcal{Y}^{\mathrm{all}}_{\geq}|$, so the total of the first halves is at most $|\mathcal{Y}^{\mathrm{all}}_{\geq}|T_1$. Substituting $K\le|\mathcal{Y}^{\mathrm{all}}_{\geq}|$ into \Cref{lem:phase2_total}, the total of the second halves is at most $M+|\mathcal{Y}^{\mathrm{all}}_{\geq}|$. Together, we obtain (2).

    We show (3). Calls to the forward-problem oracle occur $N$ times per inner iteration at Step 3 of \Cref{alg:dc_inner} and $2N$ times per outer iteration at Step 4 of \Cref{alg:dc_outer}, and $N$ times at the initialization of the branch in Step 1 of \Cref{alg:dc_outer}, and nowhere else. By (1) and (2), the first two are bounded by $N(|\mathcal{Y}^{\mathrm{all}}_{\geq}|T_1+M+|\mathcal{Y}^{\mathrm{all}}_{\geq}|)$ and $2N|\mathcal{Y}^{\mathrm{all}}_{\geq}|$, respectively, and the last is $N$.
\end{proof}

\begin{remark}[Roles of the two terms]
    \label{rem:two_terms}
    The first term of \Cref{prop:total_inner_iterations}(2) depends on the numerical fineness of the problem through the gap $\delta_{\mathrm{in}}$, whereas the second term $M+|\mathcal{Y}^{\mathrm{all}}_{\geq}|$ does not depend on the gap and is determined only by two combinatorial quantities: the number $M=|\mathcal{Y}^{\mathrm{all}}|$ of tuples of solutions that the oracle can return and the number $|\mathcal{Y}^{\mathrm{all}}_{\geq}|$ of branches ($\mathcal{Y}^{\mathrm{all}}=\prod_n\mathcal{Y}(s^n)$ and $\mathcal{Y}^{\mathrm{all}}_{\geq}=\prod_n\mathcal{Y}_{\geq x^n}(s^n)$ are different sets). Without cut inheritance, the second term would be $|\mathcal{Y}^{\mathrm{all}}_{\geq}|(M+1)$ (applying \Cref{lem:dc_inner_finite}(1) to each call), which differs by a factor of $|\mathcal{Y}^{\mathrm{all}}_{\geq}|$.
\end{remark}

\begin{remark}[Connection to the outer loop and future work]
    \label{rem:gap_outlook}
    First, the number of iterations and the number of oracle calls over the whole run are summarized in \Cref{prop:total_inner_iterations}.
    Second, a lower bound on $\delta_{\mathrm{in}}$ is given explicitly for the integer case in \Cref{subsec:delta_explicit}. Improvements using the structure of integer programs (total unimodularity, Graver bases, M-convexity) are left for future work, as described in \Cref{rem:structural_improvement}.
    Third, the finite exact attainment of \citet{kitaoka2024exact} uses the structure in which the set of minimizers has a relative interior point, but in the inner problem of a general branch, the $\argmin$ can be a low-dimensional face; hence, in this paper, we combine the discreteness of the lower bound (\Cref{lem:lb_snap}) with finishing by the cutting-plane method (\Cref{theo:cpfinish_exact}) rather than relying on the convergence of the values at the evaluation points.
\end{remark}

\subsection{Explicit lower bound on the submodel gap in the integer case}
\label{subsec:delta_explicit}

With the same motivation as \citet{kitaoka2026explicit}, who gave an explicit lower bound on the constant $\gamma(\ell^{\sub})$ (a positive constant determined by the geometry of the set of minimizers of the suboptimality loss, which bounds by $O(1/\gamma(\ell^{\sub})^2)$ the number of iterations a gradient method needs to reach $\ell^{\sub}=0$; \citealp{kitaoka2024exact}), we give an explicit lower bound on $\delta(y^\prime)$. The key is that all submodel minimum values $m_W(y^\prime)$ are rational numbers with a common upper bound on their denominators; although there are exponentially many subsets $W$, enumerating them is unnecessary (only the upper bound on the denominators matters).

\begin{proposition}[Explicit lower bound on $\delta$ in the integer case]
    \label{prop:delta_explicit_lb}
    In addition to \Cref{assu:margin}, let $N=1$, $\Theta=\Delta^{d-1}$, and $\mathcal{X}(s^1)\subseteq\mathbb{Z}^d$, and let $R:=\max_{x\in\mathcal{X}(s^1)}\|x\|_\infty$. Then for every branch $y^\prime\in\mathcal{Y}^{\mathrm{all}}_{\geq}$,
    \[
        \delta(y^\prime)\ \ge\ \frac{1}{Q^2} ,
        \qquad
        Q:=\bigl((d+1)^{1/2}\max(2R,1)\bigr)^{d+1} .
    \]
\end{proposition}

\begin{proof}
    First, we show that for every nonempty $W\subseteq\mathcal{Y}^{\mathrm{all}}$, $m_W(y^\prime)$ is a rational number with denominator at most $Q$. Since $N=1$, each cut $g_v=v-y^\prime$ ($v,y^\prime\in\mathcal{X}(s^1)\subseteq\mathbb{Z}^d$) is an integer vector with $\|g_v\|_\infty\le2R$. $m_W(y^\prime)$ is the optimal value of the linear program
    \[
        \min\bigl\{\,\tau \ \bigm|\ (\theta,\tau)\in P_W\,\bigr\} ,
        \quad
        P_W:=\Bigl\{(\theta,\tau)\in\mathbb{R}^d\times\mathbb{R} \ \Bigm|\ \begin{aligned}&\theta\ge0,\ \textstyle\sum_{i=1}^d\theta_i=1,\\ &\langle g_v,\theta\rangle\le\tau\ (\forall v\in W)\end{aligned}\Bigr\} .
    \]
    $P_W$ is nonempty and contains no line ($\theta$ is bounded, and the direction $(0,-1)$ is incompatible with the constraints $\langle g_v,\theta\rangle\le\tau$), and the optimal value is finite by the compactness of $\Theta$, so it is attained at some extreme point $(\theta,\tau)$. An extreme point is the unique solution of a nonsingular system of linear equations consisting of $(d+1)$ linearly independent tight constraints, and each row of its coefficient matrix $A$ is an integer vector of type $(g_v,-1)$, $(\mathbf{1},0)$, or $(-e_i,0)$. The $\ell_2$ norms of the rows are bounded by $\sqrt{d(2R)^2+1}\le(d+1)^{1/2}\max(2R,1)$, $\sqrt{d}$, and $1$, respectively, so by Hadamard's inequality,
    \[
        1\le|\det A|\le\bigl((d+1)^{1/2}\max(2R,1)\bigr)^{d+1}=Q
    \]
    (the lower bound holds because $A$ is a nonsingular integer matrix). By Cramer's rule, $\tau$ is a rational number with denominator $|\det A|\le Q$, and we can write $m_W(y^\prime)=p/q$ ($p\in\mathbb{Z}$, $1\le q\le Q$).

    If $\delta(y^\prime)=+\infty$, the claim holds. Otherwise, there exists $W$ attaining the $\max$ in \cref{eq:def_submodel_gap}, and $\delta(y^\prime)=m_{\sub}(y^\prime)-m_W(y^\prime)>0$. Since $m_{\sub}(y^\prime)=m_{\mathcal{Y}^{\mathrm{all}}}(y^\prime)$ is also of the above form, we can write $m_{\sub}(y^\prime)=p_1/q_1$ and $m_W(y^\prime)=p_2/q_2$ ($q_1,q_2\le Q$); since the two are distinct,
    \[
        \delta(y^\prime)=\frac{|p_1q_2-p_2q_1|}{q_1q_2}\ge\frac{1}{q_1q_2}\ge\frac{1}{Q^2} .
    \]
\end{proof}

\begin{corollary}[Extension to $N\ge2$]
    \label{cor:delta_scaling}
    In addition to \Cref{assu:margin}, let $N\ge2$, $\Theta=\Delta^{d-1}$, and $\mathcal{X}(s^n)\subseteq\mathbb{Z}^d$ for each $n$, and let $R:=\max_{n}\max_{x\in\mathcal{X}(s^n)}\|x\|_\infty$. Then for every branch $y^\prime\in\mathcal{Y}^{\mathrm{all}}_{\geq}$, we have
    \[
        \delta(y^\prime)\ \ge\ \frac{1}{N\,Q_N^2} ,
        \qquad
        Q_N:=\bigl((d+1)^{1/2}\max(2NR,1)\bigr)^{d+1} .
    \]
    As in the lower bound on $\gamma(\ell^{\sub})$ \citep{kitaoka2026explicit}, the dependence on $N$ appears through $Q_N$ and the factor $1/N$.
\end{corollary}

\begin{proof}
    For the cut $g_v=\frac{1}{N}\sum_{n=1}^N(v^n-y^{\prime n})$, we have $Ng_v\in\mathbb{Z}^d$ and $\|Ng_v\|_\infty\le2NR$. Applying the first step of the proof of \Cref{prop:delta_explicit_lb} to the family of integer cuts $\{Ng_v\}_{v}$, for every nonempty $W\subseteq\mathcal{Y}^{\mathrm{all}}$,
    \[
        N\,m_W(y^\prime)=\min_{\theta\in\Delta^{d-1}}\max_{v\in W}\langle\theta,Ng_v\rangle
    \]
    is a rational number with denominator at most $Q_N$. Hence the difference of two distinct values of $N\,m_W(y^\prime)$ is at least $1/Q_N^2$, and the difference of two distinct values of $m_W(y^\prime)$ is at least $1/(N\,Q_N^2)$. The rest follows as in the second step of the proof of \Cref{prop:delta_explicit_lb}.
\end{proof}

\begin{remark}[Fully explicit number of iterations of Phase 1]
    \label{rem:explicit_phase1}
    In the setting of \Cref{prop:delta_explicit_lb}, choosing $\epsilon_{\mathrm{sw}}:=1/(2Q^2)<\delta(y^\prime)$ and combining with \Cref{rem:level_complexity}, the number of iterations of Phase 1 (level method) is bounded by
    \[
        c(\lambda)\bigl(L_g(y^\prime)\,D_\Theta\,\cdot 2Q^2\bigr)^2
        \le c(\lambda)\,\bigl(2R\sqrt{d}\cdot\sqrt{2}\cdot2Q^2\bigr)^2 ,
    \]
    a fully explicit function of the problem size ($d$, $R$) (using $L_g(y^\prime)\le2R\sqrt{d}$ and $\diam(\Delta^{d-1})=\sqrt{2}$). Since $Q^4=\bigl((d+1)\max(2R,1)^2\bigr)^{2(d+1)}$, this bound is exponential in $d$. This is consistent with the fact that the lower bound on $\gamma(\ell^{\sub})$ for general ILPs can be exponentially small in $d$ (and that this is essential) \citep{kitaoka2026explicit}.
\end{remark}

\begin{corollary}[Explicit number of iterations over the whole run in the integer case]
    \label{cor:total_integer}
    In the setting of \Cref{prop:delta_explicit_lb} for $N=1$ and of \Cref{cor:delta_scaling} for $N\ge2$ ($\Theta=\Delta^{d-1}$, $\mathcal{X}(s^n)\subseteq\mathbb{Z}^d$ for each $n$, $R:=\max_n\max_{x\in\mathcal{X}(s^n)}\|x\|_\infty$), use the framework of \Cref{prop:total_inner_iterations} with the rule $\mathsf{R}$ being the level method and the switching threshold
    \[
        \epsilon_{\mathrm{sw}}:=\frac{1}{2NQ_N^2} ,
        \qquad
        Q_N=\bigl((d+1)^{1/2}\max(2NR,1)\bigr)^{d+1} .
    \]
    Assume that the iteration bound of \Cref{rem:level_complexity} holds for the level method in Phase 1. Then $\epsilon_{\mathrm{sw}}<\delta_{\mathrm{in}}$, and the total number of inner iterations over the whole run is at most
    \begin{align*}
        &|\mathcal{Y}^{\mathrm{all}}_{\geq}|\,c(\lambda)\,\bigl(2R\sqrt{d}\cdot\sqrt{2}\cdot {2NQ_N^2}\bigr)^2 + M+|\mathcal{Y}^{\mathrm{all}}_{\geq}|\\
        &\qquad=
        32\,c(\lambda)\,|\mathcal{Y}^{\mathrm{all}}_{\geq}|\,R^2\,d\,N^2\,Q_N^4 \;+\; M+|\mathcal{Y}^{\mathrm{all}}_{\geq}| ,
    \end{align*}
    which is a fully explicit function of the problem size ($d$, $R$, $N$), $|\mathcal{Y}^{\mathrm{all}}_{\geq}|$, and $M$.
\end{corollary}

\begin{proof}
    By \Cref{cor:delta_scaling}, $\delta(y^\prime)\ge1/(NQ_N^2)$ for every branch, so $\delta_{\mathrm{in}}\ge1/(NQ_N^2)>\epsilon_{\mathrm{sw}}$.
    For the Lipschitz constant, $\|g_v\|=\bigl\|\frac{1}{N}\sum_n(v^n-y^{\prime n})\bigr\|\le\max_n\|v^n-y^{\prime n}\|\le2R\sqrt{d}$ gives $L_g\le2R\sqrt{d}$, and $D_\Theta=\diam(\Delta^{d-1})=\sqrt{2}$.
    By \Cref{rem:level_complexity}, $T_1=c(\lambda)(L_gD_\Theta/\epsilon_{\mathrm{sw}})^2\le c(\lambda)\bigl(2R\sqrt{d}\cdot\sqrt{2}\cdot2NQ_N^2\bigr)^2=32\,c(\lambda)R^2dN^2Q_N^4$, and it suffices to substitute this into \Cref{prop:total_inner_iterations}(2).
\end{proof}

\begin{remark}[Which term dominates]
    \label{rem:which_term_dominates}
    Since $Q_N^4=\bigl((d+1)\max(2NR,1)^2\bigr)^{2(d+1)}$, the first term of \Cref{cor:total_integer} is exponential in $d$. On the other hand, the second term $M+|\mathcal{Y}^{\mathrm{all}}_{\geq}|$ depends neither on $\delta_{\mathrm{in}}$ nor on any exponent of $R$, and is determined only by the number $M=|\mathcal{Y}^{\mathrm{all}}|$ of tuples of solutions that the oracle can return and the number $|\mathcal{Y}^{\mathrm{all}}_{\geq}|$ of branches. Which term dominates depends on the ratio $M/|\mathcal{Y}^{\mathrm{all}}_{\geq}|=\prod_{n=1}^N|\mathcal{Y}(s^n)|/|\mathcal{Y}_{\geq x^n}(s^n)|$. If this ratio is at most $32\,c(\lambda)R^2dN^2Q_N^4$, the first term dominates, and the improvement on the finishing side by \Cref{lem:phase2_total} ($|\mathcal{Y}^{\mathrm{all}}_{\geq}|(M+1)\to M+|\mathcal{Y}^{\mathrm{all}}_{\geq}|$) does not change the overall order; in this case, the only way to lower the overall order is to improve the lower bound on $\delta_{\mathrm{in}}$. On the other hand, for fixed $d$, $32\,c(\lambda)R^2dN^2Q_N^4$ is polynomial in $N$, whereas the ratio can grow exponentially in $N$; hence, when $N$ is large compared with $d$, the second term can dominate, and then the improvement by \Cref{lem:phase2_total} reduces the bound by a factor of up to about $|\mathcal{Y}^{\mathrm{all}}_{\geq}|$. Given that the lower bound on $\gamma(\ell^{\sub})$ for general ILPs can be exponentially small in $d$ (and that this is essential) \citep{kitaoka2026explicit}, improvements using the structure of the forward problem are given for the test-direction-cut inner solver in \Cref{subsec:testset_gap,subsec:testset_gap_N}; improving the lower bound on $\delta_{\mathrm{in}}$ for \Cref{alg:dc_inner} itself is future work (\Cref{rem:testcut_scope}).
\end{remark}

\begin{remark}[Improvement by structure is left for future work]
    \label{rem:structural_improvement}
    \citet{kitaoka2026explicit} replaced the dependence of the lower bound on $\gamma(\ell^{\sub})$ on the range of the features (which correspond to the feasible solutions in this paper) by structural constants, using structures of the forward problem such as total unimodularity, Graver bases, and M-convexity. A similar improvement is expected for $\delta(y^\prime)$, but the minors in \Cref{prop:delta_explicit_lb} are those of the cut matrix $(g_v)_v$, and properties such as the total unimodularity of the constraint matrix of the forward problem are not immediately inherited by the cut matrix. This improvement is future work. Restricting attention to the value gap of the branch minimum value $m_{\sub}(y^\prime)$ at $0$ (a positive lower bound on $m_{\sub}(y^\prime)$ when $m_{\sub}(y^\prime)>0$), an explicit lower bound via Graver bases (test sets) is obtained in \Cref{subsec:testset_gap,subsec:testset_gap_N}, and for a variant of the inner solver restricted to test-direction cuts, gap lower bounds independent of $R$ (polynomial in the dimension in the M$^\natural$-convex case) are obtained in \Cref{prop:testcut_gap,prop:testcut_mconvex}.
\end{remark}

\subsection{Explicit lower bounds via test sets (single state)}
\label{subsec:testset_gap}

The lower bound of \Cref{subsec:delta_explicit} is of order $R^{-2(d+1)}$ in the range $R$ of the features and is exponentially small in the dimension $d$. In this subsection, using the integer structure of the forward problem (Graver bases), we bound from below explicitly, only in terms of the norm of the test set, the value gap of the branch minimum value $m_{\sub}(y^\prime)$ at $0$ instead of the submodel gap $\delta(y^\prime)$. In the second half, we define a variant of the inner solver whose cut generation is restricted to test directions, and give a lower bound on its gap independent of $R$ (a lower bound polynomial in the dimension in the M$^\natural$-convex case).
In this subsection, let $N=1$ (the extension to $N\ge2$ is given in \Cref{subsec:testset_gap_N}) and $\Theta=\Delta^{d-1}$, and suppose that the feasible region of the forward problem is given in the form
\[
    \mathcal{X}(s^1)=\bigl\{x\in\mathbb{Z}^d \ \bigm|\ Ax\le b,\ 0\le x\le u\bigr\}
    \qquad(A\in\mathbb{Z}^{n_A\times d},\ b\in\mathbb{Z}^{n_A},\ u\in\mathbb{Z}_{\ge0}^d)
\]
(a special case of \Cref{assu:margin}(2)). As in \Cref{prop:delta_explicit_lb}, let $R:=\max_{x\in\mathcal{X}(s^1)}\|x\|_\infty$. Introducing slack variables and setting $\tilde{A}:=[A \mid I_{n_A}]\in\mathbb{Z}^{n_A\times(d+n_A)}$ and $z(x):=(x,\,b-Ax)\in\mathbb{Z}^{d+n_A}$, $x\in\mathcal{X}(s^1)$ is equivalent to $\tilde{A}z=b$ and $0\le z\le\tilde{u}$ ($\tilde{u}$ is a finite upper bound determined by the box $0\le x\le u$). Let $\mathcal{T}:=\mathrm{Gr}(\tilde{A})$ be the Graver basis of $\tilde{A}$ (the set of all nonzero elements of $\ker_{\mathbb{Z}}\tilde{A}$ that are minimal with respect to the conformal order (for $t,w\in\mathbb{Z}^{d+n_A}$, $t\sqsubseteq w$ if $t_lw_l\ge0$ and $|t_l|\le|w_l|$ for every component $l$); \citealp{onn2010nonlinear}), and let $C:=\max_{t\in\mathcal{T}}\|t\|_\infty$. We write the $x$ component of $t\in\mathcal{T}$ as $t^x\in\mathbb{Z}^d$ ($\|t^x\|_\infty\le C$). We further set $C_R:=\min(C,R)$. Since $\mathcal{X}(s^1)\subseteq[0,R]^d$, every test direction $t$ such that both $y^\prime$ and $y^\prime+t^x$ belong to $\mathcal{X}(s^1)$ satisfies $\|t^x\|_\infty\le C_R$. The Graver basis is a \emph{test set} for linear optimization over $\mathcal{X}(s^1)$; that is, for any $\theta\in\mathbb{R}^d$ and $y\in\mathcal{X}(s^1)$, $y$ is a maximizer of $\langle\theta,\cdot\rangle$ if and only if there exists no $t\in\mathcal{T}$ with $0\le z(y)+t\le\tilde{u}$ and $\langle\theta,t^x\rangle>0$ \citep{onn2010nonlinear}.
For a branch $y^\prime\in\mathcal{Y}^{\mathrm{all}}_{\geq}$ (since $N=1$, $y^\prime\in\mathcal{Y}_{\geq x^1}(s^1)\subseteq\mathcal{X}(s^1)$), we define the test-set model by
\begin{equation}
    h_{\mathcal{T}}(\theta)
    :=\max\Bigl(0,\ \max\bigl\{\langle\theta,t^x\rangle \ \bigm|\ t\in\mathcal{T},\ z(y^\prime)+t\ \text{is feasible}\bigr\}\Bigr)
    \label{eq:def_testset_model}
\end{equation}
(feasible means satisfying the equality constraints of $\tilde{A}$ and the box constraints; if there is no such $t$, we read the inner $\max$ as $-\infty$, so that $h_{\mathcal{T}}=0$).

For the sandwich bound, we use the following property of Graver bases. Here, let $D:=d+n_A$, and for $t,w\in\mathbb{Z}^D$, $t\sqsubseteq w$ means that $t_lw_l\ge0$ and $|t_l|\le|w_l|$ for every component $l$ (conformality, i.e., sign compatibility).

\begin{proposition}[Conformal decomposition by the Graver basis; \citealp{onn2010nonlinear}]
    \label{prop:graver_conformal}
    For an integer matrix $\tilde{A}$ and $w\in\ker_{\mathbb{Z}}\tilde{A}\setminus\{0\}$, there exist positive integers $\alpha_1,\ldots,\alpha_r$ and $t_1,\ldots,t_r\in\mathrm{Gr}(\tilde{A})$ such that $w=\sum_{i=1}^r\alpha_it_i$ and $t_i\sqsubseteq w$ for each $i$.
\end{proposition}

\begin{lemma}[Sandwich by the test-set model]
    \label{lem:testset_sandwich}
    Let
    \[
        D_z:=\max\bigl(1,\ \max_{x,x^\prime\in\mathcal{X}(s^1)}\|z(x^\prime)-z(x)\|_1\bigr).
    \]
    For every $\theta\in\mathbb{R}^d$,
    \[
        h_{\mathcal{T}}(\theta)\ \le\ \bar{\ell}^{\sub}(\theta;y^\prime)\ \le\ D_z\,h_{\mathcal{T}}(\theta) .
    \]
\end{lemma}

\begin{proof}
    Lower side: if $z(y^\prime)+t$ is feasible, then its $x$ component $y^\prime+t^x$ belongs to $\mathcal{X}(s^1)$, so $\bar{\ell}^{\sub}(\theta;y^\prime)=\max_{v\in\mathcal{X}(s^1)}\langle\theta,v-y^\prime\rangle\ge\langle\theta,t^x\rangle$. Moreover, $\bar{\ell}^{\sub}(\theta;y^\prime)\ge0$ since $y^\prime\in\mathcal{X}(s^1)$. Together, $\bar{\ell}^{\sub}(\theta;y^\prime)\ge h_{\mathcal{T}}(\theta)$.

    Upper side: let $v\in\mathcal{X}(s^1)$ be a maximizer of $\langle\theta,\cdot-y^\prime\rangle$. Since $z(v)-z(y^\prime)\in\ker_{\mathbb{Z}}\tilde{A}$, if $z(v)\neq z(y^\prime)$, then by \Cref{prop:graver_conformal} we can write
    \[
        z(v)-z(y^\prime)=\sum_i\alpha_i t_i ,
        \qquad
        \alpha_i\in\mathbb{Z}_{>0},\ t_i\in\mathcal{T},\ t_i\sqsubseteq z(v)-z(y^\prime)
    \]
    ($\sqsubseteq$ denotes componentwise conformality). By conformality, $z(y^\prime)+t_i$ lies componentwise between $z(y^\prime)$ and $z(v)$ and satisfies the box constraints, and since $\tilde{A}(z(y^\prime)+t_i)=b$, it is feasible. Hence $\langle\theta,t_i^x\rangle\le h_{\mathcal{T}}(\theta)$ for each $i$; moreover, by conformality, $\sum_i\alpha_i\|t_i\|_1=\|z(v)-z(y^\prime)\|_1\le D_z$, and since $\|t_i\|_1\ge1$ for each $i$, $\sum_i\alpha_i\le D_z$. Therefore,
    \[
        \bar{\ell}^{\sub}(\theta;y^\prime)=\Bigl\langle\theta,\sum_i\alpha_i t_i^x\Bigr\rangle
        =\sum_i\alpha_i\langle\theta,t_i^x\rangle
        \le\Bigl(\sum_i\alpha_i\Bigr)\,h_{\mathcal{T}}(\theta)
        \le D_z\,h_{\mathcal{T}}(\theta) .
    \]
\end{proof}

\begin{proposition}[Graver lower bound on the branch value gap]
    \label{prop:testset_value_gap}
    In the above setting, we have
    \[
        m_{\sub}(y^\prime)\in\{0\}\cup\Bigl[\frac{1}{Q_{\mathcal{T}}},\ \infty\Bigr) ,
        \qquad
        Q_{\mathcal{T}}:=\bigl((d+1)^{1/2}\max(C_R,1)\bigr)^{d+1} .
    \]
    In particular, since $C_R\le C$, we have $Q_{\mathcal{T}}\le((d+1)^{1/2}\max(C,1))^{d+1}$, so this lower bound is bounded independently of $R$ by the norm $C$ of the Graver basis. Moreover, since $C_R\le R$, $Q_{\mathcal{T}}$ is at most $Q$ of \Cref{prop:delta_explicit_lb}. For example, if $A$ (and hence $\tilde{A}=[A\mid I_{n_A}]$) is totally unimodular, then $C=1$ (\citealp{onn2010nonlinear}) and $Q_{\mathcal{T}}\le(d+1)^{(d+1)/2}$.
\end{proposition}

\begin{proof}
    If $m_{\sub}(y^\prime)=0$, there is nothing to show. Suppose $m_{\sub}(y^\prime)>0$. By the upper side of \Cref{lem:testset_sandwich}, for every $\theta\in\Delta^{d-1}$, $h_{\mathcal{T}}(\theta)\ge\bar{\ell}^{\sub}(\theta;y^\prime)/D_z\ge m_{\sub}(y^\prime)/D_z>0$, so $\mu:=\min_{\theta\in\Delta^{d-1}}h_{\mathcal{T}}(\theta)>0$. $\mu$ is the optimal value of the linear program
    \[
        \min\Bigl\{\tau \ \Bigm|\ \begin{aligned}&\theta\ge0,\ \textstyle\sum_i\theta_i=1,\ \tau\ge0,\\ &\langle t^x,\theta\rangle\le\tau\ (t\in\mathcal{T}\text{ such that }z(y^\prime)+t\text{ is feasible})\end{aligned}\Bigr\} ,
    \]
    and the rows of its coefficient matrix are integer vectors of type $(t^x,-1)$ ($\ell_2$ norm $\le\sqrt{dC_R^2+1}\le(d+1)^{1/2}\max(C_R,1)$), $(\mathbf{0},-1)$, $(\mathbf{1},0)$, or $(-e_i,0)$. By the same extreme-point, Cramer, and Hadamard argument as in the proof of \Cref{prop:delta_explicit_lb}, $\mu$ is a rational number with denominator at most $Q_{\mathcal{T}}$. Since $\mu>0$, $\mu\ge1/Q_{\mathcal{T}}$, and by the lower side of \Cref{lem:testset_sandwich},
    \[
        m_{\sub}(y^\prime)=\min_{\theta\in\Delta^{d-1}}\bar{\ell}^{\sub}(\theta;y^\prime)\ \ge\ \mu\ \ge\ \frac{1}{Q_{\mathcal{T}}} .
    \]
    Note that $D_z$ is used only to derive $\mu>0$ and does not appear in the value of the lower bound itself.
\end{proof}

The decomposition into test directions in \Cref{lem:testset_sandwich} and the estimates via extreme points, Cramer's rule, and Hadamard's inequality in \Cref{prop:delta_explicit_lb,prop:testset_value_gap} are similar to the arguments used for the lower bound on $\gamma(\ell^{\sub})$ in \citet[Lemma~6.9, Theorem~6.24]{kitaoka2026explicit}; in this section, we apply them to $\delta(y^\prime)$, $m_{\sub}(y^\prime)$, and $\mu(y^\prime)$ (defined after \cref{eq:def_testcut_gap}).

We call a branch $y^\prime$ with $m_{\sub}(y^\prime)=0$ a \emph{zero branch}.

\begin{corollary}[Exact decision of zero branches]
    \label{cor:zero_branch_cert}
    In the above setting, if $\bar{\ell}^{\sub}(\bar{\theta};y^\prime)<1/Q_{\mathcal{T}}$ is observed for some $\bar{\theta}\in\Delta^{d-1}$, then $m_{\sub}(y^\prime)=0$. In particular, $m_{\sub}(y^\prime)=0$ is established as soon as $\mathrm{UB}<1/Q_{\mathcal{T}}$ during the execution of \Cref{alg:dc_inner}. Furthermore, for $N=1$, by \cref{eq:dc_min_min}, $\min_{\theta\in\Theta}\bar{\ell}^{\optdex}(\theta)=\min_{y^{\prime\prime}\in\mathcal{Y}^{\mathrm{all}}_{\geq}}m_{\sub}(y^{\prime\prime})$, so this gives a decision of the realizable case (global minimum value $0$) with the explicit tolerance $1/Q_{\mathcal{T}}$ (for $N\ge2$, see \Cref{cor:zero_branch_cert_N}).
\end{corollary}

\begin{proof}
    This follows from $m_{\sub}(y^\prime)\le\bar{\ell}^{\sub}(\bar{\theta};y^\prime)<1/Q_{\mathcal{T}}$ and \Cref{prop:testset_value_gap}. The same holds for $\mathrm{UB}$, since it is the minimum of the values of $\bar{\ell}^{\sub}(\cdot\,;y^\prime)$ at the evaluation points.
\end{proof}

Next, we define a variant of the inner solver whose cut generation is restricted to test directions (\Cref{def:testcut_variant}), and give a lower bound on its gap independent of $R$ (\Cref{prop:testcut_gap}), a lower bound polynomial in the dimension in the M-convex and M$^\natural$-convex cases (\Cref{prop:testcut_mconvex}), and the connection to the decision of zero branches (\Cref{cor:testcut_zero_decision}).

\begin{definition}[Test-direction-cut inner solver]
    \label{def:testcut_variant}
    Let $\mathcal{T}_{y^\prime}:=\{t^x \mid t\in\mathcal{T},\ z(y^\prime)+t\ \text{is feasible}\}$.
    A test-direction oracle is an oracle that, given $\theta\in\mathbb{R}^d$, returns the value $\max_{g\in\mathcal{T}_{y^\prime}}\langle\theta,g\rangle$ and a cut $g\in\mathcal{T}_{y^\prime}$ attaining the maximum (a maximizing cut) (its implementability depends on the structure: in the M-convex case it reduces to feasibility checks of at most $d(d-1)$ directions, and of at most $d(d+1)$ directions in the M$^\natural$-convex case, whereas for a general Graver basis it can be implemented by computing the Graver basis (\citealp{onn2010nonlinear,deloera2012algebraic}) and checking the feasibility of each of its elements).
    The test-direction-cut inner solver is \Cref{alg:dc_inner} with the objective function replaced by $h_{\mathcal{T}}$ (\cref{eq:def_testset_model}), cut generation replaced from the forward-problem oracle by the test-direction oracle, and the zero cut $g=\mathbf{0}$ (corresponding to $h_{\mathcal{T}}\ge0$) always added as an initial cut.
\end{definition}

$h_{\mathcal{T}}$ is the maximum of the cuts of the finite family $\mathcal{T}_{y^\prime}\cup\{\mathbf{0}\}$, and the test-direction oracle returns the value and a maximizing cut at each evaluation point, so the proofs of \Cref{lem:submodel_gap,lem:lb_snap,theo:cpfinish_exact} go through verbatim with the family of cuts $\{g_v \mid v\in\mathcal{Y}^{\mathrm{all}}\}$ replaced by $\mathcal{T}_{y^\prime}\cup\{\mathbf{0}\}$. We define the corresponding submodel gap by
\begin{equation}
    \delta_{\mathcal{T}}(y^\prime)
    :=\mu(y^\prime)-\max\bigl\{\,m^{\mathcal{T}}_{G}(y^\prime)\ \bigm|\ \emptyset\neq G\subseteq\mathcal{T}_{y^\prime}\cup\{\mathbf{0}\},\ m^{\mathcal{T}}_{G}(y^\prime)<\mu(y^\prime)\,\bigr\} .
    \label{eq:def_testcut_gap}
\end{equation}
Here $m^{\mathcal{T}}_{G}(y^\prime):=\min_{\theta\in\Delta^{d-1}}\max_{g\in G}\langle\theta,g\rangle$ and $\mu(y^\prime):=\min_{\theta\in\Delta^{d-1}}h_{\mathcal{T}}(\theta)$, and as in \cref{eq:def_submodel_gap}, if there is no such $G$, we read $\delta_{\mathcal{T}}(y^\prime):=+\infty$.

\begin{proposition}[Lower bound on the test-direction gap independent of $R$]
    \label{prop:testcut_gap}
    In the setting of \Cref{prop:testset_value_gap}, we have
    \[
        \delta_{\mathcal{T}}(y^\prime)\ \ge\ \frac{1}{Q_{\mathcal{T}}^2} .
    \]
    In particular, since $Q_{\mathcal{T}}\le((d+1)^{1/2}\max(C,1))^{d+1}$, this lower bound is bounded by a quantity independent of $R$. Furthermore, when the test-direction-cut inner solver is run with $0\le\epsilon_{\mathrm{in}}<\delta_{\mathcal{T}}(y^\prime)$, analogues of \Cref{lem:lb_snap,theo:cpfinish_exact} hold: at stopping with $\mathrm{UB}-\mathrm{LB}\le\epsilon_{\mathrm{in}}$, $\mathrm{LB}=\mu(y^\prime)$ is identified exactly, and finishing with the selection rule of the cutting-plane method reaches an exact minimizer of $h_{\mathcal{T}}$ in finitely many iterations.
\end{proposition}

\begin{proof}
    Each $m^{\mathcal{T}}_{G}(y^\prime)$ is the optimal value of an LP consisting of integer cuts whose components have absolute values at most $C_R$, so by the same extreme-point, Cramer, and Hadamard argument as in the proof of \Cref{prop:testset_value_gap}, it is a rational number with denominator at most $Q_{\mathcal{T}}$. The difference of two distinct values is at least $1/Q_{\mathcal{T}}^2$, which gives the first half. The second half holds because the proofs of \Cref{lem:submodel_gap,lem:lb_snap,theo:cpfinish_exact} use only ``the family of cuts is finite'' and ``the oracle returns a maximizing cut at each evaluation point''.
\end{proof}

Next, we consider the case where the feasible region of the forward problem is an M-convex set. A set $\mathcal{M}\subseteq\mathbb{Z}^d$ is an \emph{M-convex set} if $\mathcal{M}$ is nonempty and, for any $x,y\in \mathcal{M}$ and any $i$ with $x_i>y_i$, there exists $j$ with $x_j<y_j$ such that $x-e_i+e_j\in \mathcal{M}$ and $y+e_i-e_j\in \mathcal{M}$ (exchange axiom). A set $\mathcal{M}\subseteq\mathbb{Z}^d$ is an \emph{M$^\natural$-convex set} if $\mathcal{M}$ is nonempty and, for any $x,y\in\mathcal{M}$ and any $i$ with $x_i>y_i$, either $x-e_i\in\mathcal{M}$ and $y+e_i\in\mathcal{M}$, or there exists $j$ with $x_j<y_j$ such that $x-e_i+e_j\in\mathcal{M}$ and $y+e_i-e_j\in\mathcal{M}$ \citep{murota2003discrete}. Every M-convex set is M$^\natural$-convex.
For linear optimization over M-convex sets, the following optimality criterion holds.

\begin{proposition}[Optimality criterion for linear optimization over M-convex sets; \citealp{murota2003discrete}]
    \label{prop:mconvex_optimality}
    For a bounded M-convex set $\mathcal{M}\subset\mathbb{Z}^d$, $\theta\in\mathbb{R}^d$, and $y\in \mathcal{M}$, $y$ is a maximizer of $\langle\theta,\cdot\rangle$ over $\mathcal{M}$ if and only if $\langle\theta,e_j-e_i\rangle\le0$ for all $i\neq j$ with $y+e_j-e_i\in \mathcal{M}$.
\end{proposition}

\begin{proposition}[Lower bound polynomial in the dimension for the M-convex and M$^\natural$-convex cases]
    \label{prop:testcut_mconvex}
    Consider the case where $\mathcal{X}(s^1)\subset\mathbb{Z}^d$ is a bounded M-convex set, let $\mathcal{T}_{\mathrm{M}}:=\{e_i-e_j \mid i\neq j\}$ and $\mathcal{T}_{\mathrm{M},y^\prime}:=\{g\in\mathcal{T}_{\mathrm{M}} \mid y^\prime+g\in\mathcal{X}(s^1)\}$, and replace $\mathcal{T}_{y^\prime}$ in \cref{eq:def_testset_model} and \cref{eq:def_testcut_gap} by $\mathcal{T}_{\mathrm{M},y^\prime}$. Then:

    (a) $\mathcal{T}_{\mathrm{M}}$ is a test set for linear optimization over $\mathcal{X}(s^1)$. That is, for every $\theta\in\mathbb{R}^d$ and $y\in\mathcal{X}(s^1)$, $y$ is a maximizer of $\langle\theta,\cdot\rangle$ if and only if there is no $g\in\mathcal{T}_{\mathrm{M}}$ with $y+g\in\mathcal{X}(s^1)$ and $\langle\theta,g\rangle>0$.

    (b) For every nonempty $G\subseteq\mathcal{T}_{\mathrm{M},y^\prime}\cup\{\mathbf{0}\}$, $m^{\mathcal{T}}_{G}(y^\prime)$ is a rational number with denominator at most $d^2$; hence we have
    \[
        \delta_{\mathcal{T}_{\mathrm{M}}}(y^\prime)\ \ge\ \frac{1}{d^4} ,
    \]
    and we have
    \[
        \mu(y^\prime)>0\ \Longrightarrow\ \mu(y^\prime)\ \ge\ \frac{1}{d^2} .
    \]
    For M$^\natural$-convex sets, the same lower bounds hold with $\mathcal{T}=\{e_i-e_j\}\cup\{\pm e_i\}$ (that $\mathcal{T}$ is a test set follows from the optimality criterion for linear optimization over M$^\natural$-convex sets; \citealp{murota2003discrete}).
\end{proposition}

\begin{proof}
    (a) is \Cref{prop:mconvex_optimality} applied to $\mathcal{M}=\mathcal{X}(s^1)$.
    We show (b). As in the proof of \Cref{prop:delta_explicit_lb}, an extreme point of the LP of $m^{\mathcal{T}}_{G}(y^\prime)$ is determined by a nonsingular integer matrix $A$ consisting of $(d+1)$ linearly independent tight constraints, and each row of $A$ is of type $(e_i-e_j,\,-1)$, $(\mathbf{0},\,-1)$, $(\mathbf{1},\,0)$, or $(-e_k,\,0)$. Since the equality constraint $\sum_i\theta_i=1$ is tight at every feasible point, we can choose these $d+1$ tight constraints so that they include the row $(\mathbf{1},\,0)$ (extend $\{(\mathbf{1},0)\}$ to a basis of the row space of the tight constraints). By cofactor expansion along the $\tau$ column, $|\det A|$ is bounded by the sum, over ``the rows having a nonzero entry in the $\tau$ column (at most $d$ rows, since the $\tau$ entry of $(\mathbf{1},0)$ is $0$)'', of the absolute values of the determinants of the $d\times d$ submatrices $M^\prime$ formed by the $\theta$ parts of the remaining $d$ rows. Each $M^\prime$ contains exactly one all-ones row, and its other rows are of type $e_i-e_j$, of type $-e_k$, or zero rows. If it contains a zero row, $\det M^\prime=0$. Otherwise, cofactor expansion along the all-ones row bounds $|\det M^\prime|$ by the sum of $d$ determinants of $(d-1)\times(d-1)$ matrices in which each row has at most one $+1$ and at most one $-1$; their transposes are submatrices of the incidence matrix of a directed graph and are totally unimodular, so $|\det M^\prime|\le d$. Hence $1\le|\det A|\le d^2$, and $m^{\mathcal{T}}_{G}(y^\prime)$ is a rational number with denominator at most $d^2$. The difference of two distinct values is at least $1/d^4$, and positive values are at least $1/d^2$. Adding rows of type $\pm e_i$ does not change the structure ``each row has at most one $+1$ and at most one $-1$'', so the same bound holds in the M$^\natural$-convex case.
\end{proof}

\begin{remark}[Branch minima in the M-convex case]
    \label{rem:mconvex_degenerate}
    Let $\Theta=\Delta^{d-1}$ and let $\mathcal{X}(s^1)$ be an M-convex set. Then $m_{\sub}(y^\prime)=\mu(y^\prime)=0$ for every branch $y^\prime$. Indeed, by the exchange axiom, an M-convex set lies on a hyperplane $\{x \mid \sum_i x_i=\text{const}\}$, so for the uniform weight $\bar{\theta}:=(1/d,\ldots,1/d)\in\Delta^{d-1}$, we have $\langle\bar{\theta},v-y^\prime\rangle=0$ for every $v\in\mathcal{X}(s^1)$ and $\langle\bar{\theta},e_i-e_j\rangle=0$ for every $i\ne j$. Hence $\bar{\ell}^{\sub}(\bar{\theta};y^\prime)=h_{\mathcal{T}}(\bar{\theta})=0$, and together with nonnegativity, $m_{\sub}(y^\prime)=\mu(y^\prime)=0$. Therefore, in the M-convex case, no branch satisfies the hypothesis $\mu(y^\prime)>0$ of the latter half of \Cref{prop:testcut_mconvex}(b), and \Cref{cor:testcut_zero_decision} decides that every branch is a zero branch. The latter half of \Cref{prop:testcut_mconvex}(b) and \Cref{cor:testcut_zero_decision} are used in the case of M$^\natural$-convex sets that are not M-convex to identify branches with a positive branch minimum.
\end{remark}

\begin{proposition}[Exact minimization of the OS loss in the M$^\natural$-convex case]
    \label{prop:mnatural_uniform_exact}
    Let $\Theta=\Delta^{d-1}$, and for each $n$, let $\mathcal{X}(s^n)\subset\mathbb{Z}^d$ be a bounded M$^\natural$-convex set (including M-convex sets) and $x^n\in\mathcal{X}(s^n)$. Then the uniform weight $\bar{\theta}:=(1/d,\ldots,1/d)\in\Delta^{d-1}$ satisfies $\ell^{\optdex}(\bar{\theta};x^n,s^n)=0$ for every $n$. Hence $\bar{\theta}$ is a global minimizer of $\bar{\ell}^{\optdex}$ over $\Theta$.
\end{proposition}

\begin{proof}
    Fix $n$ and let $\mathcal{X}:=\mathcal{X}(s^n)$. The set $\mathcal{T}=\{e_i-e_j\mid i\neq j\}\cup\{\pm e_i\}$ is a test set for linear optimization over M$^\natural$-convex sets (\Cref{prop:testcut_mconvex}; \citealp{murota2003discrete}). Since $\langle\bar{\theta},e_i-e_j\rangle=0$ and $\langle\bar{\theta},-e_i\rangle<0$, if $x\in\mathcal{X}$ is not a maximizer of $\langle\bar{\theta},\cdot\rangle$, then $x+e_i\in\mathcal{X}$ for some $i$. Let $z^0:=x^n$, and as long as $z^k$ is not a maximizer, update $z^{k+1}:=z^k+e_i\in\mathcal{X}$. Since $\langle\bar{\theta},z^k\rangle$ strictly increases and $\mathcal{X}$ is finite, this update reaches a maximizer $z\in\argmax_{x\in\mathcal{X}}\langle\bar{\theta},x\rangle$ after finitely many steps. Each update only increases a component, so $z\ge x^n$, that is, $z\in\mathcal{X}_{\geq x^n}(s^n)$. Hence $\max_{x\in\mathcal{X}_{\geq x^n}(s^n)}\langle\bar{\theta},x\rangle=\max_{x\in\mathcal{X}}\langle\bar{\theta},x\rangle$, and $\ell^{\optdex}(\bar{\theta};x^n,s^n)=0$ by \cref{eq:def_optdex}. Since $\bar{\ell}^{\optdex}\ge0$, $\bar{\theta}$ is a global minimizer.
\end{proof}

\Cref{prop:mnatural_uniform_exact} gives a global minimizer of the OS loss. Attaining \cref{eq:outperformance_IOP} additionally requires that the optimal solution be unique or that the oracle return an optimal solution dominating the action (\Cref{sec:performance_metric}).

\begin{proposition}[Number of iterations of one call of the test-direction-cut inner solver]
    \label{prop:testcut_per_call}
    In the setting of \Cref{prop:testset_value_gap} or \Cref{prop:testcut_mconvex}, run the test-direction-cut inner solver (\Cref{def:testcut_variant}) for a branch $y^\prime$ with tolerance $0$ and $J=\infty$, and assume that the LPs are solved exactly.

    (i) If the rule for selecting evaluation points is the cutting-plane method, the solver stops within at most $|\mathcal{T}_{y^\prime}|+2$ iterations, and $\mathrm{LB}=\mathrm{UB}=\mu(y^\prime)$ at stopping; the returned point is an exact minimizer over $\Delta^{d-1}$ of the test-set model $h_{\mathcal{T}}$ (\cref{eq:def_testset_model}), with minimum value $\mu(y^\prime)$. In particular, in the M$^\natural$-convex case, $|\mathcal{T}_{y^\prime}|\le d(d+1)$, so the number is at most $d(d+1)+2$.

    (ii) In the setting of \Cref{prop:testset_value_gap}, let the rule be cutting-plane finishing (\Cref{alg:rule_cpfinish}) with the level method as the rule $\mathsf{R}$, $\epsilon_{\mathrm{sw}}:=1/(2Q_{\mathcal{T}}^2)$, and $\epsilon_{\mathrm{term}}=0$. If the iteration bound of \Cref{rem:level_complexity} holds for the level method in Phase 1, then the number of iterations of Phase 1 is at most $c(\lambda)\bigl(C_R\sqrt{d}\cdot\sqrt{2}\cdot2Q_{\mathcal{T}}^2\bigr)^2$, that of Phase 2 is at most $|\mathcal{T}_{y^\prime}|+1$, and the returned point is an exact minimizer of $h_{\mathcal{T}}$ over $\Delta^{d-1}$, with minimum value $\mu(y^\prime)$.
\end{proposition}

\begin{proof}
    Let $\mathcal{G}_j$ be the cut set at the end of iteration $j$; the zero cut $\mathbf{0}$ always belongs to $\mathcal{G}_j$. Let $g_j$ be the maximizing cut returned by the test-direction oracle at the evaluation point $\theta^j$; then $h_{\mathcal{T}}(\theta^j)=\max\{\langle\theta^j,g_j\rangle,0\}$.

    We first show the following: if the evaluation point $\theta^j$ of iteration $j\ge1$ is a minimizer over $\Delta^{d-1}$ of the model $\max_{g\in\mathcal{G}_{j-1}}\langle\theta,g\rangle$ and $g_j\in\mathcal{G}_{j-1}$, then the solver stops at iteration $j$. Indeed, then $\mathcal{G}_j=\mathcal{G}_{j-1}$, so $\mathrm{LB}=\min_{\theta\in\Delta^{d-1}}\max_{g\in\mathcal{G}_{j-1}}\langle\theta,g\rangle=\max_{g\in\mathcal{G}_{j-1}}\langle\theta^j,g\rangle\ge h_{\mathcal{T}}(\theta^j)\ge\mathrm{UB}$, and the stopping test $\mathrm{UB}-\mathrm{LB}\le0$ holds.

    (i) With the cutting-plane method, the evaluation point of each iteration $j\ge1$ is a minimizer of the model, so by the claim above, if the solver does not stop at iteration $j\ge1$, then $g_j\in\mathcal{T}_{y^\prime}\setminus\mathcal{G}_{j-1}$, and a new element of $\mathcal{T}_{y^\prime}$ is added to $\mathcal{G}$. Since $\mathcal{T}_{y^\prime}$ is finite, there are at most $|\mathcal{T}_{y^\prime}|$ non-stopping iterations $j\ge1$. Adding iteration $j=0$ and the stopping iteration, the number of iterations is at most $|\mathcal{T}_{y^\prime}|+2$. Since the model is at most $h_{\mathcal{T}}$, $\mathrm{LB}\le\mu(y^\prime)$, and since $\mathrm{UB}$ is the minimum of the values of $h_{\mathcal{T}}$ at the evaluation points, $\mu(y^\prime)\le\mathrm{UB}$. At stopping, $\mathrm{UB}\le\mathrm{LB}$, so $\mathrm{LB}=\mathrm{UB}=\mu(y^\prime)$, and the returned point $\hat{\theta}$ satisfies $h_{\mathcal{T}}(\hat{\theta})=\mathrm{UB}=\mu(y^\prime)$. In the M$^\natural$-convex case, $\mathcal{T}_{y^\prime}\subseteq\{e_i-e_j\mid i\ne j\}\cup\{\pm e_i\}$, whose cardinality is $d(d-1)+2d=d(d+1)$.

    (ii) Each element $g$ of $\mathcal{T}_{y^\prime}\cup\{\mathbf{0}\}$ satisfies $\|g\|_2\le {C_R}\sqrt{d}$, so $h_{\mathcal{T}}$ is $C_R\sqrt{d}$-Lipschitz. Since $\diam(\Delta^{d-1})=\sqrt{2}$, the assumed bound of \Cref{rem:level_complexity} implies that the number of iterations of Phase 1 is at most $c(\lambda)\bigl(C_R\sqrt{d}\cdot\sqrt{2}/\epsilon_{\mathrm{sw}}\bigr)^2=c(\lambda)\bigl(C_R\sqrt{d}\cdot\sqrt{2}\cdot2Q_{\mathcal{T}}^2\bigr)^2$. By \Cref{prop:testcut_gap}, $\epsilon_{\mathrm{sw}}<1/Q_{\mathcal{T}}^2\le\delta_{\mathcal{T}}(y^\prime)$; hence, applying the proof of \Cref{lem:lb_snap} with $\epsilon_{\mathrm{in}}$, $\delta(y^\prime)$, and $m_{\sub}(y^\prime)$ replaced by $\epsilon_{\mathrm{sw}}$, $\delta_{\mathcal{T}}(y^\prime)$, and $\mu(y^\prime)$, we have $\mathrm{LB}=\mu(y^\prime)$ at the end of Phase 1. The evaluation point of each iteration of Phase 2 is a minimizer of the model, so by the claim above, each non-stopping iteration of Phase 2 adds a new element of $\mathcal{T}_{y^\prime}$ to $\mathcal{G}$. Hence there are at most $|\mathcal{T}_{y^\prime}|$ non-stopping iterations in Phase 2, and adding the stopping iteration, at most $|\mathcal{T}_{y^\prime}|+1$. At stopping, $\mathrm{UB}=\mu(y^\prime)$ as in (i), so the returned point is an exact minimizer of $h_{\mathcal{T}}$.
\end{proof}

\begin{corollary}[Connection to the decision of zero branches]
    \label{cor:testcut_zero_decision}
    In the setting of \Cref{prop:testset_value_gap} or \Cref{prop:testcut_mconvex}, we have
    \[
        \mu(y^\prime)=0
        \quad\Longleftrightarrow\quad
        m_{\sub}(y^\prime)=0 .
    \]
    Therefore, the test-direction-cut inner solver decides exactly whether a branch is a zero branch through the exact computation of $\mu(y^\prime)$. In particular, in the M$^\natural$-convex case, the size of the family of cuts is at most $d(d+1)+1$, so the decision is completed with at most $d(d+1)+2$ calls to the test-direction oracle using only the selection rule of the cutting-plane method (\Cref{prop:testcut_per_call}(i)). Moreover, when $\mu(y^\prime)>0$, $m_{\sub}(y^\prime)\ge\mu(y^\prime)$, and in the Graver-basis case, $m_{\sub}(y^\prime)\le D_z\,\mu(y^\prime)$ (\Cref{lem:testset_sandwich}) also holds.
\end{corollary}

\begin{proof}
    Suppose $\mu(y^\prime)=0$. Then $h_{\mathcal{T}}(\bar{\theta})=0$ for some $\bar{\theta}\in\Delta^{d-1}$, that is, no feasible test direction from $y^\prime$ has $\langle\bar{\theta},\cdot\rangle>0$. In the Graver-basis case, by the upper side of \Cref{lem:testset_sandwich}, $\bar{\ell}^{\sub}(\bar{\theta};y^\prime)\le D_z\,h_{\mathcal{T}}(\bar{\theta})=0$; in the M-convex case, by \Cref{prop:testcut_mconvex}(a), and in the M$^\natural$-convex case, by the test-set property for M$^\natural$-convex sets stated at the end of that proposition, $y^\prime$ is a maximizer with respect to $\bar{\theta}$, so $\bar{\ell}^{\sub}(\bar{\theta};y^\prime)=0$. Hence $m_{\sub}(y^\prime)=0$ (together with nonnegativity). Conversely, suppose $m_{\sub}(y^\prime)=0$. Then $\bar{\ell}^{\sub}(\theta^\ast;y^\prime)=0$ for some $\theta^\ast$, and for every feasible test direction $t$, $y^\prime+t^x\in\mathcal{X}(s^1)$, so $\langle\theta^\ast,t^x\rangle\le\bar{\ell}^{\sub}(\theta^\ast;y^\prime)=0$; hence $h_{\mathcal{T}}(\theta^\ast)=0$ and $\mu(y^\prime)=0$. The inequality $m_{\sub}(y^\prime)\ge\mu(y^\prime)$ follows from $\bar{\ell}^{\sub}(\cdot\,;y^\prime)\ge h_{\mathcal{T}}$ (the lower side of \Cref{lem:testset_sandwich}; it uses only the feasibility of single steps, so it also holds in the M-convex case). The bound on the size of the family of cuts and the finite stopping follow from \Cref{prop:testcut_per_call}(i).
\end{proof}

\begin{remark}[Scope]
    \label{rem:testcut_scope}
    First, what \Cref{prop:testcut_gap,prop:testcut_mconvex} improve is $\delta_{\mathcal{T}}(y^\prime)$ of the test-direction-cut inner solver, not $\delta(y^\prime)$ of \Cref{alg:dc_inner} itself, which uses vertex cuts. The denominators of $\delta(y^\prime)$ are determined by the cuts $v-y^\prime$ actually generated (with components of magnitude up to $2R$), and the tie values of two cuts with slopes of order $R$ can approach each other to within order $1/R^2$; we therefore conjecture that $\delta(y^\prime)$ itself admits no lower bound independent of $R$. Its proof is future work.
    Second, what the test-direction-cut inner solver minimizes is $h_{\mathcal{T}}$, not $\bar{\ell}^{\sub}(\cdot\,;y^\prime)$, so it cannot be used directly for the premise of \Cref{prop:dc_finite_convergence} (exact solution of the convex subproblems). Its uses are the exact decision of zero branches (\Cref{cor:testcut_zero_decision}; for $N\ge2$, \Cref{cor:zero_branch_cert_N,cor:testcut_per_call_N}) and the $D_z$-approximation of $m_{\sub}(y^\prime)$ in the Graver-basis case.
\end{remark}

\begin{remark}[What can be said in the M-convex and Graver-basis cases]
    \label{rem:total_structured}
    We discuss whether \Cref{prop:total_inner_iterations} can be specialized to the structured cases, in contrast with \Cref{cor:total_integer}. The conclusion is that the lower bounds of this subsection give a bound on the number of iterations of one call (per branch) of the test-direction-cut inner solver (\Cref{prop:testcut_per_call}).

    The reason why the lower bounds of this subsection are for the test-direction-cut inner solver is as described in \Cref{rem:testcut_scope}. \Cref{prop:total_inner_iterations} requires a lower bound on $\delta_{\mathrm{in}}=\min_{y^\prime}\delta(y^\prime)$, whereas \Cref{prop:testcut_gap,prop:testcut_mconvex} give lower bounds on $\delta_{\mathcal{T}}(y^\prime)$ of the test-direction-cut inner solver. The former is determined by the vertex cuts $v-y^\prime$ (with components of magnitude up to $2R$), and the latter by the test directions $t^x$ (with components of magnitude up to $C_R$, and $1$ in the M$^\natural$-convex case). If the conjecture of \Cref{rem:testcut_scope} is true, $\delta(y^\prime)$ admits no lower bound independent of $R$, and the dependence on $R$ in \Cref{cor:total_integer} is essential.

    The number of iterations of one call (per branch) of the test-direction-cut inner solver is given in \Cref{prop:testcut_per_call}. In the Graver-basis case, the number of iterations of Phase 1 is bounded by $c(\lambda)\bigl(C_R\sqrt{d}\cdot\sqrt{2}\cdot2Q_{\mathcal{T}}^2\bigr)^2$ (part (ii) of that proposition). Since $Q_{\mathcal{T}}\le\bigl((d+1)^{1/2}\max(C,1)\bigr)^{d+1}$, this number has an upper bound independent of $R$. It is exponential in $d$.
    The M$^\natural$-convex case is even stronger: the size of the family of cuts is at most $d(d+1)+1$, so the solver stops with at most $d(d+1)+2$ calls to the test-direction oracle using only the selection rule of the cutting-plane method, without going through the level method (\Cref{prop:testcut_per_call}(i)). That is, $O(d^2)$ calls per branch, polynomial in the dimension and independent of both $R$ and $\delta$.

    Therefore, the question of whether the number over the whole run can be made polynomial in the dimension reduces to whether the exact computation of $\mu(y^\prime)$ in \Cref{cor:testcut_zero_decision} can be lifted to the exact computation of $m_{\sub}(y^\prime)$ (the premise of \Cref{prop:dc_finite_convergence}). In the M-convex case, the uniform weight $\bar{\theta}$ gives a minimizer of $\bar{\ell}^{\sub}(\cdot\,;y^\prime)$ (with minimum value $0$) for every branch (\Cref{rem:mconvex_degenerate}), so $m_{\sub}(y^\prime)$ is computed directly. In the M$^\natural$-convex case as well, the global minimization of the OS loss itself is attained by the uniform weight (\Cref{prop:mnatural_uniform_exact}). What remains is, as a subproblem of DC optimization, recovering the minimization of $\bar{\ell}^{\sub}(\cdot\,;y^\prime)$ at each branch from the minimization of $h_{\mathcal{T}}$ in the M$^\natural$-convex case, which is, as in \Cref{rem:structural_improvement}, future work.
\end{remark}

\subsection{Explicit lower bounds via test sets (multiple states)}
\label{subsec:testset_gap_N}

In this subsection, we extend the results of \Cref{subsec:testset_gap} to $N\ge1$. Let $\Theta=\Delta^{d-1}$, and suppose that the constraint matrix $A\in\mathbb{Z}^{n_A\times d}$ is common to all states, while the right-hand side and the upper bounds of the box may depend on the state:
\[
    \mathcal{X}(s^n)=\bigl\{x\in\mathbb{Z}^d \ \bigm|\ Ax\le b^n,\ 0\le x\le u^n\bigr\}
    \qquad(b^n\in\mathbb{Z}^{n_A},\ u^n\in\mathbb{Z}_{\ge0}^d,\ n=1,\ldots,N).
\]
Setting $z_n(x):=(x,\,b^n-Ax)\in\mathbb{Z}^{d+n_A}$, $x\in\mathcal{X}(s^n)$ is equivalent to $\tilde{A}z_n(x)=b^n$ and $0\le z_n(x)\le\tilde{u}^n$ ($\tilde{u}^n$ is the finite upper bound determined by the box $0\le x\le u^n$). The Graver basis $\mathcal{T}=\mathrm{Gr}(\tilde{A})$ and $C=\max_{t\in\mathcal{T}}\|t\|_\infty$ are determined by $\tilde{A}$ alone and are therefore common to all states. Setting $R:=\max_n\max_{x\in\mathcal{X}(s^n)}\|x\|_\infty$ and $C_R:=\min(C,R)$, every test direction $t$ feasible at a state satisfies $\|t^x\|_\infty\le C_R$. For a branch $y^\prime=(y^{\prime n})_{n=1}^N$ ($y^{\prime n}\in\mathcal{X}(s^n)$), let
\[
    \mathcal{T}_{y^\prime,n}:=\{t^x \mid t\in\mathcal{T},\ 0\le z_n(y^{\prime n})+t\le\tilde{u}^n\},
    \qquad
    h_{\mathcal{T},n}(\theta):=\max\Bigl(0,\ \max_{g\in\mathcal{T}_{y^\prime,n}}\langle\theta,g\rangle\Bigr),
\]
\[
    h^N_{\mathcal{T}}(\theta):=\frac{1}{N}\sum_{n=1}^N h_{\mathcal{T},n}(\theta),
    \qquad
    \mu_N(y^\prime):=\min_{\theta\in\Delta^{d-1}}h^N_{\mathcal{T}}(\theta)
\]
(since $t\in\ker_{\mathbb{Z}}\tilde{A}$, whether $z_n(y^{\prime n})+t$ is feasible at state $n$ is determined by the box constraints). We also set $D_{z,n}:=\max\bigl(1,\ \max_{x,x^\prime\in\mathcal{X}(s^n)}\|z_n(x^\prime)-z_n(x)\|_1\bigr)$ and $D^N_z:=\max_n D_{z,n}$.

\begin{lemma}[Sandwich by the test-set model with multiple states]
    \label{lem:testset_sandwich_N}
    In the above setting, for any $\theta\in\mathbb{R}^d$,
    \[
        h^N_{\mathcal{T}}(\theta)\ \le\ \bar{\ell}^{\sub}(\theta;y^\prime)\ \le\ D^N_z\,h^N_{\mathcal{T}}(\theta) .
    \]
\end{lemma}

\begin{proof}
    We have $\bar{\ell}^{\sub}(\theta;y^\prime)=\frac{1}{N}\sum_{n=1}^N\ell^{\sub}(\theta;y^{\prime n},s^n)$. The proof of \Cref{lem:testset_sandwich} uses only the feasible region $\{x\in\mathbb{Z}^d\mid Ax\le b,\ 0\le x\le u\}$ of a single state and $y^\prime\in\mathcal{X}(s^1)$; applying it to each $n$ with $(b,u)$ replaced by $(b^n,u^n)$ gives $h_{\mathcal{T},n}(\theta)\le\ell^{\sub}(\theta;y^{\prime n},s^n)\le D_{z,n}\,h_{\mathcal{T},n}(\theta)$. Averaging over $n$ with $D_{z,n}\le D^N_z$ and $h_{\mathcal{T},n}\ge0$ yields the claim.
\end{proof}

\begin{proposition}[Graver lower bound on the branch value gap with multiple states]
    \label{prop:testset_value_gap_N}
    In the above setting, let
    \[
        Q_{\mathcal{T},N}:=\Bigl(1+\frac{d}{N}\Bigr)^N\bigl(d^{1/2}\max(C_R,1)\bigr)^d
        \ \le\ \bigl(\mathrm{e}\,d^{1/2}\max(C,1)\bigr)^d .
    \]
    Then
    \[
        N\mu_N(y^\prime)\in\{0\}\cup\Bigl[\frac{1}{Q_{\mathcal{T},N}},\ \infty\Bigr) ,
        \qquad
        m_{\sub}(y^\prime)\in\{0\}\cup\Bigl[\frac{1}{N\,Q_{\mathcal{T},N}},\ \infty\Bigr) .
    \]
    Moreover, $\mu_N(y^\prime)=0$ if and only if $m_{\sub}(y^\prime)=0$. In particular, setting $\bar{Q}_{\mathcal{T}}:=(\mathrm{e}\,d^{1/2}\max(C,1))^d$, if $m_{\sub}(y^\prime)>0$, then $m_{\sub}(y^\prime)\ge1/(N\,\bar{Q}_{\mathcal{T}})$; this lower bound depends neither on the right-hand sides $b^n$ nor on the upper bounds $u^n$ of the boxes, and depends on $N$ only through the factor $1/N$.
\end{proposition}

\begin{proof}
    $N\mu_N(y^\prime)$ is the optimal value of the linear program in the variables $(\theta,\tau)\in\mathbb{R}^d\times\mathbb{R}^N$
    \[
        \min\Bigl\{\textstyle\sum_{n=1}^N\tau_n \ \Bigm|\ \theta\ge0,\ \sum_i\theta_i=1,\ \tau_n\ge0,\ \tau_n\ge\langle g,\theta\rangle\ (g\in\mathcal{T}_{y^\prime,n},\ n=1,\ldots,N)\Bigr\} .
    \]
    The feasible region is nonempty and contains no line ($\theta$ is bounded and $\tau\ge0$), and the objective is nonnegative, so the optimal value is attained at an extreme point. An extreme point is determined by a nonsingular integer matrix $A^\prime$ consisting of $d+N$ linearly independent tight constraints and an integer right-hand side, so by Cramer's rule each component of the extreme point, and hence the objective value $\sum_n\tau_n$, is a rational number with denominator $|\det A^\prime|$.

    We bound $|\det A^\prime|$. Each row of $A^\prime$ is of type $(g,-e_n)$ ($g\in\mathcal{T}_{y^\prime,n}$), $(\mathbf{0},-e_n)$, $(\mathbf{1},\mathbf{0})$, or $(-e_i,\mathbf{0})$, and in the $N$ columns corresponding to $\tau$, each row has at most one nonzero entry ($-1$). Let $k_n$ be the number of rows with a nonzero entry in the $\tau_n$ column; since $A^\prime$ is nonsingular, $k_n\ge1$, and $\sum_nk_n\le d+N$. In the Laplace expansion along the $N$ columns of $\tau$, the choices of $N$ rows for which the $\tau$ part is nonsingular are exactly those choosing one row from each $\tau_n$ column, and the absolute value of the determinant of that part is $1$. There are $\prod_nk_n$ such choices, and by the inequality of arithmetic and geometric means, $\prod_nk_n\le\bigl((d+N)/N\bigr)^N=(1+d/N)^N$. The $\theta$ part of the remaining $d$ rows has rows of norm at most $d^{1/2}\max(C_R,1)$, since $\|g\|_2\le d^{1/2}C_R$, $\|\mathbf{1}\|_2=d^{1/2}$, and $\|e_i\|_2=1$; hence by Hadamard's inequality, the absolute value of its determinant is at most $(d^{1/2}\max(C_R,1))^d$. Therefore $1\le|\det A^\prime|\le Q_{\mathcal{T},N}$, and $N\mu_N(y^\prime)$ is a nonnegative rational number with denominator at most $Q_{\mathcal{T},N}$. Hence, if $N\mu_N(y^\prime)>0$, then $N\mu_N(y^\prime)\ge1/Q_{\mathcal{T},N}$. Also, $(1+d/N)^N\le\mathrm{e}^d$.

    By \Cref{lem:testset_sandwich_N}, $\mu_N(y^\prime)\le m_{\sub}(y^\prime)\le D^N_z\,\mu_N(y^\prime)$. Hence $\mu_N(y^\prime)=0$ if and only if $m_{\sub}(y^\prime)=0$, and if $m_{\sub}(y^\prime)>0$, then $m_{\sub}(y^\prime)\ge\mu_N(y^\prime)\ge1/(N\,Q_{\mathcal{T},N})$.
\end{proof}

\begin{corollary}[Decision of zero branches and realizability with multiple states]
    \label{cor:zero_branch_cert_N}
    In the above setting, if $\bar{\ell}^{\sub}(\bar{\theta};y^\prime)<1/(N\,Q_{\mathcal{T},N})$ for some $\bar{\theta}\in\Delta^{d-1}$ and a branch $y^\prime\in\mathcal{X}^{\mathrm{all}}_{\geq}$, then $m_{\sub}(y^\prime)=0$. In particular, if $\bar{\ell}^{\optdex}(\bar{\theta})<1/(N\,Q_{\mathcal{T},N})$ for some $\bar{\theta}\in\Delta^{d-1}$, then $\min_{\theta\in\Theta}\bar{\ell}^{\optdex}(\theta)=0$.
\end{corollary}

\begin{proof}
    The first half follows from \Cref{prop:testset_value_gap_N}. For the second half, $\bar{\ell}^{\optdex}(\bar{\theta})=\min_{y^\prime\in\mathcal{X}^{\mathrm{all}}_{\geq}}\bar{\ell}^{\sub}(\bar{\theta};y^\prime)$ by \cref{eq:dc_pointwise_min}; applying the first half to a branch $y^\prime$ attaining this minimum gives $m_{\sub}(y^\prime)=0$, and by \cref{eq:dc_min_min}, $\min_{\theta\in\Theta}\bar{\ell}^{\optdex}(\theta)\le m_{\sub}(y^\prime)=0$. Equality holds since $\bar{\ell}^{\optdex}\ge0$.
\end{proof}

The test-direction-cut inner solver (\Cref{def:testcut_variant}) also extends to $N\ge1$. At each evaluation point $\theta$, let the test-direction oracle of each state return a maximizing cut $g_n$ of $\max_{g\in\mathcal{T}_{y^\prime,n}\cup\{\mathbf{0}\}}\langle\theta,g\rangle$, and use $\frac{1}{N}\sum_ng_n$, which we call the aggregated cut, as the cut. For the family of aggregated cuts
\[
    \mathcal{F}_{y^\prime}:=\Bigl\{\frac{1}{N}\sum_{n=1}^Ng_n \ \Bigm|\ g_n\in\mathcal{T}_{y^\prime,n}\cup\{\mathbf{0}\}\Bigr\},
\]
we have $h^N_{\mathcal{T}}(\theta)=\max_{f\in\mathcal{F}_{y^\prime}}\langle\theta,f\rangle$, and the aggregated cut is a maximizer of this maximum.

\begin{corollary}[Number of iterations of one call of the test-direction-cut inner solver with multiple states]
    \label{cor:testcut_per_call_N}
    In the above setting, run the test-direction-cut inner solver with aggregated cuts, with the zero cut $\mathbf{0}$ as the initial cut, the cutting-plane method as the rule for selecting evaluation points, tolerance $0$, and $J=\infty$, and assume that the LPs are solved exactly. Then the solver stops within at most $|\mathcal{F}_{y^\prime}|+1\le\prod_{n=1}^N\bigl(|\mathcal{T}_{y^\prime,n}|+1\bigr)+1$ iterations, and the returned point is an exact minimizer over $\Delta^{d-1}$ of $h^N_{\mathcal{T}}$, with minimum value $\mu_N(y^\prime)$. If each $\mathcal{X}(s^n)$ is an M$^\natural$-convex set, then with the test set $\{e_i-e_j\mid i\ne j\}\cup\{\pm e_i\}$, $|\mathcal{T}_{y^\prime,n}|\le d(d+1)$, so the number is at most $(d(d+1)+1)^N+1$.
\end{corollary}

\begin{proof}
    The proof of \Cref{prop:testcut_per_call}(i) uses only that the family of cuts is finite, that the oracle returns a maximizing cut of the model function at the evaluation point, and that the zero cut always belongs to the cut set. Replacing the family $\mathcal{T}_{y^\prime}\cup\{\mathbf{0}\}$ by $\mathcal{F}_{y^\prime}$ and $h_{\mathcal{T}}$ by $h^N_{\mathcal{T}}$, the same argument holds. Since the zero cut is an element of $\mathcal{F}_{y^\prime}$ and belongs to the cut set from the beginning, each non-stopping iteration $j\ge1$ adds a new element of $\mathcal{F}_{y^\prime}\setminus\{\mathbf{0}\}$, so there are at most $|\mathcal{F}_{y^\prime}|-1$ of them. Adding iteration $j=0$ and the stopping iteration, the number of iterations is at most $|\mathcal{F}_{y^\prime}|+1$. Moreover, $|\mathcal{F}_{y^\prime}|\le\prod_n(|\mathcal{T}_{y^\prime,n}|+1)$, and the bound in the M$^\natural$-convex case follows as in \Cref{prop:testcut_per_call}(i).
\end{proof}

{

\section{Proofs for the generalization analysis}
\label{sec:appendix_generalization}

In this section, we prove the results of \Cref{sec:generalization}.
We follow the notation of \Cref{sec:generalization}.
Following the standard treatment in empirical process theory, we assume that the events appearing below (such as unions over all weights consistent with the observations) are measurable.

$\ell^{01}(\theta;x^\prime,s)=0$ holds if and only if the OS loss $\ell^{\optdex}(\theta;x^\prime,s)$ is $0$ and the optimal solution is unique (\Cref{lem:gen_e_vs_os}).
In this case, the oracle returns the unique optimal solution, which dominates $x^\prime$; hence $\mathrm{GE}(\theta)$ is an upper bound on the probability that, at a new state, the optimal solution returned by the oracle does not dominate the agent's action, regardless of which optimal solution the oracle returns when there are several.
We do not use the OS loss itself because, for a weight whose optimal solution is not unique, the oracle may return an optimal solution that does not dominate the agent's action even when the OS loss is $0$ (\Cref{sec:performance_metric}).
Under \Cref{assu:gen}(2), $\theta^*$ is a.s.\ consistent with the observations.

In this section, we use the following notation. For a set $Z\subseteq\mathbb{R}^d$, we write $\Conv Z$ for its convex hull, and for a convex set $C$, we write $\ext C$ for the set of its extreme points (\Cref{subsec:dc_notation}). For a bounded polyhedron $P$, $\ext P$ is the set of vertices of $P$. A point $u$ of a convex set $C$ is called an \emph{exposed point} of $C$ if, for some $\theta\in\mathbb{R}^d$, $u$ is the unique maximizer of $\langle\theta,\cdot\rangle$ over $C$. We write $\inte Z$ for the interior of a set $Z$. A \emph{face} of a bounded polyhedron $P$ is a subset that can be written as $\argmax_{y\in P}\langle\theta,y\rangle$ for some $\theta\in\mathbb{R}^d$. A finite family $\mathcal{A}$ of distinct affine hyperplanes $H=\{x\mid g_H(x)=0\}$ in $\mathbb{R}^k$ (where $g_H$ is a nonconstant affine function) is called a \emph{hyperplane arrangement}. For a sign pattern $\tau\in\{-,+\}^{\mathcal{A}}$, if $\{x\in\mathbb{R}^k\mid\operatorname{sign}g_H(x)=\tau_H\ (\forall H\in\mathcal{A})\}$ is nonempty, it is called a \emph{region} of $\mathcal{A}$. The regions form a partition of $\mathbb{R}^k\setminus\bigcup_{H\in\mathcal{A}}H$, and their number does not depend on the choice of signs of the $g_H$.

\subsection{Preliminaries}
\label{subsec:appendix_gen_preliminaries}

We supplement the setting of \Cref{sec:generalization}. We assume that $x\in\mathcal{X}(s)$ holds almost surely (a.s.) for $(s,x)\sim\mathcal{P}$. The $\operatorname{ess\,sup}$ in \Cref{assu:gen}(1) is the essential supremum with respect to the marginal distribution of $\mathcal{P}$ in $s$. When $\mathcal{S}$ is finite, $\kappa<\infty$ in \Cref{assu:gen}(1) holds automatically, since each $\mathcal{Y}(s)$ is finite (\Cref{rem:assu_margin}).

\begin{lemma}
    \label{lem:gen_extreme}
    Assume \Cref{assu:margin}(2)(3), and fix $s\in\mathcal{S}$ and $\theta\in\mathbb{R}^d$.
    \begin{description}
        \item[(i)] $\ext\Conv\mathcal{X}(s)\subseteq\mathcal{Y}(s)\subseteq\mathcal{X}(s)$, $\mathcal{Y}(s)$ is a finite set, and we have $\max_{x\in\mathcal{X}(s)}\langle\theta,x\rangle=\max_{u\in\mathcal{Y}(s)}\langle\theta,u\rangle$.
        \item[(ii)] $\argmax_{y\in\mathcal{X}(s)}\langle\theta,y\rangle$ is a singleton if and only if there is exactly one point $u$ in $\mathcal{Y}(s)$ that maximizes $\langle\theta,\cdot\rangle$, and in this case $\argmax_{y\in\mathcal{X}(s)}\langle\theta,y\rangle=\{u\}$.
        In particular, if $\argmax_{y\in\mathcal{X}(s)}\langle\theta,y\rangle=\{v\}$, then $v\in\mathcal{Y}(s)$.
    \end{description}
\end{lemma}

\begin{proof}
    (i) By \Cref{assu:margin}(2), we can write $\mathcal{X}(s)=\bigcup_{j}Q_j$ (finitely many nonempty bounded polyhedra $Q_j$). $\Conv\mathcal{X}(s)=\Conv\bigl(\bigcup_j\ext {Q_j}\bigr)$ is a bounded polyhedron, and its extreme points are contained in the finite set $\bigcup_j\ext {Q_j}\subseteq\mathcal{X}(s)$.
    Let $u\in\ext\Conv\mathcal{X}(s)$. Since the extreme points of a polyhedron are exposed points, there exists $\theta^\prime$ such that $u$ is the unique maximizer over $\Conv\mathcal{X}(s)$, and $\argmax_{y\in\mathcal{X}(s)}\langle\theta^\prime,y\rangle=\{u\}$. Since the oracle returns an optimal solution, $x^*(\theta^\prime,s)=u$, and thus $u\in\mathcal{Y}(s)$.
    $\mathcal{Y}(s)\subseteq\mathcal{X}(s)$ and the finiteness of $\mathcal{Y}(s)$ are as stated in the remark after \Cref{assu:margin}(3).
    Since the maximum of a linear function over a compact convex set is attained at an extreme point, we have $\max_{\mathcal{X}(s)}\langle\theta,\cdot\rangle\ge\max_{\mathcal{Y}(s)}\langle\theta,\cdot\rangle\ge\max_{\ext\Conv\mathcal{X}(s)}\langle\theta,\cdot\rangle=\max_{\Conv\mathcal{X}(s)}\langle\theta,\cdot\rangle\ge\max_{\mathcal{X}(s)}\langle\theta,\cdot\rangle$.

    (ii) Let $c^*$ be the maximum, and set $U:=\{u\in\ext\Conv\mathcal{X}(s)\mid\langle\theta,u\rangle=c^*\}$ and $U_{\mathcal{Y}}:=\{u\in\mathcal{Y}(s)\mid\langle\theta,u\rangle=c^*\}$. By (i), $\emptyset\neq U\subseteq U_{\mathcal{Y}}\subseteq\argmax_{y\in\mathcal{X}(s)}\langle\theta,y\rangle$.
    The set $F^*$ of optimal solutions over $\Conv\mathcal{X}(s)$ is a face of the bounded polyhedron $\Conv\mathcal{X}(s)$, and since a face of a bounded polyhedron is the convex hull of the vertices it contains, $F^*=\Conv U$; by (i), $\argmax_{y\in\mathcal{X}(s)}\langle\theta,y\rangle=\mathcal{X}(s)\cap F^*$.
    If $|U_{\mathcal{Y}}|=1$, then $|U|=1$, so $F^*=U\subseteq\mathcal{X}(s)$ gives $\argmax_{y\in\mathcal{X}(s)}\langle\theta,y\rangle=U=U_{\mathcal{Y}}$.
    If $|U_{\mathcal{Y}}|\ge2$, then $U_{\mathcal{Y}}\subseteq\argmax_{y\in\mathcal{X}(s)}\langle\theta,y\rangle$ shows that $\argmax_{y\in\mathcal{X}(s)}\langle\theta,y\rangle$ is not a singleton.
\end{proof}

\begin{lemma}
    \label{lem:gen_e_vs_os}
    Let $\theta\in\Theta$, $s\in\mathcal{S}$, and $x^\prime\in\mathcal{X}(s)$.
    \begin{description}
        \item[(i)] We have $\ell^{01}(\theta;x^\prime,s)=0$ if and only if $\ell^{\optdex}(\theta;x^\prime,s)=0$ and the set $\argmax_{y\in\mathcal{X}(s)}\langle\theta,y\rangle$ is a singleton.
        \item[(ii)] If $\ell^{01}(\theta;x^\prime,s)=0$, then the optimal solution returned by the oracle satisfies $x^*(\theta,s)\geq x^\prime$.
        \item[(iii)] Under \Cref{assu:gen}(2), $\theta^*$ is a.s.\ consistent with the observations.
    \end{description}
\end{lemma}

\begin{proof}
    (i) As stated in \Cref{sec:performance_metric}, $\ell^{\optdex}(\theta;x^\prime,s)=0$ is equivalent to $\argmax_{y\in\mathcal{X}(s)}\langle\theta,y\rangle\cap\mathcal{X}_{\geq x^\prime}(s)\neq\emptyset$. When $\argmax_{y\in\mathcal{X}(s)}\langle\theta,y\rangle=\{v\}$, this is equivalent to $v\geq x^\prime$.
    (ii) Since $v$ is the only optimal solution, $x^*(\theta,s)=v\geq x^\prime$.
    (iii) This follows from the fact that $\ell^{01}(\theta^*;x,s)=0$ holds a.s.
\end{proof}

\subsection{Proof of \texorpdfstring{\Cref{thm:gen_upper_general}}{the theorem}}
\label{subsec:appendix_gen_upper_general}

Fix finitely many points $z_i=(s_i,x_i)$ ($i=1,\ldots,m$, $|\mathcal{Y}(s_i)|\le\kappa$).
For each $i$ and distinct $u,u^\prime\in\mathcal{Y}(s_i)$, consider the hyperplane through the origin $\{\theta\mid\langle\theta,u-u^\prime\rangle=0\}$; let $\mathcal{A}$ be the family of all such hyperplanes (identical hyperplanes are counted once), and let $h$ be their number.
Then $h\le m\kappa(\kappa-1)/2\le m\kappa^2/2$.
Fix a normal vector $a_H$ of each hyperplane $H\in\mathcal{A}$. For a sign vector $\sigma\in\{-,0,+\}^{\mathcal{A}}$, if $\{\theta\in\mathbb{R}^d\mid\operatorname{sign}\langle\theta,a_H\rangle=\sigma_H\ (\forall H\in\mathcal{A})\}$ is nonempty, it is called a \emph{cell} of $\mathcal{A}$ (we use this name to distinguish it from a face of a bounded polyhedron).
$\mathbb{R}^d$ is partitioned into the cells of $\mathcal{A}$, and a cell whose signs are all nonzero is a region of $\mathcal{A}$.

\begin{lemma}
    \label{lem:gen_face_constant}
    Assume \Cref{assu:margin}(2)(3). The map $\theta\mapsto(\ell^{01}(\theta;x_i,s_i))_{i=1}^m$ ($\theta\in\Theta$) is constant on the intersection of each cell of $\mathcal{A}$ with $\Theta$.
\end{lemma}

\begin{proof}
    On a cell, the sign of $\langle\theta,u\rangle-\langle\theta,u^\prime\rangle$ is constant for all $i$ and $u,u^\prime\in\mathcal{Y}(s_i)$, so the set of points in $\mathcal{Y}(s_i)$ that maximize $\langle\theta,\cdot\rangle$ is constant.
    By \Cref{lem:gen_extreme}(ii), whether $\argmax_{y\in\mathcal{X}(s_i)}\langle\theta,y\rangle$ is a singleton, and its element, are determined by this set alone. Hence $\ell^{01}(\theta;x_i,s_i)$ is constant.
\end{proof}

The following two lemmas bound the number of cells. In what follows, let $\Phi_k(h):=\sum_{i=0}^k\binom{h}{i}$ (with $\binom{h}{i}=0$ for $i>h$). $\Phi_k(h)$ is nondecreasing in $h$ and satisfies $\Phi_k(0)=\Phi_0(h)=1$ and $\Phi_k(h)=\Phi_k(h-1)+\Phi_{k-1}(h-1)$ ($k,h\ge1$).

\begin{lemma}
    \label{lem:gen_region_count}
    The number of regions of an arrangement of at most $h$ hyperplanes in $\mathbb{R}^k$ is at most $\Phi_k(h)$.
\end{lemma}

\begin{proof}
    Let $r_k(h)$ be the maximum number of regions of an arrangement of at most $h$ hyperplanes in $\mathbb{R}^k$; we show $r_k(h)\le\Phi_k(h)$ by induction on $k+h$.
    If $h=0$, the only region is $\mathbb{R}^k$, and if $k=0$, $\mathbb{R}^0$ is a single point and has no hyperplanes; hence $r_k(0)=r_0(h)=1$.
    Let $k,h\ge1$, let $\mathcal{A}$ be an arrangement of $h$ hyperplanes in $\mathbb{R}^k$, let $H\in\mathcal{A}$, and set $\mathcal{A}^\prime:=\mathcal{A}\setminus\{H\}$.
    For a region $R^\prime$ of $\mathcal{A}^\prime$ with sign pattern $\tau^\prime$, if $R^\prime\cap H=\emptyset$, then $g_H$ has no zeros on the convex set $R^\prime$, so its sign is constant, and $R^\prime$ is a region of $\mathcal{A}$.
    If $R^\prime\cap H\neq\emptyset$, then, since $R^\prime$ is an open convex set, it contains both points with $g_H>0$ and points with $g_H<0$, and $R^\prime\setminus H$ splits into the two regions $R^\prime\cap\{g_H>0\}$ and $R^\prime\cap\{g_H<0\}$ of $\mathcal{A}$.
    Since each region of $\mathcal{A}$ is contained in exactly one $R^\prime$, the number of regions of $\mathcal{A}$ equals the number of regions of $\mathcal{A}^\prime$ plus the number of regions of $\mathcal{A}^\prime$ that intersect $H$.

    We bound the number of regions of $\mathcal{A}^\prime$ that intersect $H$. Identify $H$ with $\mathbb{R}^{k-1}$ via an affine isomorphism. If $H^\prime\in\mathcal{A}^\prime$ is parallel to $H$ ($H\cap H^\prime=\emptyset$), then $g_{H^\prime}$ has constant sign on $H$; otherwise, $H\cap H^\prime$ is a hyperplane of $H$.
    Let $\mathcal{A}_H$ be the family of all the latter $H\cap H^\prime$ (duplicates counted once); then $\mathcal{A}_H$ is an arrangement of at most $h-1$ hyperplanes in $H$.
    If $R^\prime\cap H\neq\emptyset$, then $R^\prime\cap H=\{x\in H\mid\operatorname{sign}g_{H^\prime}(x)=\tau^\prime_{H^\prime}\ (\forall H^\prime\in\mathcal{A}^\prime)\}$ is a region of $\mathcal{A}_H$ (the conditions for parallel $H^\prime$ hold identically on $H$).
    Since distinct $R^\prime$ are disjoint, the map $R^\prime\mapsto R^\prime\cap H$ is injective. Hence the number of regions of $\mathcal{A}^\prime$ that intersect $H$ is at most $r_{k-1}(h-1)$.
    Therefore $r_k(h)\le r_k(h-1)+r_{k-1}(h-1)$, and the induction hypothesis together with $\Phi_k(h)=\Phi_k(h-1)+\Phi_{k-1}(h-1)$ yields $r_k(h)\le\Phi_k(h)$.
\end{proof}

\begin{lemma}
    \label{lem:gen_face_count}
    The number of cells of an arrangement of $h$ hyperplanes through the origin in $\mathbb{R}^d$ is at most $\Phi_d(2h)$.
\end{lemma}

\begin{proof}
    Let $H\in\mathcal{A}$ be the hyperplanes of the arrangement and $a_H$ their normal vectors, and let $\Sigma\subseteq\{-,0,+\}^{\mathcal{A}}$ be the set of sign vectors that can be written as $(\operatorname{sign}\langle\theta,a_H\rangle)_{H\in\mathcal{A}}$ for some $\theta\in\mathbb{R}^d$ (that is, the sign vectors corresponding to cells); $\Sigma$ is a finite set.
    For each $\sigma\in\Sigma$, choose a point $\theta^{\sigma}$ realizing the sign vector $\sigma$.
    Choose $\delta>0$ such that $|\langle\theta^{\sigma},a_H\rangle|>\delta$ for all $(\sigma,H)$ with $\langle\theta^{\sigma},a_H\rangle\neq0$.
    Let $\mathcal{A}_\delta$ be the arrangement of the $2h$ affine hyperplanes $\{\theta\mid\langle\theta,a_H\rangle=\delta\}$ and $\{\theta\mid\langle\theta,a_H\rangle=-\delta\}$ ($H\in\mathcal{A}$) (these are distinct, since the normal vectors of distinct $H$ are not parallel).
    By the choice of $\delta$, no $\theta^{\sigma}$ lies on any hyperplane of $\mathcal{A}_\delta$, so each $\theta^{\sigma}$ belongs to some region $R_\sigma$ of $\mathcal{A}_\delta$.
    $R_\sigma$ determines, for each $H$, the signs of $\langle\theta,a_H\rangle-\delta$ and $\langle\theta,a_H\rangle+\delta$, and $\sigma_H$ is determined from them: $-$ if both are negative, $0$ if they have opposite signs, and $+$ if both are positive.
    Hence $\sigma\mapsto R_\sigma$ is injective, and by \Cref{lem:gen_region_count}, $|\Sigma|\le\Phi_d(2h)$.
\end{proof}

Let $G(m)$ denote the maximum, over $m$ points $(s_i,x_i)$ ($|\mathcal{Y}(s_i)|\le\kappa$), of the number of error value patterns $(\ell^{01}(\theta;x_i,s_i))_{i=1}^m\in\{0,1\}^m$ realized by $\theta\in\Theta$ (the growth function).
From $2h\le m\kappa(\kappa-1)\le m\kappa^2$, the monotonicity of $\Phi_d$, and \Cref{lem:gen_face_constant,lem:gen_face_count}, we obtain
\begin{equation}
    G(m)\le\Phi_d(m\kappa^2)
    \label{eq:gen_growth}
\end{equation}
Moreover, if $n\ge d$, then we have
\begin{equation}
    \Phi_d(n)\le\left(\frac{n}{d}\right)^d\sum_{i=0}^{d}\binom{n}{i}\left(\frac{d}{n}\right)^i\le\left(\frac{n}{d}\right)^d\left(1+\frac{d}{n}\right)^n\le\left(\frac{\mathrm{e}n}{d}\right)^d .
    \label{eq:gen_binom_sum}
\end{equation}
Here, the first inequality holds because $n/d\ge1$ implies, for each $i=0,\ldots,d$, $\binom{n}{i}\le\binom{n}{i}\left(\frac{n}{d}\right)^{d-i}=\left(\frac{n}{d}\right)^d\binom{n}{i}\left(\frac{d}{n}\right)^i$. The second inequality follows by adding the nonnegative terms for $i=d+1,\ldots,n$ (recall $d\le n$) and applying the binomial theorem $\sum_{i=0}^{n}\binom{n}{i}\left(\frac{d}{n}\right)^i=\left(1+\frac{d}{n}\right)^n$. The last inequality follows from $1+t\le\mathrm{e}^t$ ($t\in\mathbb{R}$) with $t=d/n$, which gives $\left(1+\frac{d}{n}\right)^n\le\mathrm{e}^d$.

\begin{proof}[Proof of \Cref{thm:gen_upper_general}]
    Write $z^n:=(s^n,x^n)$ and $E_\theta:=\{(s,x)\mid \ell^{01}(\theta;x,s)=1\}$, and set
    \begin{equation*}
        \varepsilon:=\frac{2}{N}\left(d\log_2\frac{2\mathrm{e}N\kappa^2}{d}+\log_2\frac{2}{\beta}\right).
    \end{equation*}
    Define the event $\mathcal{E}$ as ``there exists $\theta\in\Theta$ that is consistent with the observations and satisfies $\mathrm{GE}(\theta)>\varepsilon$.''
    If $\mathcal{E}$ does not occur, the generalization error of every weight consistent with the observations is at most $\varepsilon$, so it suffices to show $\mathbb{P}(\mathcal{E})\le\beta$.
    Below, we show $\mathbb{P}(\mathcal{E})\le2\mathbb{P}(\mathcal{E}^\prime)$ in Step 2 (where $\mathcal{E}^\prime$ is the event defined in Step 2), $\mathbb{P}(\mathcal{E}^\prime)\le G(2N)2^{-\varepsilon N/2}$ in Step 3, and $2G(2N)2^{-\varepsilon N/2}\le\beta$ in Step 4.

    \emph{Step 1 (checking $N\varepsilon\ge8$)}: Since $N\ge d$ and $\kappa\ge2$, we have $N\kappa^2/d\ge4$, so $N\varepsilon\ge2d\log_2(8\mathrm{e})>8$.

    \emph{Step 2 (symmetrization)}: Independently of the observations, prepare an i.i.d.\ sample $z^{\prime1},\ldots,z^{\prime N}\sim\mathcal{P}$.
    Define the event $\mathcal{E}^\prime$ as ``there exists $\theta\in\Theta$ that is consistent with the observations and for which the number of $n$ with $z^{\prime n}\in E_\theta$ is at least $\varepsilon N/2$.''
    Write $\mathcal{E}^\prime_\omega$ for the section of $\mathcal{E}^\prime$ obtained by fixing the observations $\omega:=(z^1,\ldots,z^N)$. Since the observations and $z^\prime:=(z^{\prime1},\ldots,z^{\prime N})$ are independent, when integrating the indicator of $\mathcal{E}^\prime$ with respect to the product measure in $(\omega,z^\prime)$, we may first integrate in $z^\prime$ and then in $\omega$ (the Fubini--Tonelli theorem), so $\mathbb{P}(\mathcal{E}^\prime)=\mathbb{E}_\omega[\mathbb{P}_{z^\prime}(\mathcal{E}^\prime_\omega)]$.
    Let $\omega\in \mathcal{E}$, and take one $\theta^{\omega}$ that is consistent with the observations and satisfies $p:=\mathrm{GE}(\theta^{\omega})>\varepsilon$.
    The number $Y$ of $n$ with $z^{\prime n}\in E_{\theta^{\omega}}$ follows the binomial distribution $\mathrm{Bin}(N,p)$, so by Chebyshev's inequality and Step 1,
    \begin{equation*}
        \mathbb{P}_{z^\prime}\Bigl(Y<\frac{Np}{2}\Bigr)
        \le\mathbb{P}_{z^\prime}\Bigl(|Y-Np|>\frac{Np}{2}\Bigr)
        \le\frac{4Np(1-p)}{(Np)^2}
        \le\frac{4}{N\varepsilon}
        \le\frac{1}{2}.
    \end{equation*}
    If $Y\ge Np/2$, then $Y\ge\varepsilon N/2$, and $\theta^{\omega}$ satisfies the condition of $\mathcal{E}^\prime_\omega$; hence $\mathbb{P}_{z^\prime}(\mathcal{E}^\prime_\omega)\ge1/2$.
    Therefore $\mathbb{P}(\mathcal{E}^\prime)\ge\mathbb{E}_\omega[\mathbf{1}_{\mathcal{E}}(\omega)/2]=\mathbb{P}(\mathcal{E})/2$.

    \emph{Step 3 (swapping)}: Let $w:=(z^1,\ldots,z^N,z^{\prime1},\ldots,z^{\prime N})$. For $\tau\in\{0,1\}^N$, write $\mathrm{Sw}_\tau w$ for the tuple obtained by swapping $z^n$ and $z^{\prime n}$ for each $n$ with $\tau_n=1$.
    Since the $2N$ components of $w$ are i.i.d., $\mathrm{Sw}_\tau w$ and $w$ have the same distribution for each $\tau$, and $\mathbb{P}(\mathrm{Sw}_\tau w\in\mathcal{E}^\prime)=\mathbb{P}(w\in\mathcal{E}^\prime)=\mathbb{P}(\mathcal{E}^\prime)$.
    Averaging this over the $2^N$ choices of $\tau\in\{0,1\}^N$ and interchanging the finite sum and the expectation, we obtain
    \begin{equation*}
        \mathbb{P}(\mathcal{E}^\prime)
        =\frac{1}{2^N}\sum_{\tau\in\{0,1\}^N}\mathbb{E}_w\bigl[\mathbf{1}\{\mathrm{Sw}_\tau w\in\mathcal{E}^\prime\}\bigr]
        =\mathbb{E}_w\bigl[\mathbb{P}_\tau(\mathrm{Sw}_\tau w\in\mathcal{E}^\prime)\bigr].
    \end{equation*}
    Here, for fixed $w$, $\mathbb{P}_\tau(\mathrm{Sw}_\tau w\in\mathcal{E}^\prime):=2^{-N}|\{\tau\in\{0,1\}^N\mid\mathrm{Sw}_\tau w\in\mathcal{E}^\prime\}|$ is the probability when $\tau$ is drawn uniformly.
    Fix $w$, and let $\Xi(w)\subseteq\{0,1\}^{2N}$ be the set of error value patterns realized by $\theta\in\Theta$ on the $2N$ points of $w$.
    Since each component of $w$ satisfies $|\mathcal{Y}(s)|\le\kappa$ a.s., we have $|\Xi(w)|\le G(2N)$ a.s.
    Since the error value pattern of $\theta$ on $\mathrm{Sw}_\tau w$ is obtained by swapping the components of its pattern $\xi$ on $w$ according to $\tau$, if $\mathrm{Sw}_\tau w\in \mathcal{E}^\prime$, then for some $\xi\in\Xi(w)$ the following event $\mathcal{E}_\xi$ occurs: ``after swapping, the first $N$ components are all $0$, and at least $\varepsilon N/2$ of the last $N$ components are $1$.''
    Let $r$ be the number of components of $\xi$ equal to $1$. If $r<\varepsilon N/2$, then $\mathcal{E}_\xi$ does not occur. If, for some $n$, both the $n$th and the $(N+n)$th components of $\xi$ are $1$, then a $1$ remains in the first half regardless of the swap, so $\mathcal{E}_\xi$ does not occur.
    Otherwise, the $r$ components equal to $1$ belong to distinct $n$, and $\mathcal{E}_\xi$ is equivalent to all of them being placed in the second half, so $\mathbb{P}_\tau(\mathcal{E}_\xi)=2^{-r}\le2^{-\varepsilon N/2}$.
    Therefore $\mathbb{P}_\tau(\mathrm{Sw}_\tau w\in \mathcal{E}^\prime)\le\sum_{\xi\in\Xi(w)}\mathbb{P}_\tau(\mathcal{E}_\xi)\le G(2N)2^{-\varepsilon N/2}$, and we obtain $\mathbb{P}(\mathcal{E}^\prime)\le G(2N)2^{-\varepsilon N/2}$.

    \emph{Step 4}: By Steps 2 and 3, $\mathbb{P}(\mathcal{E})\le2G(2N)2^{-\varepsilon N/2}$.
    Since $2N\kappa^2\ge d$, \cref{eq:gen_growth,eq:gen_binom_sum} give $G(2N)\le(2\mathrm{e}N\kappa^2/d)^d$, and by the definition of $\varepsilon$,
    \begin{equation*}
        2G(2N)2^{-\varepsilon N/2}
        \le2\left(\frac{2\mathrm{e}N\kappa^2}{d}\right)^{d}2^{-d\log_2(2\mathrm{e}N\kappa^2/d)-\log_2(2/\beta)}
        =\beta.
    \end{equation*}
    Hence $\mathbb{P}(\mathcal{E})\le\beta$.
    That is, with probability at least $1-\beta$, $\mathcal{E}$ does not occur, and every weight $\hat{\theta}^N$ consistent with the observations satisfies $\mathrm{GE}(\hat{\theta}^N)\le\varepsilon$. This is the statement of \Cref{thm:gen_upper_general}.
\end{proof}

\subsection{Proof of \texorpdfstring{\Cref{thm:gen_lower}}{the theorem}}
\label{subsec:appendix_gen_lower}

Consider the following instance.
Let $d_0\ge2$ and $d:=2d_0+1$, and write elements of $\mathbb{R}^d$ as $(y_0,y_+,y_-)\in\mathbb{R}\times\mathbb{R}^{d_0}\times\mathbb{R}^{d_0}$.
Let the state set be the unit sphere $\mathcal{S}:=\mathbb{S}^{d_0-1}\subset\mathbb{R}^{d_0}$, and let the state distribution $\mathcal{P}_{\mathcal{S}}$ be absolutely continuous with respect to the uniform surface measure on $\mathbb{S}^{d_0-1}$.
Let the feasible region be
\begin{equation*}
    \mathcal{X}(s):=\{\,0,\ (-1,s,-s),\ (-3,-s,s)\,\}
\end{equation*}
and let the agent's action be $x=0$. That is, the distribution $\mathcal{P}$ of state--action pairs is given by $s\sim\mathcal{P}_{\mathcal{S}}$ and $x=0$. Let the weight space be $\Theta:=\Delta^{d-1}$ and $\theta^*:=(1,0,0)\in\Theta$.
For $\theta=(\theta_0,p,q)$, setting $u:=p-q\in\mathbb{R}^{d_0}$, we have
\begin{equation}
    \langle\theta,0\rangle=0,\quad
    \langle\theta,(-1,s,-s)\rangle=-\theta_0+\langle s,u\rangle,\quad
    \langle\theta,(-3,-s,s)\rangle=-3\theta_0-\langle s,u\rangle.
    \label{eq:gen_lower_values}
\end{equation}
This is obtained from the instance in the proof of Theorem~3.1 of \citet{fatemi2026tight} ($\psi(s,a)=(-2|a|+a,as)$, $a\in[-1,1]$, $\theta=(1,u)$) by lifting the features corresponding to $a\in\{-1,0,1\}$ so that the positive and negative parts of $u$ are placed in separate coordinates.

\begin{lemma}
    \label{lem:gen_lower_instance}
    This instance satisfies \Cref{assu:margin} and \Cref{assu:gen} with $\kappa=3$. Moreover, for $\theta=(\theta_0,p,q)\in\Theta$ and $u=p-q$, $\ell^{01}(\theta;0,s)=0$ is equivalent to ``$\langle s,u\rangle<\theta_0$ and $-\langle s,u\rangle<3\theta_0$.''
\end{lemma}

\begin{proof}
    Since $\mathcal{X}(s)$ consists of three points (a union of bounded polyhedra each consisting of a single point), it satisfies \Cref{assu:margin}(2), and every oracle satisfies \Cref{assu:margin}(3). The three points are distinct and are all extreme points of the convex hull, so by \Cref{lem:gen_extreme}(i), $\mathcal{Y}(s)=\mathcal{X}(s)$ and $\kappa=3$.
    By \cref{eq:gen_lower_values}, the values under $\theta^*$ are $0,-1,-3$, so $\argmax_{y\in\mathcal{X}(s)}\langle\theta^*,y\rangle=\{0\}$ and $x^*(\theta^*,s)=0\geq x$.
    Since the other two points have a negative first component, $\mathcal{X}_{\geq0}(s)=\{0\}$.
    $\argmax_{y\in\mathcal{X}(s)}\langle\theta,y\rangle=\{v\}$ with $v\ge0$ holds only when $v=0$ is the unique optimal solution, so \cref{eq:gen_lower_values} gives the latter equivalence.
\end{proof}

The proof uses the following two propositions.
The first is Theorem~1 of \citet{campi2008exact} on the violation probability of scenario programs. Here we state it in the form used by \citet{fatemi2026tight} in Theorem~A.1 and in the proof of Theorem~3.1.

\begin{proposition}[Theorem~1 of \citet{campi2008exact}]
    \label{prop:gen_cg}
    Let $k$ be a positive integer and $c\in\mathbb{R}^k$, and suppose that each point $\xi$ of a probability space $(\Xi,P_\Xi)$ is associated with a closed convex set $C_\xi\subseteq\mathbb{R}^k$.
    Let $\xi^1,\ldots,\xi^N$ be an i.i.d.\ sample from $P_\Xi$, and suppose that the scenario program
    \begin{equation*}
        \min_{u\in\mathbb{R}^k}\ \langle c,u\rangle
        \quad\text{s.t.}\quad
        u\in C_{\xi^n}\quad(n=1,\ldots,N)
    \end{equation*}
    has a.s.\ a unique solution $\hat{u}_N$. Set $V(u):=P_\Xi(u\notin C_\xi)$.
    Then, for any $\varepsilon\in(0,1]$, we have
    \begin{equation*}
        \mathbb{P}\bigl(V(\hat{u}_N)>\varepsilon\bigr)\le\sum_{i=0}^{k-1}\binom{N}{i}\varepsilon^i(1-\varepsilon)^{N-i} .
    \end{equation*}
    For each $n$, the constraint $u\in C_{\xi^n}$ is called a \emph{support constraint} if the scenario program with only this constraint removed has an optimal solution different from $\hat{u}_N$ or has no optimal solution.
    If the number of support constraints is a.s.\ exactly $k$ (that is, if the problem is fully supported in the sense of Definition~A.5 of \citet{fatemi2026tight}), then the above inequality holds with equality.
\end{proposition}

The second is the following property of an LP, shown in the proof of Theorem~3.1 of \citet{fatemi2026tight}.
Fix $c^\prime\in\mathbb{R}^{d_0}\setminus\{0\}$, and for $s^1,\ldots,s^N\sim\mathcal{P}_{\mathcal{S}}$ (i.i.d.) consider
\begin{equation}
    \min_{u\in\mathbb{R}^{d_0}}\ \langle c^\prime,u\rangle
    \quad\text{s.t.}\quad
    \langle s^n,u\rangle\le1,\quad
    -\langle s^n,u\rangle\le3\quad(n=1,\ldots,N).
    \label{eq:gen_lower_lp}
\end{equation}

\begin{proposition}[Proof of Theorem~3.1 of \citet{fatemi2026tight}]
    \label{prop:gen_fatemi_lp}
    Let $N\ge d_0$. Then \cref{eq:gen_lower_lp} has a.s.\ a unique solution $\hat{u}_N$.
    Moreover, the number of $n$ such that the LP with both inequalities $\langle s^n,u\rangle\le1$ and $-\langle s^n,u\rangle\le3$ corresponding to observation $n$ removed has an optimal solution different from $\hat{u}_N$ or has no optimal solution is a.s.\ exactly $d_0$.
\end{proposition}

\cref{eq:gen_lower_lp} is the scenario program in \Cref{prop:gen_cg} with $k=d_0$, $\Xi=\mathcal{S}$, $P_\Xi=\mathcal{P}_{\mathcal{S}}$, $c=c^\prime$, and $C_s:=\{u\mid\langle s,u\rangle\le1,\ -\langle s,u\rangle\le3\}$.
Hence, when $N\ge d_0$, \Cref{prop:gen_cg,prop:gen_fatemi_lp} imply that, for $V(u):=\mathcal{P}_{\mathcal{S}}(\langle s,u\rangle>1\text{ or }-\langle s,u\rangle>3)$ and any $\varepsilon\in(0,1]$,
\begin{equation}
    \mathbb{P}\bigl(V(\hat{u}_N)>\varepsilon\bigr)=\sum_{i=0}^{d_0-1}\binom{N}{i}\varepsilon^i(1-\varepsilon)^{N-i}.
    \label{eq:gen_lower_fatemi}
\end{equation}
We assume $d_0\ge2$ because, when $d_0=1$, the distribution on $\mathbb{S}^0=\{\pm1\}$ has atoms, identical constraints are duplicated, and the property in \Cref{prop:gen_fatemi_lp} fails.
Note that \cref{eq:gen_lower_lp} is Eq.~(6) of \citet{fatemi2026tight} written out for the normalization $\Theta_{\mathrm{F}}:=\{(1,u)\mid u\in\mathbb{R}^{d_0}\}$ and $\Psi(\theta):=\langle(0,c^\prime),\theta\rangle$. This $\Psi$ belongs to the class of objective functions $\mathcal{J}:=\{\theta\mapsto f(\langle c,\theta\rangle+b)\mid f\text{ is strictly monotone},\ b\in\mathbb{R},\ \langle c,\theta\rangle\text{ is not constant on }\Theta_{\mathrm{F}}\}$ targeted by the lower bound in that paper. Moreover, $V(\hat{u}_N)$ is the set-level mismatch of $\hat{\theta}=(1,\hat{u}_N)$, that is, the probability that the true optimal solution $0$ is not contained in the set of optimal solutions of $\hat{\theta}$ ($d$ and $T$ in that paper correspond to $d_0$ and $N$ in this section).

\begin{proof}[Proof of \Cref{thm:gen_lower}]
    \emph{Step 1}: For $\eta\in(0,1)$, let $u_\eta:=(1-\eta)\hat{u}_N$ and $\theta^{\eta}:=(1,u_\eta^+,u_\eta^-)/(1+\|u_\eta\|_1)\in\Theta$ (where $u^\pm$ are the componentwise positive and negative parts). Setting $Z_\eta:=1+\|u_\eta\|_1>0$, the first component of $\theta^{\eta}$ is $1/Z_\eta$ and $p-q=u_\eta/Z_\eta$, so by \Cref{lem:gen_lower_instance}, $\ell^{01}(\theta^{\eta};0,s)=0$ is equivalent to ``$\langle s,u_\eta\rangle<1$ and $-\langle s,u_\eta\rangle<3$.''
    For each $n$, the constraints of \cref{eq:gen_lower_lp} give $\langle s^n,u_\eta\rangle\le1-\eta<1$ and $-\langle s^n,u_\eta\rangle\le3(1-\eta)<3$, so $\theta^{\eta}$ is consistent with the observations.

    \emph{Step 2}: By \Cref{lem:gen_lower_instance},
    \begin{equation*}
        \mathrm{GE}(\theta^{\eta})=\mathcal{P}_{\mathcal{S}}(\Lambda_\eta),\quad
        \Lambda_\eta:=\Bigl\{\langle s,\hat{u}_N\rangle\ge\frac{1}{1-\eta}\text{ or }-\langle s,\hat{u}_N\rangle\ge\frac{3}{1-\eta}\Bigr\}.
    \end{equation*}
    As $\eta$ decreases, $\Lambda_\eta$ increases monotonically, and $\bigcup_{\eta\in(0,1)}\Lambda_\eta=\{\langle s,\hat{u}_N\rangle>1\text{ or }-\langle s,\hat{u}_N\rangle>3\}$. By continuity of measure from below, $\lim_{\eta\downarrow0}\mathrm{GE}(\theta^{\eta})=V(\hat{u}_N)$.
    Fix $\varepsilon\in(0,1]$; if there exists an integer $k\ge2$ with $\mathrm{GE}(\theta^{1/k})>\varepsilon$, let $k_N$ be the smallest such integer, and otherwise let $k_N:=2$; set $\hat{\theta}^N:=\theta^{1/k_N}$. $\hat{\theta}^N$ is a function of the observations and, by Step 1, is consistent with the observations.
    Since $\mathrm{GE}(\theta^{1/k})\to V(\hat{u}_N)$ also along $\eta=1/k$ ($k\to\infty$), $V(\hat{u}_N)>\varepsilon$ implies $\mathrm{GE}(\hat{\theta}^N)>\varepsilon$. Hence, by \cref{eq:gen_lower_fatemi},
    \begin{equation}
        \mathbb{P}\bigl(\mathrm{GE}(\hat{\theta}^N)>\varepsilon\bigr)
        \ge\sum_{i=0}^{d_0-1}\binom{N}{i}\varepsilon^i(1-\varepsilon)^{N-i}.
        \label{eq:gen_lower_binom}
    \end{equation}

    \emph{Step 3}: Let $X\sim\mathrm{Bin}(N,\varepsilon)$; then the right-hand side of \cref{eq:gen_lower_binom} equals $\mathbb{P}(X\le d_0-1)$.
    For $\varepsilon_1:=d_0/(2N)$, Markov's inequality gives $\mathbb{P}(X\ge d_0)\le N\varepsilon_1/d_0=1/2$, so the right-hand side is at least $1/2\ge\beta$.
    For $\varepsilon_2:=\ln(1/\beta)/(2N)\le1/2$, since $1-x\ge \mathrm{e}^{-2x}$ for $x\in[0,1/2]$, the right-hand side is at least $(1-\varepsilon_2)^N\ge \mathrm{e}^{-2N\varepsilon_2}=\beta$.
    Taking $\hat{\theta}^N$ of Step 2 for $\varepsilon:=\max\{\varepsilon_1,\varepsilon_2\}$, \cref{eq:gen_lower_binom} shows that, with probability at least $\beta$, $\mathrm{GE}(\hat{\theta}^N)>\varepsilon\ge(\varepsilon_1+\varepsilon_2)/2=(d_0+\ln(1/\beta))/(4N)$.
\end{proof}

\begin{remark}
    \label{rem:gen_lower_fatemi_estimator}
    Remove the nonnegativity constraint on the weights, and let the weight space be $\Theta_{\mathrm{F}}=\{(1,u)\mid u\in\mathbb{R}^{d_0}\}\subset\mathbb{R}^{d_0+1}$. In this remark only, we define $\ell^{01}$ and $\mathrm{GE}$ of \Cref{def:gen_error} for elements of $\Theta_{\mathrm{F}}$ by the same formulas. In the instance of Theorem~3.1 of \citet{fatemi2026tight}, $\mathcal{X}(s):=\{0,(-1,s),(-3,-s)\}$ and $x=0$, we have $\mathcal{X}_{\geq0}(s)=\{0\}$, so the OS loss coincides with the suboptimality loss.
    Hence the estimator $\hat{\theta}:=\argmin_{\theta\in\Theta_{\mathrm{F}}}\Psi(\theta)$ s.t.\ $\ell^{\optdex}(\theta;x^n,s^n)\le0$ ($\forall n$), which is Eq.~(6) of that paper written with the OS loss, coincides with the estimator of that paper written with the suboptimality loss (Eq.~(6) of that paper, denoted there by $\hat{\theta}^{\mathrm{sub}}_T$).
    Furthermore, since the set-level mismatch $0\notin\argmax_{y\in\mathcal{X}(s)}\langle\hat{\theta},y\rangle$ implies $\ell^{01}(\hat{\theta};0,s)=1$ (by the same argument as in Remark~3.2 of that paper), the lower bound of Theorem~3.1 of that paper, $\sup_{\theta^*,\mathcal{P}_{\mathcal{S}}}\inf_{\Psi\in\mathcal{J}}\mathbb{P}(\mathrm{GE}(\hat{\theta})>\varepsilon)\ge\sum_{i=0}^{d_0-1}\binom{N}{i}\varepsilon^i(1-\varepsilon)^{N-i}$, holds as it is.
    However, since $\hat{\theta}$ is a vertex of the LP, the optimal solution is not unique at observations with active constraints, so $\hat{\theta}$ is not consistent in the sense of this paper. \Cref{thm:gen_lower} converts this into a lower bound for consistent weights by shrinking $\hat{\theta}$ toward the interior.
\end{remark}

\section{Details of the numerical experiments}
\label{sec:appendix_experiments}

This section gives the details of the numerical experiments in \Cref{sec:numerical_experiments}.

\subsection{Problem settings of the experiments}
\label{subsec:appendix_problems}

{
This subsection gives the mathematical models of the forward problems of the five experiments and how their instances are constructed. In all four experiments other than the single-machine scheduling problem, the number of states is $N=20$, the dimension is $d\in\{20,40,60,80,100\}$, and $T=100$, that is, the oracle-call budget is $B=4000$. In every experiment, all the methods use the same instance within the same episode.

\paragraph{Knapsack problem}
We consider decisions on how many units of each item to pack under a limited capacity (e.g., loading or inventory planning).
The forward problem is the bounded integer knapsack problem
\[
    \max_{x\in\mathbb{Z}^d}\ \langle\theta,x\rangle
    \quad\text{s.t.}\quad
    \langle w_s,x\rangle\le c_s,\ \ 0\le x\le u_s .
\]
A state $s=(w_s,c_s,u_s)$ is generated by drawing the weights $w_s\in\{1,\ldots,9\}^d$, the capacity $c_s\in\{1,\ldots,19\}$, and the upper bounds on the numbers of items $u_s\in\{1,\ldots,4\}^d$ uniformly and independently for each component.
The true weight $\theta^*$ is drawn from the Dirichlet distribution $\mathrm{Dir}(\mathbf{1})$, and the action $x^n$ of the agent is generated by decreasing each component of $x^*(\theta^*,s^n)$ by $1$ independently with probability $0.2$ (without going below $0$).
Hence $x^*(\theta^*,s^n)\ge x^n$, and $\theta^*$ satisfies \cref{eq:outperformance_IOP}.

Solving outperformance inverse optimization in this problem means finding weights of the item values under which an optimal packing contains every item at least as many times as the agent's packing.

The forward problem is solved exactly by dynamic programming.
We run 50 episodes for each $d$ ($\theta^*$ and the states are redrawn for each pair of an episode and $d$).

\paragraph{Single-machine scheduling problem}
We consider deciding the order in which jobs with different release times are processed on a single machine according to their importance (e.g., sequencing tasks on a production line).
The forward problem is single-machine scheduling with release times that minimizes the weighted sum of completion times $\sum_{j=1}^d\theta_jC_j$.
A state $s$ consists of the release times $r_j\in\{0,\ldots,14\}$ and the processing times $p_j\in\{1,\ldots,4\}$ of the jobs $j=1,\ldots,d$ (each drawn independently and uniformly), and, to match the maximization form, the action is the sign-reversed completion times $x=-(C_1,\ldots,C_d)$.
Hence $x\ge x^n$ means that every job is completed no later than by the agent, that is, a schedule earlier than that of the agent.
From the optimal schedule under the true weight $\theta^*\sim\mathrm{Dir}(\mathbf{1})$, we generate the action $x^n$ of the agent by using thresholds $\tau^n_1\in\{8,\ldots,12\}$ and $\tau^n_2\in\{18,\ldots,22\}$ drawn for each state, delaying by $1$ every job whose completion time is at least $\tau^n_1$, and further delaying by another $1$ every job whose completion time after this delay is at least $\tau^n_2$ (the generation procedure is given in \Cref{alg:gen_scheduling}).
Hence $x^*(\theta^*,s^n)\ge x^n$, and $\theta^*$ satisfies \cref{eq:outperformance_IOP}.

Solving outperformance inverse optimization in this problem means finding weights of the job importance under which an optimal schedule completes every job no later than the agent's schedule.

We set the number of states to $N=10$, $d\in\{4,6,8\}$, 25 episodes for each $d$, and $T=100$, that is, the oracle-call budget is $B=2000$, and the forward problem is solved by SCIP.

\paragraph{Movie selection problem with viewing costs}
We consider a user selecting movies to watch under a budget on the viewing cost and per-genre upper limits.
We use the rating data (\texttt{u1.base}) of MovieLens100K \citep{harper2015movielens}.
The ground set consists of the top $d$ movies in the number of ratings, and $\Gamma\in\{0,1\}^{18\times d}$ is the genre indicator matrix of the movies (the rows correspond to the genres $g=1,\ldots,18$).
A state $s$ corresponds to one user, and $W_s$ denotes the set of movies in the ground set rated by the user $s$.
The forward problem is the movie selection under per-genre upper limits and a budget on the viewing cost,
\[
    \max_{x\in\{0,1\}^d}\ \langle\theta,x\rangle
    \quad\text{s.t.}\quad
    \langle w,x\rangle\le b_s,\ \ \Gamma x\le\bar{c}_s .
\]
The cost $w_i$ is the number of years elapsed from the release of the movie $i$ at the time of rating (the median over users, rounded to years and set to at least $1$), and the budget and the per-genre upper limits are determined from the viewing history as $b_s=\langle w,\mathbf{1}_{W_s}\rangle+\min_i w_i$ (where $\mathbf{1}_{W_s}$ is the indicator vector of $W_s$) and $\bar{c}_{s,g}=(\Gamma\mathbf{1}_{W_s})_g+1$.
The action of the agent is the indicator vector $x^n=\mathbf{1}_{R_s}$ of the set $R_s\subseteq W_s$ of movies to which the user $s$ gave a rating of at least $4$.
The users with $|R_s|\ge5$ are divided into groups of 20 in the order of user IDs, each group forming one episode, and we run 20 episodes for each $d$. No randomness is used.
Since $R_s\subseteq W_s$, $x^n$ is feasible. In this experiment, we do not specify a true weight and do not assume the existence of a weight satisfying \cref{eq:outperformance_IOP}.

Solving outperformance inverse optimization in this problem means finding preference weights on the movies under which an optimal selection contains all the movies to which the user gave a rating of at least $4$.

The forward problem is solved by SCIP.

\paragraph{Weighted multi-cover problem (Cora, ogbn-arxiv)}
We formulate the situation in which a researcher selects (reads or cites) a limited number of papers as a weighted multi-cover of words.
For candidate papers $i\in P_s$ and words $j=1,\ldots,d$ (the set $P_s$ is described below), let $A_{ij}\in\{0,1\}$ indicate whether the word $j$ appears in the title or the text of the paper $i$, and let the forward problem be
\[
    \max_{z\in\{0,1\}^{|P_s|},\,x\in\{0,\ldots,r\}^d}\ \langle\theta,x\rangle
    \quad\text{s.t.}\quad
    x_j\le\sum_{i\in P_s}A_{ij}z_i,\ \ \sum_{i\in P_s}z_i\le k_s
\]
($x\in\{0,\ldots,r\}^d$; the oracle returns, as the action, the componentwise $\min(r,A^\top z)$ computed from the selected papers $z$; that is, $\mathcal{X}(s)$ is the set of all $x$ satisfying the above constraints for some $z$). The action $x_j$ is ``the number of papers covering the word $j$ (up to $r$)'', and $\theta$ is the weight of interest in the words.
A state $s$ corresponds to a paper $v$, the action of the agent is the depth $x^n_j=\min(r,\,|\{i\in E_s:A_{ij}=1\}|)$ of the set $E_s\subseteq P_s$ of papers in a citation relation with $v$ (described below), and the upper bound on the number of papers is $k_s=|E_s|$.
Hence $x\ge x^n$ means ``covering every word at a depth no smaller than that of the agent with the same number of papers as the agent'', and $E_s$ itself is a feasible solution.

Solving outperformance inverse optimization in this problem means finding weights of interest in the words under which an optimal selection of papers covers every word at a depth no smaller than that of the agent with the same number of papers.

The words are the first $d$ words obtained by taking, alternately in decreasing order of the field size, the words ranked by their per-field scores, so that the words of several fields compete for the limited number of papers even when the number of words is small.
We impose the upper bound on the depth because, with $r=1$ (maximum coverage), the OS loss tends to be $0$ by construction under the uniform weight, and without the upper bound ($r\ge k_s$) the problem reduces to selecting the top $k_s$ papers (\Cref{subsec:appendix_multicover}). Note, however, that this upper bound is not effective in the states with $k_s\le r$ (38\% for Cora and 35\% for ogbn-arxiv).
For Cora \citep{sen2008collective}, the words are the 1433-word Bag-of-Words, the scores are the occurrence rates in each of the 7 fields, the states are the papers $v$ of degree 3--10, $E_s$ is the set of adjacent papers in the undirected citation graph (the papers cited by $v$ and the papers citing $v$), $P_s$ is the 2-hop neighborhood of $v$ excluding $v$ itself, and $r=3$.
For ogbn-arxiv \citep{hu2020open}, we use the direction of citations and the publication years; the states are the papers $v$ of 2019 that cite 5--8 papers (published in 2019 or earlier), $E_s$ is the set of papers actually cited by $v$, and $P_s$ consists of $E_s$ and the papers cited by them (all published in 2019 or earlier, excluding $v$ itself).
The words are constructed from the titles and abstracts \citep{wang2020microsoft}, the scores are the distinctiveness in each of the 40 fields (the occurrence rate within the field minus the overall occurrence rate), and $r=5$.
For ogbn-arxiv, the states of the episode $e$ are selected from the field with the $(e+1)$-th largest number of papers in 2019 among the fields with at least $N$ papers satisfying the conditions ($e=0,\ldots,19$).
In both cases, we run 20 episodes for each $d$, and the forward problem is solved by SCIP. The generation procedures are given in \Cref{alg:gen_cora,alg:gen_arxiv}.

}
We describe the computing environment in \Cref{subsec:appendix_environment}, the implementation and the hyperparameters of each method in \Cref{subsec:appendix_implementation}, the details of the instances of the weighted multi-cover problem in \Cref{subsec:appendix_multicover}, the pseudocode of the evaluation protocol, of each method, and of the instance generation in \Cref{subsec:appendix_pseudocode_methods,subsec:appendix_pseudocode_instances}, and the results of all the methods in the five experiments in \Cref{subsec:appendix_full_results}.

\subsection{Computing environment}
\label{subsec:appendix_environment}

All the experiments were run on a machine with Intel Xeon Gold 6444Y CPUs (16 cores, 2 sockets, 64 threads in total), 1\,TB of memory, and Ubuntu 24.04.4 LTS.
The implementation is in Python 3.11.15 \citep{Python3} with NumPy 2.4.3 and SciPy 1.17.1.
As forward problems, the knapsack problem was solved by dynamic programming, and the MovieLens, Cora, ogbn-arxiv, and single-machine scheduling problems were solved by SCIP 10.0.0 (with SoPlex 8.0.0 as the LP solver) via OR-Tools 9.15 \citep{ortools}.
The time limit of SCIP was 60 seconds per solve for the MovieLens, Cora, and ogbn-arxiv problems, and 300 seconds for the single-machine scheduling problem.
For the MovieLens, Cora, ogbn-arxiv, and single-machine scheduling problems, optimality was proved in every solve (the implementation stops with an error when optimality is not proved, and solutions without a proof of optimality are not used). For the single-machine scheduling problem, we observed the following two issues with SCIP, and therefore we first solved with presolve turned off and, when optimality was not proved, rebuilt the model and solved again with presolve turned on. First, with presolve turned on, the process terminates abnormally within presolve for certain inputs. Second, with presolve turned off, SCIP sometimes erroneously declares a feasible problem infeasible when the objective coefficients include extremely small values (of the order of $10^{-9}$). In either path, the output of the oracle is a solution whose optimality for the original weight is proved.
Among the master problems of the inner solvers, the LPs were solved by HiGHS 1.12.0 in SciPy and the QPs by Clarabel 0.11.1. The metric projections of ONS and MetaGrad were also solved by Clarabel.
MovieLens100K and Cora were loaded with PyTorch Geometric 2.8.0.post1 \citep{fey2019fast} (for Cora, the Planetoid version \citep{yang2016revisiting}).
MovieLens100K was used in accordance with the terms of use of the GroupLens Research Project (use for research purposes only, no redistribution of the data, and citation of \citet{harper2015movielens} in publications).
For Cora (the Planetoid version), we used the distributed files at \url{https://github.com/kimiyoung/planetoid} (MIT license).
For ogbn-arxiv, we directly loaded the files distributed by OGB (Open Graph Benchmark) \citep{hu2020open} (citations, publication years, fields, and the correspondence to the paper IDs of the Microsoft Academic Graph) and the raw titles and abstracts \citep{wang2020microsoft} also distributed by OGB. We used ogbn-arxiv as distributed under the ODC-BY (Open Data Commons Attribution License) license.
Three CPU cores were allocated to the computation, and at most three processes were run simultaneously.
The movie selection problem with viewing costs was run in three simultaneous processes by dividing the episodes into 0--7, 8--15, and 16--19, and the knapsack problem was run in a single process following the end of one of these processes. The single-machine scheduling problem was run in three simultaneous processes by dividing the episodes into 0--8, 9--16, and 17--24.
The two experiments on the weighted multi-cover problem were run in three simultaneous processes by dividing the episodes (while three processes run simultaneously, each process uses one core).
Each method was run sequentially within a single process, and the computation time was measured in wall-clock time. As the random number generator we used \texttt{numpy.random.}\allowbreak\texttt{default\_rng}; for the knapsack problem and the single-machine scheduling problem, a generator with the seed $(20260928,e,d)$ was created for each pair of an episode $e$ and a dimension $d$ to generate the instances. For the weighted multi-cover problem, the seed $(20260928,e)$ was used to sample the states of the episode $e$. Since in both cases the seed is determined only by the episode and the dimension, the same instances are obtained even when a run is resumed from the middle.

\subsection{Implementation and hyperparameters of the methods}
\label{subsec:appendix_implementation}

{
All 13 methods are listed in \Cref{tab:methods}. In addition to the 7 methods compared in the main text (\Cref{subsec:experimental_setup}), they include Sub-ONS and Sub-MetaGrad, which are online learning methods updating state by state, and variants of OS-DC with different inner solvers (OS-DC[CP]+SGF, OS-DC[LV+CPF]+SGF, OS-DC[ONS], and OS-DC[MetaGrad]). SGF (subgradient finishing) denotes applying the projected subgradient method with a small step size after DC optimization stops, starting from the obtained weight. It is added because, even at a weight at which the OS loss is $0$, the optimal solution of the forward problem is in general not unique, and the solution of the oracle does not necessarily dominate the action. All hyperparameters were fixed in advance and were not tuned on the evaluation instances. Within each pair that differs only in the loss function, the OS side and the Sub side use the same values (the Sub side is given twice as many iterations so that the oracle-call budgets coincide).

\begin{table}[tbp]
    \centering
    \caption{The 13 compared methods. PSG is the projected subgradient method, ONS is the Online Newton Step \citep{hazan2007logarithmic}, and MetaGrad is \citet{van2016metagrad}; we also call these optimization methods learners. FB (full-batch) denotes the update that gives the learner the average of the subgradients over all states only once per iteration and does not use the averaged iterate (the average of the iterates so far). The online-to-batch conversion is the conversion that outputs the averaged iterate. The inner solvers CP, LV+CPF, and PB+CPF are the cutting-plane method, the level method with cutting-plane finishing, and the proximal bundle method with cutting-plane finishing, respectively (\Cref{sec:appendix_inner_rules}). In the main text, OS-DC[PB+CPF]+SGF is simply written as OS-DC.}
    \label{tab:methods}
    \begin{tabular}{lllp{0.22\linewidth}}
        \hline
        Method & \twoline{Optimization}{method} & Pseudocode & Remarks \\
        \hline
        \multicolumn{4}{l}{Empirical mean of the suboptimality loss (existing)} \\
        \hline
        Sub-PSG & PSG & \Cref{alg:psg} & \\
        Sub-ONS & ONS & \Cref{alg:sub_online} & state-by-state update, online-to-batch conversion \citep{sakaue2025online} \\
        Sub-MetaGrad & MetaGrad & \Cref{alg:sub_online} & state-by-state update, online-to-batch conversion \citep{sakaue2025online} \\
        Sub-ONS(FB) & ONS & \Cref{alg:fb_online} & FB \\
        Sub-MetaGrad(FB) & MetaGrad & \Cref{alg:fb_online} & FB \\
        \hline
        \multicolumn{4}{l}{Empirical mean of the OS loss (proposed)} \\
        \hline
        OS-PSG & PSG & \Cref{alg:psg} & \\
        OS-ONS(FB) & ONS & \Cref{alg:fb_online} & FB \\
        OS-MetaGrad(FB) & MetaGrad & \Cref{alg:fb_online} & FB \\
        OS-DC[CP]+SGF & DC optimization & \Cref{alg:osdc_full} & inner: CP, with SGF \\
        OS-DC[LV+CPF]+SGF & DC optimization & \Cref{alg:osdc_full} & inner: LV+CPF, with SGF \\
        OS-DC[PB+CPF]+SGF & DC optimization & \Cref{alg:osdc_full} & inner: PB+CPF, with SGF (OS-DC in the main text) \\
        OS-DC[ONS] & DC optimization & \Cref{alg:osdc_full} & inner: ONS (FB; \Cref{alg:dc_inner_online}) \\
        OS-DC[MetaGrad] & DC optimization & \Cref{alg:osdc_full} & inner: MetaGrad (FB; \Cref{alg:dc_inner_online}) \\
        \hline
    \end{tabular}
\end{table}
}

\paragraph{Implementation of DC optimization}
The implementation of DC optimization (\Cref{alg:dc_outer,alg:dc_inner} in the appendix) is a DCA that truncates the inner loop at $J=20$ iterations by default (inexact DCA). Its validity under truncation is ensured by the following design: (i) the descent property that holds even under truncation (\Cref{prop:dc_descent}); (ii) tie stopping only when the inner solver reports convergence ($\mathrm{converged}=\mathrm{TRUE}$) (to avoid stopping at a tie that occurs by chance while the inner problem has not converged); (iii) cut inheritance, which improves the accuracy of the inner solver without additional oracle cost (\Cref{prop:dc_cut_validity}); and (iv) certificate stopping (stopping when the empirical mean of the OS loss falls below the threshold $\epsilon_{\mathrm{gap}}$; \Cref{prop:dc_global_certificate}) and target-attainment stopping (stopping when $\ell^{\mathrm{OP}}$ falls below the threshold), both of which are valid regardless of the convergence of the inner loop.

\paragraph{Projected subgradient method (Sub-PSG, OS-PSG)}
The update at iteration $j$ is
\[
    \theta^{j+1}=\Pi_{\Delta^{d-1}}\!\left(\theta^{j}-\eta\,(j+1)^{-1/2}\,\frac{g_j}{\|g_j\|_2}\right),
    \qquad
    g_j=\sum_{n=1}^N\bigl(x^*(\theta^{j},s^n)-y_j^n\bigr) .
\]
Here $\Pi_{\Delta^{d-1}}$ is the Euclidean projection onto the probability simplex (an exact sorting-based method), and $y_j^n$ is $x^n$ for Sub-PSG and $x^*(\theta^{j};x^n,s^n)$ for OS-PSG ($g_j/N$ of OS-PSG is the gradient of the empirical mean of the OS loss in the sense of weak derivatives; \Cref{prop:gredient_OA_loss}).
The step-size coefficient is $\eta=0.1$. The method stops when $g_j=0$. The same coefficient, normalization, and projection are used in all five experiments, including the single-machine scheduling problem.
When $\langle\theta^{j},x^*(\theta^{j},s^n)-y_j^n\rangle<0$ holds because of solver errors, $x^*(\theta^{j},s^n)$ is replaced by $y_j^n$ before computing $\ell^{\mathrm{OP}}$ (the solution before the replacement is used to compute $g_j$).
The number of iterations is $T$ for OS-PSG and $2T$ for Sub-PSG.

\paragraph{Sub-ONS}
One pass over the states $s^1,\ldots,s^N$ in order constitutes one epoch, and at each round we update by the Online Newton Step \citep{hazan2007logarithmic}
\[
    H_t=H_{t-1}+g_tg_t^\top,\qquad
    \theta^{t+1}=\Pi^{H_t}_{\Delta^{d-1}}\!\left(\theta^{t}-\gamma^{-1}H_t^{-1}g_t\right)
\]
using the subgradient $g_t=x^*(\theta^{t},s^n)-x^n$ of $\ell^{\sub}(\cdot\,;x^n,s^n)$. Here $\Pi^{H}_{\Delta^{d-1}}$ is the projection onto the probability simplex in the $H$-norm, $H_0=\epsilon_{\mathrm{reg}}I$, and $\gamma=1$, $\epsilon_{\mathrm{reg}}=1$.
At the end of each epoch, the average of the $\theta^{t}$ visited so far (online-to-batch conversion) is evaluated on all states and recorded.
The number of oracle calls per epoch is $2N$ in total, $N$ for the updates and $N$ for the evaluation, so that $T$ epochs use up the oracle-call budget $B=2TN$.

\paragraph{Sub-MetaGrad}
This method replaces the learner of Sub-ONS by MetaGrad \citep{van2016metagrad}. It runs in parallel an ONS for each learning rate $\eta_i=2^{-i}/(5DG)$ ($i=0,\ldots,\lceil\log_2(TN+1)/2\rceil$) and aggregates them by tilted exponential weights.
We use the diameter $D=\sqrt{2}$ of the probability simplex, and for the upper bound $G$ on the norm of the subgradients we use the norm of the first nonzero (per-state) subgradient.
The online-to-batch conversion and the counting of oracle calls are the same as for Sub-ONS.

\paragraph{Full-batch ONS and MetaGrad (Sub-ONS(FB), Sub-MetaGrad(FB), OS-ONS(FB), OS-MetaGrad(FB))}
As in the projected subgradient method, at iteration $j$ all states are solved with the same weight $\theta^{j}$, and the full-batch subgradient
\[
    g_j=\frac{1}{N}\sum_{n=1}^N\bigl(x^*(\theta^{j},s^n)-y_j^n\bigr)
\]
is given to the learner only once. Here $y_j^n$ is $x^n$ on the Sub side and $x^*(\theta^{j};x^n,s^n)$ on the OS side.
ONS updates as
\[
    H_j=H_{j-1}+g_jg_j^\top,\qquad
    \theta^{j+1}=\Pi^{H_j}_{\Delta^{d-1}}\!\left(\theta^{j}-\gamma^{-1}H_j^{-1}g_j\right)
\]
(with the initial metric $\epsilon_{\mathrm{reg}}I$, $\gamma=1$, $\epsilon_{\mathrm{reg}}=1$), and MetaGrad applies the same update as Sub-MetaGrad to $g_j$ ($D=\sqrt{2}$, $G=\|g_0\|_2$).
The number of learning rates of MetaGrad is set to $\lceil\log_2(n_{\mathrm{it}}+1)/2\rceil+1$ from the number of iterations $n_{\mathrm{it}}$.
The number of iterations is $n_{\mathrm{it}}=T$ on the OS side ($2N$ oracle calls per iteration) and $n_{\mathrm{it}}=2T$ on the Sub side ($N$ per iteration), and both use up exactly the oracle-call budget $B=2TN$.
The averaged iterate is not used; at each iteration, $\ell^{\mathrm{OP}}$ at $\theta^{j}$ and the empirical mean of the OS loss are computed directly from the solutions used for the update and recorded ($x^*(\theta^{j};x^n,s^n)$ used for recording the OS loss on the Sub side is a solve used only for logging and is not counted as an oracle call).
Stopping when $g_j=0$ is the same as in the projected subgradient method. The replacement against solver errors is used not only for $\ell^{\mathrm{OP}}$ but also for computing $g_j$.
These methods differ from Sub-ONS and Sub-MetaGrad in that each update uses all states and in that the averaged iterate is not used.

\paragraph{DC optimization of the OS loss (OS-DC)}
The maximum number of outer iterations is $K=2T/20$, the maximum number of inner iterations is $J=20$, the threshold of tie stopping is $\epsilon_{\mathrm{tie}}=10^{-8}$, and the threshold of certificate stopping is $\epsilon_{\mathrm{gap}}=10^{-6}$.
As the rules for selecting the evaluation points of the inner solver (\Cref{sec:appendix_inner_rules}), we compare the cutting-plane method (CP), the level method with cutting-plane finishing (LV+CPF), and the proximal bundle method with cutting-plane finishing (PB+CPF), and subgradient finishing is added to all of them.
In addition, we also compare OS-DC with ONS or MetaGrad as the inner solver (without subgradient finishing) (next paragraph).
For CP, the inner tolerance is $\epsilon_{\mathrm{in}}=10^{-6}$; for LV+CPF and PB+CPF, $\epsilon_{\mathrm{in}}=0$, the threshold for switching phases is $10^{-6}$, and the termination tolerance is $\epsilon_{\mathrm{term}}=10^{-9}$ (\Cref{alg:rule_cpfinish}). The exactness of \Cref{theo:cpfinish_exact} is guaranteed for $\epsilon_{\mathrm{term}}=0$; the positive value is a safeguard against rounding errors.
In subgradient finishing (SGF), OS-PSG is run for $T$ iterations from the output of DC optimization with the step-size coefficient $\eta=10^{-3}$ ($1/100$ of that of the usual projected subgradient method) in all five experiments.
The step size is made small because a large move makes the OS loss positive again.
At every evaluation of the forward-problem oracle (including the inner evaluation points), $\ell^{\mathrm{OP}}$ is computed, and the whole procedure stops when it becomes $0$.
The proximal weight of the proximal bundle method is fixed at $u=1$, the descent parameter is $\rho=0.1$, and the level parameter of the level method is $\lambda=1/(2+\sqrt{2})$. The output of DC optimization is the attaining point in the case of target-attainment stopping and, otherwise, the point with the smallest $\ell_k$ among all outer iterations (\Cref{alg:dc_outer}). The whole procedure is given in \Cref{alg:osdc_full}.
The evaluation points computed by the inner solver along the way (the solutions of the LPs and QPs of the master problems) are passed to the oracle without being projected onto the probability simplex, and only the output of the inner solver is projected onto the probability simplex. Hence the intermediate evaluation points may lie outside the probability simplex within the tolerance of the LP and QP solvers. The weights reported as best-so-far may also lie outside the probability simplex within the same tolerance.
OS-DC in the main text is OS-DC[PB+CPF]+SGF.

\paragraph{OS-DC with ONS or MetaGrad as the inner solver (OS-DC[ONS], OS-DC[MetaGrad])}
This is OS-DC that solves the convex subproblem $\min_{\theta\in\Theta}\bar{\ell}^{\sub}(\theta;y^\prime)$ (\cref{eq:def_sub_loss_main}) of a branch $y^\prime$ by ONS or MetaGrad.
The convex subproblem $\bar{\ell}^{\sub}(\cdot\,;y^\prime)$ is exactly the empirical mean of the suboptimality loss in which the actions of the agent are replaced by the dominating vertices $y^{\prime n}$.
The outer loop (\Cref{alg:dc_outer}) and $K$, $\epsilon_{\mathrm{tie}}$, and $\epsilon_{\mathrm{gap}}$ are the same as for OS-DC, and only the inner loop is replaced as follows.
At each inner iteration $t$, $x^*(\theta^{t},s^n)$ is computed for all states ($N$ oracle calls), and, using the full-batch subgradient $g_t=\frac{1}{N}\sum_{n=1}^N\bigl(x^*(\theta^{t},s^n)-y^{\prime n}\bigr)$, OS-DC[ONS] performs the same update as Sub-ONS,
\[
    H_t=H_{t-1}+g_tg_t^\top,\qquad
    \theta^{t+1}=\Pi^{H_t}_{\Delta^{d-1}}\!\left(\theta^{t}-\gamma^{-1}H_t^{-1}g_t\right)
\]
($H_0=\epsilon_{\mathrm{reg}}I$, $\gamma=1$, $\epsilon_{\mathrm{reg}}=1$).
OS-DC[MetaGrad] replaces this update by the same MetaGrad as Sub-MetaGrad ($D=\sqrt{2}$, and $G$ is the norm of the first subgradient). The number of learning rates is set to $\lceil\log_2(KJ+1)/2\rceil+1$ from the total number of rounds $KJ$ of the whole DC optimization.
Within each outer iteration, the inner loop outputs the point with the smallest $\bar{\ell}^{\sub}(\cdot\,;y^\prime)$ among the visited $\theta^{t}$ and their average (online-to-batch conversion).
The internal state of the learner is carried over across outer iterations. That is, the first outer iteration starts from $\theta^0=\mathbf{1}/d$, and each subsequent one resumes from the state at the end of the previous outer iteration. What is carried over is the metric $H_t$ and the last iterate for ONS, and, for MetaGrad, the point and the metric of each learning rate, the aggregation weights, $G$, and the learning rates.
On the other hand, the cumulative sum for the average is reset at each outer iteration, because the guarantee of the online-to-batch conversion holds only for a fixed loss, and the convex subproblem changes when the branch changes.
The maximum number of inner iterations is $J=20$ and the tolerance is $\epsilon_{\mathrm{in}}=10^{-6}$; the inner loop stops when $\bar{\ell}^{\sub}(\theta^{t};y^\prime)\le\epsilon_{\mathrm{in}}$.
Since ONS and MetaGrad have no cutting-plane model, no lower bound $\mathrm{LB}$ is available, and cut inheritance is not used. The inner loop reports converged $=$ TRUE only when $\bar{\ell}^{\sub}(\theta^{t};y^\prime)\le\epsilon_{\mathrm{in}}$ or when the subgradient becomes $0$.
Neither subgradient finishing nor cutting-plane finishing is used.
Even when the branch changes, the metric built from the subgradients of the previous branch is carried over without correction. Hence the whole procedure takes the form of ``running a single ONS (MetaGrad) while updating the linearization point along the way''.

{
\begin{remark}[Replacement against solver errors]
    \label{rem:solver_error_replacement}
    In the projected subgradient method, subgradient finishing, and the full-batch ONS and MetaGrad, when $\langle\theta,x^*(\theta,s^n)-y^n\rangle<0$, $x^*(\theta,s^n)$ is replaced by $y^n$ ($x^n$ on the Sub side and $x^*(\theta;x^n,s^n)$ on the OS side) before computing $\ell^{\mathrm{OP}}$ (in the projected subgradient method and subgradient finishing, the solution before the replacement is used for the subgradient $g_j$, whereas in the full-batch ONS and MetaGrad the solution after the replacement is used).
    On the other hand, in Sub-ONS, Sub-MetaGrad, and the main body of DC optimization, $\ell^{\mathrm{OP}}$ is computed from the solution returned by the oracle as it is.
    However, since every $y^n$ is a feasible solution in $\mathcal{X}(s^n)$, $\langle\theta,x^*(\theta,s^n)\rangle\ge\langle\theta,y^n\rangle$ always holds if $x^*(\theta,s^n)$ is an exact optimal solution, and this replacement does not occur.
    In our experiments, the knapsack problem is solved exactly by dynamic programming, and the MovieLens, Cora, ogbn-arxiv, and single-machine scheduling problems stop with an error when optimality is not proved (\Cref{subsec:appendix_environment}); hence this replacement does not occur.
    Therefore, the difference in the treatment among the methods does not affect the results.
\end{remark}
}

\paragraph{Reuse of the models of the forward problems}
For the MovieLens, Cora, and ogbn-arxiv problems (for the single-machine scheduling problem the model is rebuilt at every solve), the model of the forward problem for the same state and the same dominance constraint is kept, and only the objective coefficients are replaced for re-solving.
Since which optimal solution is returned when the optimal solution is not unique may depend on the preceding solve history, the kept models are reset before running each method, so that the trajectory of a method does not depend on the order in which the other methods are run.

\subsection{Details of the instances of the weighted multi-cover problem}
\label{subsec:appendix_multicover}

\paragraph{Words}
For Cora, the 1433-word Bag-of-Words (presence or absence) is used as it is.
For ogbn-arxiv, the titles and abstracts are lowercased, and we use the words consisting only of alphabetic characters with at least three letters, excluding common English stop words (no stemming). Words that appear in fewer than 20 papers are discarded.
For each field $c$, the score of a word $j$ is the fraction of the papers of the field $c$ containing the word $j$ (occurrence rate) for Cora, and the occurrence rate minus the occurrence rate over all papers (distinctiveness) for ogbn-arxiv; the ranking of the words is constructed by taking one word at a time alternately from the per-field rankings of the scores, starting from the field with the largest number of papers (skipping words already taken).
The words of dimension $d$ are the first $d$ words of this ranking, so that the sets of words are nested as $d$ increases.
We use the distinctiveness rather than the occurrence rate for ogbn-arxiv because with the occurrence rate, general words independent of the fields, such as based, paper, show, and model, occupy the top ranks.

\paragraph{States}
The selection of the states does not depend on the dimension, and 20 papers are sampled at random for each episode (since the episode number is included in the seed, the same states are used in the same episode even when the dimension is changed).
For Cora, the number of candidates has median 25 (3--210), and the number of papers selected by the agent has median 4 (3--10).
For ogbn-arxiv, an edge $(u,w)$ represents that $u$ cites $w$ (98.2\% of the edges have a citing paper whose publication year is no earlier than that of the cited paper).
By the condition on the publication year (2019 or earlier), $E_s$ and $P_s$ consist only of papers that existed at the time $v$ was written.
The number of candidates has median 38 (6--211), and the number of papers selected by the agent has median 6 (5--8).
The fields of the episodes $0,1,\ldots,19$ are, in decreasing order of the number of papers in 2019, cs.LG, cs.CV, cs.CL, cs.IT, cs.RO, \ldots.

\paragraph{Degeneracy at the uniform weight}
In a state in which the OS loss is $0$ at the uniform weight $\theta^0=\mathbf{1}/d$, the barycenter of the probability simplex is itself a global minimizer.
The upper bound $k_s$ on the number of papers and the upper bound $r$ on the depth were chosen so that this degeneracy is reduced.
The fraction of the states at which the OS loss is positive at the uniform weight (60 states of the episodes 0--2) is 57\% for $d=20$ and 88\% for $d=40$ for Cora ($r=3$), and 33\% for $d=20$ and 82\% for $d=40$ for ogbn-arxiv ($r=5$); for $d=100$, it was positive at all 20 states of the episode 0 in both experiments.
Setting the upper bound on the number of papers to $k_s=|E_s|+1$ increases the degeneracy (because the extra paper can recover the words of the agent), so we set $k_s=|E_s|$.
In the measured range, the states at which the OS loss is positive under the uniform weight account for at least 80\% for $d=40$ and for all states for $d=100$ (episode 0).

{
\subsection{Pseudocode of the evaluation protocol and of each method}
\label{subsec:appendix_pseudocode_methods}

This subsection gives the evaluation protocol as implemented and the pseudocode of each method.
In what follows, ``evaluating'' a weight $\theta$ means solving $v^n\gets x^*(\theta,s^n)$ for all states ($N$ oracle calls) and computing and recording $\ell^{\mathrm{OP}}(\theta)=\frac{1}{N}\sum_{n=1}^N\sum_{i=1}^d\max\{x^n_i-v^n_i,0\}$ (\cref{eq:def_opl}).
In every method, if $\ell^{\mathrm{OP}}\le0$ at an evaluation, the whole method stops at that point and returns the weight evaluated then (target-attainment stopping).
The solves $\tilde{y}^n\gets x^*(\theta;x^n,s^n)$ used only for logging (for recording $\bar{\ell}^{\optdex}$) are not counted as oracle calls.
Each method is evaluated at the weight with the smallest $\ell^{\mathrm{OP}}$ among the weights visited so far (for Sub-ONS and Sub-MetaGrad, the averaged iterates at the end of each epoch (one pass over the states); for OS-DC, all evaluation points of the inner loop, the outer loop, and subgradient finishing) (best-so-far; among weights with the same $\ell^{\mathrm{OP}}$, the one with the smaller empirical mean of the OS loss), and the empirical mean of the OS loss is also measured at the same weight.
Letting $T$ be the number of iterations of the projected subgradient method and $N$ the number of states, we impose the oracle-call budget $B:=2TN$ common to all methods and compare the values at the time when the number of oracle calls reaches $B$.
OS-PSG, OS-ONS(FB), and OS-MetaGrad(FB) make $2N$ oracle calls per iteration and are run for $T$ iterations, and Sub-PSG, Sub-ONS(FB), and Sub-MetaGrad(FB) make $N$ oracle calls per iteration and are run for $2T$ iterations.
Sub-ONS and Sub-MetaGrad make $2N$ oracle calls per epoch (one pass over the states), $N$ for the updates and $N$ for the evaluation of the averaged iterate, and are run for $T$ epochs.
The number of oracle calls of OS-DC differs from one outer iteration to another, and it may be run beyond the oracle-call budget $B$, but the values at the time when the number of oracle calls reaches $B$ are used for the comparison.
All methods stop as target attainment when $\ell^{\mathrm{OP}}=0$ is reached.

\begin{algorithm}[tbp]
    \caption{Common evaluation protocol (oracle-call budget $B$ and best-so-far)}
    \label{alg:exp_protocol}
    \begin{algorithmic}[1]
        \REQUIRE method, oracle-call budget $B=2TN$, attainment tolerance $\epsilon_{\mathrm{att}}=10^{-9}$
        \STATE Run the method and obtain the sequence of records $(c_i,\ell^{\mathrm{OP}}_i,\bar{\ell}^{\optdex}_i)_{i=0,1,\ldots}$ ($c_i$ is the cumulative number of oracle calls at the time of recording: the number of calls to $x^*(\cdot,s)$ for the 5 existing methods, and the sum of the numbers of calls to $x^*(\cdot,s)$ and $x^*(\cdot\,;x^\prime,s)$ for the 8 proposed methods)
        \FOR{$i=0,1,\ldots$}
            \STATE $(L_i,O_i)\gets$ the lexicographic minimum of $(\ell^{\mathrm{OP}}_{i^\prime},\bar{\ell}^{\optdex}_{i^\prime})_{i^\prime\le i}$ (an undefined $\bar{\ell}^{\optdex}$ is regarded as $+\infty$, and ties keep the earlier record)
        \ENDFOR
        \STATE $i_B\gets\max\{i\mid c_i\le B\}$
        \RETURN $L_{i_B}$, $O_{i_B}$, attainment $[L_{i_B}\le\epsilon_{\mathrm{att}}]$, and computation time (the maximum wall-clock time from the start of the method)
    \end{algorithmic}
\end{algorithm}

\begin{algorithm}[tbp]
    \caption{Projected subgradient method (Sub-PSG, OS-PSG, subgradient finishing)}
    \label{alg:psg}
    \begin{algorithmic}[1]
        \REQUIRE initial point $\theta^{0}$, step-size coefficient $\eta$, number of iterations $n_{\mathrm{it}}$, type $\in\{\mathrm{Sub},\mathrm{OS}\}$
        \FOR{$j=0,1,\ldots,n_{\mathrm{it}}-1$}
            \STATE For each $n$, $v^n\gets x^*(\theta^{j},s^n)$ and $\tilde{y}^n\gets x^*(\theta^{j};x^n,s^n)$ (counted for OS; used only for logging for Sub)
            \STATE For each $n$, $y^n\gets\tilde{y}^n$ (OS) or $y^n\gets x^n$ (Sub)
            \STATE $\bar{\ell}^{\optdex}_j\gets\frac{1}{N}\sum_{n=1}^N\max\{\langle\theta^{j},v^n-\tilde{y}^n\rangle,0\}$
            \STATE For each $n$, $\hat{v}^n\gets y^n$ if $\langle\theta^{j},v^n-y^n\rangle<0$, and $\hat{v}^n\gets v^n$ otherwise
            \STATE Record $\ell^{\mathrm{OP}}_j\gets\frac{1}{N}\sum_{n=1}^N\sum_{i=1}^d\max\{x^n_i-\hat{v}^n_i,0\}$
            \IF{$\ell^{\mathrm{OP}}_j\le0$}
                \RETURN $\theta^{j}$
            \ENDIF
            \STATE $g_j\gets\sum_{n=1}^N(v^n-y^n)$ (using $v^n$ before the replacement)
            \IF{$g_j=0$}
                \RETURN $\theta^{j}$
            \ENDIF
            \STATE $\theta^{j+1}\gets\Pi_{\Delta^{d-1}}\bigl(\theta^{j}-\eta\,(j+1)^{-1/2}g_j/\|g_j\|_2\bigr)$
        \ENDFOR
        \RETURN $\theta^{n_{\mathrm{it}}-1}$ (the last evaluated point)
    \end{algorithmic}
\end{algorithm}

In \Cref{alg:psg}, $(\theta^{0},\eta,n_{\mathrm{it}})$ is $(\mathbf{1}/d,0.1,2T)$ for Sub-PSG, $(\mathbf{1}/d,0.1,T)$ for OS-PSG, and $(\text{the output of DC optimization},10^{-3},T)$ for subgradient finishing (type OS).
In subgradient finishing, no stopping based on $\bar{\ell}^{\optdex}$ is performed, and the decay $(j+1)^{-1/2}$ of the step size is counted again from $j=0$.
These parameters are common to all five experiments, including the single-machine scheduling problem.

In what follows, let $\mathcal{H}_d$ be the set of $d\times d$ symmetric positive definite matrices, and for $H\in\mathcal{H}_d$ and $z\in\mathbb{R}^d$, let $\Pi^{H}_{\Delta^{d-1}}(z):=\argmin_{\theta\in\Delta^{d-1}}(\theta-z)^\top H(\theta-z)$ (the minimizer is unique since the objective function is strictly convex; it is solved as a quadratic program by Clarabel).
An online learner is given by a set $\mathbb{L}$ of internal states together with an initialization $\mathsf{I}\colon\Delta^{d-1}\to\mathbb{L}$, a prediction $\mathsf{P}\colon\mathbb{L}\to\Delta^{d-1}$, and an update $\mathsf{U}\colon\mathbb{L}\times(\mathbb{R}^d\setminus\{0\})\to\mathbb{L}$.

\begin{definition}[ONS]
\label{def:learner_ons}
Let $\gamma=1$ and $\epsilon_{\mathrm{reg}}=1$. The set of internal states is $\mathbb{L}_{\mathrm{ONS}}:=\Delta^{d-1}\times\mathcal{H}_d$, and for $\mathcal{L}=(\theta,H)\in\mathbb{L}_{\mathrm{ONS}}$ we define
\begin{align*}
    \mathsf{I}(\theta^{\mathrm{init}})&:=(\theta^{\mathrm{init}},\epsilon_{\mathrm{reg}}I), \qquad \mathsf{P}(\mathcal{L}):=\theta,\\
    \mathsf{U}(\mathcal{L},g)&:=\bigl(\Pi^{H^\prime}_{\Delta^{d-1}}(\theta-\gamma^{-1}(H^\prime)^{-1}g),\,H^\prime\bigr),\qquad H^\prime:=H+gg^\top.
\end{align*}
\end{definition}

\begin{definition}[MetaGrad]
\label{def:learner_metagrad}
Let the diameter be $D=\sqrt{2}$ and the total number of rounds be $n_{\mathrm{r}}$, and put $n_\eta:=\lceil\log_2(n_{\mathrm{r}}+1)/2\rceil+1$. The set of internal states is
\[
    \mathbb{L}_{\mathrm{MG}}:=\bigl(\mathbb{R}_{>0}^{n_\eta}\cup\{\bot\}\bigr)\times\Delta^{n_\eta-1}\times(\Delta^{d-1})^{n_\eta}\times(\mathcal{H}_d)^{n_\eta},
\]
whose elements we write as $\mathcal{L}=(\eta,\pi,(w^i)_{i=0}^{n_\eta-1},(H^i)_{i=0}^{n_\eta-1})$. Here $\bot$ is a symbol indicating that the learning rates $\eta$ are not yet set; they are determined from the norm of the first subgradient at the first update.
We define the initialization and the prediction by
\begin{gather*}
    \mathsf{I}(\theta^{\mathrm{init}}):=\bigl(\bot,\pi^{\mathrm{init}},(\theta^{\mathrm{init}})_{i},(I/D^2)_{i}\bigr),\qquad \pi^{\mathrm{init}}_i:=\frac{n_\eta+1}{n_\eta(i+1)(i+2)},\\
    \mathsf{P}(\mathcal{L}):=\begin{cases} w^0 & (\eta=\bot),\\ \dfrac{\sum_i\pi_i\eta_iw^i}{\sum_i\pi_i\eta_i} & (\text{otherwise}).\end{cases}
\end{gather*}
The update $\mathsf{U}(\mathcal{L},g)=(\eta^\prime,\pi^\prime,(w^{i\prime})_i,(H^{i\prime})_i)$ is defined, with $\theta:=\mathsf{P}(\mathcal{L})$, by
\begin{align*}
    \eta^\prime&:=\begin{cases}\bigl(2^{-i}/(5D\|g\|_2)\bigr)_{i=0}^{n_\eta-1} & (\eta=\bot),\\ \eta & (\text{otherwise}),\end{cases}\\
    \varsigma_i&:=\langle w^i-\theta,g\rangle,\qquad H^{i\prime}:=H^i+2(\eta^\prime_i)^2gg^\top,\\
    w^{i\prime}&:=\Pi^{H^{i\prime}}_{\Delta^{d-1}}\bigl(w^i-\eta^\prime_i(1+2\eta^\prime_i\varsigma_i)(H^{i\prime})^{-1}g\bigr),\\
    \pi^\prime_i&:=\frac{\pi_i\exp(-\eta^\prime_i\varsigma_i-(\eta^\prime_i\varsigma_i)^2)}{\sum_{j}\pi_j\exp(-\eta^\prime_j\varsigma_j-(\eta^\prime_j\varsigma_j)^2)}.
\end{align*}
\end{definition}

Since $w^i=\theta^{\mathrm{init}}$ for all $i$ in $\mathsf{I}(\theta^{\mathrm{init}})$, the prediction right after $\eta$ is set equals $w^0$ also by the general formula, so the case distinction in $\mathsf{P}$ causes no discontinuous change of the prediction.

\begin{algorithm}[tbp]
    \caption{Sub-ONS and Sub-MetaGrad (state-by-state update and online-to-batch conversion)}
    \label{alg:sub_online}
    \begin{algorithmic}[1]
        \REQUIRE learner $(\mathsf{I},\mathsf{P},\mathsf{U})$ (ONS: \Cref{def:learner_ons}, MetaGrad: \Cref{def:learner_metagrad}; $n_{\mathrm{r}}=TN$ for MetaGrad), $\theta^{\mathrm{init}}=\mathbf{1}/d$, number of epochs $T$
        \STATE $\mathcal{L}\gets\mathsf{I}(\theta^{\mathrm{init}})$, $S\gets0$, $t\gets0$
        \FOR{$e=0,1,\ldots,T-1$}
            \FOR{$n=1,\ldots,N$}
                \STATE $\theta^{t}\gets\mathsf{P}(\mathcal{L})$, $S\gets S+\theta^{t}$, $t\gets t+1$
                \STATE $g_t\gets x^*(\theta^{t},s^n)-x^n$ (one oracle call)
                \IF{$g_t\ne0$}
                    \STATE $\mathcal{L}\gets\mathsf{U}(\mathcal{L},g_t)$
                \ENDIF
            \ENDFOR
            \STATE Evaluate $\bar{\theta}^e\gets S/t$ ($N$ oracle calls), and record $\bar{\ell}^{\optdex}$ from $\tilde{y}^n\gets x^*(\bar{\theta}^e;x^n,s^n)$ (used only for logging)
        \ENDFOR
        \RETURN $\bar{\theta}^{T-1}$ (the last evaluated point)
    \end{algorithmic}
\end{algorithm}

\begin{algorithm}[tbp]
    \caption{Full-batch ONS and MetaGrad (Sub-ONS(FB), Sub-MetaGrad(FB), OS-ONS(FB), OS-MetaGrad(FB))}
    \label{alg:fb_online}
    \begin{algorithmic}[1]
        \REQUIRE learner $(\mathsf{I},\mathsf{P},\mathsf{U})$ (ONS: \Cref{def:learner_ons}, MetaGrad: \Cref{def:learner_metagrad}; $n_{\mathrm{r}}=n_{\mathrm{it}}$ for MetaGrad), $\theta^{\mathrm{init}}=\mathbf{1}/d$, number of iterations $n_{\mathrm{it}}$ (OS: $T$, Sub: $2T$)
        \STATE $\mathcal{L}\gets\mathsf{I}(\theta^{\mathrm{init}})$
        \FOR{$j=0,1,\ldots,n_{\mathrm{it}}-1$}
            \STATE $\theta^{j}\gets\mathsf{P}(\mathcal{L})$
            \STATE For each $n$, $v^n\gets x^*(\theta^{j},s^n)$ and $\tilde{y}^n\gets x^*(\theta^{j};x^n,s^n)$ (counted for OS; used only for logging for Sub)
            \STATE For each $n$, let $y^n\gets\tilde{y}^n$ (OS) or $y^n\gets x^n$ (Sub), and set $v^n\gets y^n$ if $\langle\theta^{j},v^n-y^n\rangle<0$
            \STATE Record $\ell^{\mathrm{OP}}_j$ and $\bar{\ell}^{\optdex}_j$ with $v^n$ after the replacement
            \IF{$\ell^{\mathrm{OP}}_j\le0$}
                \RETURN $\theta^{j}$
            \ENDIF
            \STATE $g_j\gets\frac{1}{N}\sum_{n=1}^N(v^n-y^n)$
            \IF{$g_j=0$}
                \RETURN $\theta^{j}$
            \ENDIF
            \STATE $\mathcal{L}\gets\mathsf{U}(\mathcal{L},g_j)$
        \ENDFOR
        \RETURN $\theta^{n_{\mathrm{it}}-1}$ (the last evaluated point)
    \end{algorithmic}
\end{algorithm}

\begin{algorithm}[tbp]
    \caption{Inner loop of OS-DC[ONS] and OS-DC[MetaGrad] (replacing line \ref{line:dc_inner_call} of \Cref{alg:dc_outer})}
    \label{alg:dc_inner_online}
    \begin{algorithmic}[1]
        \REQUIRE branch $y^\prime$, outer point $\theta^k$, learner $(\mathsf{I},\mathsf{P},\mathsf{U})$ (ONS: \Cref{def:learner_ons}, MetaGrad: \Cref{def:learner_metagrad}; $D=\sqrt{2}$ and $n_{\mathrm{r}}=KJ$ for MetaGrad), internal state of the learner $\mathcal{L}$ (undefined for $k=0$), $\epsilon_{\mathrm{in}}=10^{-6}$, $J=20$, feasibility tolerance $\epsilon_{\mathrm{feas}}=10^{-9}$
        \STATE $f(\theta):=\bar{\ell}^{\sub}(\theta;y^\prime)$, $g(\theta):=\frac{1}{N}\sum_{n=1}^N(x^*(\theta,s^n)-y^{\prime n})$ (one evaluation is shared for the same $\theta$)
        \STATE $f^\star\gets+\infty$, $\mathrm{converged}\gets\mathrm{FALSE}$
        \IF{$\mathcal{L}$ is undefined}
            \IF{the learner is ONS}
                \STATE $\mathcal{L}\gets\mathsf{I}(\theta^k)$
            \ELSIF{the learner is MetaGrad}
                \STATE Evaluate $\theta^k$ ($N$ oracle calls), and set $f^\star\gets f(\theta^k)$, $\hat{\theta}\gets\theta^k$, and $G\gets\|g(\theta^k)\|_2$
                \IF{$G=0$}
                    \RETURN $(\hat{\theta},\mathrm{TRUE},\mathcal{L})$ ($\mathcal{L}$ remains undefined)
                \ENDIF
                \STATE Set $\mathcal{L}\gets\mathsf{I}(\theta^k)$ and replace the $\eta$ component of $\mathcal{L}$ by $\bigl(2^{-i}/(5DG)\bigr)_{i=0}^{n_\eta-1}$
            \ENDIF
        \ELSE
            \STATE Use the input $\mathcal{L}$ as it is ($\theta^k$ is not used)
        \ENDIF
        \STATE $\Sigma\gets0$
        \FOR{$t=1,\ldots,J$}
            \STATE Evaluate $\theta^{t}\gets\mathsf{P}(\mathcal{L})$ ($N$ oracle calls; at $t=1$ of the first outer iteration of MetaGrad, the prediction is mathematically equal to $\theta^k$, but it is evaluated again if the bit patterns do not coincide because of floating-point rounding)
            \STATE If $f(\theta^{t})<f^\star$, then $f^\star\gets f(\theta^{t})$ and $\hat{\theta}\gets\theta^{t}$. $\Sigma\gets\Sigma+\theta^{t}$
            \IF{$f^\star\le\epsilon_{\mathrm{in}}$ or $g(\theta^{t})=0$}
                \STATE $\mathrm{converged}\gets\mathrm{TRUE}$, \textbf{break}
            \ENDIF
            \STATE $\mathcal{L}\gets\mathsf{U}(\mathcal{L},g(\theta^{t}))$
        \ENDFOR
        \IF{$\mathrm{converged}=\mathrm{FALSE}$ and $\bar{\theta}:=\Sigma/t\in\Theta$ (tolerance $\epsilon_{\mathrm{feas}}$)}
            \STATE Evaluate $\bar{\theta}$ ($N$ oracle calls), and set $\hat{\theta}\gets\bar{\theta}$ if $f(\bar{\theta})<f^\star$
        \ENDIF
        \RETURN $(\hat{\theta},\mathrm{converged},\mathcal{L})$
    \end{algorithmic}
\end{algorithm}

\begin{algorithm}[p]
    \caption{The whole procedure of OS-DC[$\cdot$] (+SGF)}
    \label{alg:osdc_full}
    \begin{algorithmic}[1]
        \REQUIRE $K=\lfloor2T/20\rfloor$, $J=20$, $\epsilon_{\mathrm{tie}}=10^{-8}$, $\epsilon_{\mathrm{gap}}=10^{-6}$, inner solver, whether to perform subgradient finishing
        \STATE $\theta\gets\Pi_{\Delta^{d-1}}(\mathbf{1}/d)$, $y^{\prime n}\gets x^*(\theta;x^n,s^n)$ for each $n$ ($N$ oracle calls), $\mathcal{V}\gets\emptyset$, and let the internal state of the learner $\mathcal{L}$ be undefined (ONS, MetaGrad)
        \FOR{$k=0,1,\ldots,K-1$}
            \STATE Run the inner solver (CP, LV+CPF, PB+CPF: \Cref{alg:dc_inner} (using inherited cuts from $\mathcal{V}$); ONS, MetaGrad: \Cref{alg:dc_inner_online} (inputs $\theta^k:=\theta$ and $\mathcal{L}$; the output $\mathcal{L}$ replaces $\mathcal{L}$)) from the initial point $\theta$, and obtain $(\theta^\prime,\mathrm{converged})$
            \STATE Evaluate $\theta\gets\Pi_{\Delta^{d-1}}(\theta^\prime)$ to obtain $v^n$ and add them to $\mathcal{V}$, and set $y^{\prime\prime n}\gets x^*(\theta;x^n,s^n)$ for each $n$ ($2N$ oracle calls in total)
            \STATE $\ell_k\gets\frac{1}{N}\sum_{n=1}^N\langle\theta,v^n-y^{\prime\prime n}\rangle$, $\Delta_k\gets\frac{1}{N}\sum_{n=1}^N\langle\theta,y^{\prime\prime n}-y^{\prime n}\rangle$, $\theta^{(k)}\gets\theta$
            \IF{$\ell_k\le\epsilon_{\mathrm{gap}}$, or ($\Delta_k\le\epsilon_{\mathrm{tie}}$ and $\mathrm{converged}=\mathrm{TRUE}$)}
                \STATE \textbf{break}
            \ENDIF
            \STATE $y^\prime\gets y^{\prime\prime}$
        \ENDFOR
        \STATE $\theta^{\mathrm{DC}}\gets\theta^{(k^\star)}$, where $k^\star$ is the smallest index minimizing $\ell_k$ (regardless of the reason for stopping)
        \IF{subgradient finishing is performed}
            \STATE Run \Cref{alg:psg} with type OS and $(\theta^{0},\eta,n_{\mathrm{it}})=(\theta^{\mathrm{DC}},10^{-3},T)$, and let $\theta^{\mathrm{SGF}}$ be its output
            \RETURN $\theta^{\mathrm{SGF}}$
        \ENDIF
        \RETURN $\theta^{\mathrm{DC}}$
    \end{algorithmic}
\end{algorithm}

In OS-DC, $\bar{\ell}^{\optdex}$ is recorded at all evaluation points of the inner and outer loops by solves used only for logging.
The pool $\mathcal{V}$ of inherited cuts is shared within a single run of DC optimization.

\subsection{Pseudocode for instance generation}
\label{subsec:appendix_pseudocode_instances}

For each experiment, an instance is generated once for each pair of an episode and a dimension $d$, and the same $\{(s^n,x^n)\}_{n=1}^N$ is given to all methods.
All random numbers in the following pseudocode are drawn in order from a single generator of \texttt{numpy.random.}\allowbreak\texttt{default\_rng}. Its seed is $(20260928,e,d)$ for the knapsack problem and the single-machine scheduling problem (\Cref{alg:gen_knapsack,alg:gen_scheduling}) and $(20260928,e)$ for the weighted multi-cover problem (\Cref{alg:gen_cora,alg:gen_arxiv}), where $e$ is the episode index (\Cref{subsec:appendix_environment}).

\paragraph{Knapsack problem}
For a state $s=(w,c,u)$, we have $x^*(\theta,s)\in\argmax\{\langle\theta,x\rangle\mid x\in\mathbb{Z}^d,\ \langle w,x\rangle\le c,\ 0\le x\le u\}$ and $x^*(\theta;x^n,s)\in\argmax\{\langle\theta,x\rangle\mid x\in\mathbb{Z}^d,\ \langle w,x\rangle\le c,\ x^n\le x\le u\}$, both of which are solved exactly by dynamic programming for the bounded knapsack problem (for the latter, we solve the problem in $x-x^n$ with capacity $c-\langle w,x^n\rangle$ and upper bound $u-x^n$).

\begin{algorithm}[tbp]
    \caption{Instance generation for the knapsack problem (one episode and dimension $d$)}
    \label{alg:gen_knapsack}
    \begin{algorithmic}[1]
        \REQUIRE episode index $e$, dimension $d$, number of states $N=20$
        \STATE $\theta^*\sim\mathrm{Dir}(\mathbf{1})$
        \FOR{$n=1,\ldots,N$}
            \STATE $w^n_i\sim\mathrm{U}\{1,\ldots,9\}$, $u^n_i\sim\mathrm{U}\{1,\ldots,4\}$ ($i=1,\ldots,d$), $c^n\sim\mathrm{U}\{1,\ldots,19\}$ (all independent), $s^n\gets(w^n,c^n,u^n)$
        \ENDFOR
        \FOR{$n=1,\ldots,N$}
            \STATE $x^n\gets x^*(\theta^*,s^n)$
            \FOR{$i=1,\ldots,d$}
                \STATE with probability $0.2$, $x^n_i\gets\max\{x^n_i-1,0\}$
            \ENDFOR
        \ENDFOR
        \RETURN $\{(s^n,x^n)\}_{n=1}^N$
    \end{algorithmic}
\end{algorithm}

Since $0\le x^n\le x^*(\theta^*,s^n)$ and $w\ge0$, $x^n$ is feasible, and $x^*(\theta^*,s^n)$ is a feasible solution of the dominance-constrained problem.

\paragraph{Single-machine scheduling problem}
For a state $s=(r,p)$ and an upper bound $\bar{C}\in\mathbb{Z}^d$ on the completion times, we solve
\begin{align*}
    \min_{S,C\in\mathbb{Z}^d,\ z\in\{0,1\}^{\{(j,k):j<k\}}}\ & \sum_{j=1}^d\theta_jC_j\\
    \text{s.t.}\ & C_j=S_j+p_j,\quad r_j\le S_j,\quad C_j\le\bar{C}_j\quad(j=1,\ldots,d),\\
    & S_k\ge S_j+p_j-M_{\mathrm{big}}(1-z_{jk})\quad(1\le j<k\le d),\\
    & S_j\ge S_k+p_k-M_{\mathrm{big}}z_{jk}\quad(1\le j<k\le d)
\end{align*}
with SCIP and return the action $x=-C$.
For $x^*(\theta,s)$ we set $\bar{C}_j=\max_kr_k+\sum_kp_k$, for $x^*(\theta;x^n,s)$ we set $\bar{C}_j=-x^n_j$, and $M_{\mathrm{big}}=2\max_j(\bar{C}_j+p_j)$.

\begin{algorithm}[tbp]
    \caption{Instance generation for the single-machine scheduling problem (one episode and dimension $d$)}
    \label{alg:gen_scheduling}
    \begin{algorithmic}[1]
        \REQUIRE episode index $e$, number of jobs $d$, number of states $N=10$
        \STATE $\theta^*\sim\mathrm{Dir}(\mathbf{1})$
        \FOR{$n=1,\ldots,N$}
            \STATE $r^n_j\sim\mathrm{U}\{0,\ldots,14\}$, $p^n_j\sim\mathrm{U}\{1,\ldots,4\}$ ($j=1,\ldots,d$, all independent), $s^n\gets(r^n,p^n)$
        \ENDFOR
        \FOR{$n=1,\ldots,N$}
            \STATE $x^n\gets x^*(\theta^*,s^n)$
            \STATE $\tau^n_1\sim\mathrm{U}\{8,\ldots,12\}$, $\tau^n_2\sim\mathrm{U}\{18,\ldots,22\}$
            \FOR{$j=1,\ldots,d$}
                \STATE if $x^n_j\le-\tau^n_1$, then $x^n_j\gets x^n_j-1$
                \STATE if $x^n_j\le-\tau^n_2$, then $x^n_j\gets x^n_j-1$
            \ENDFOR
        \ENDFOR
        \RETURN $\{(s^n,x^n)\}_{n=1}^N$
    \end{algorithmic}
\end{algorithm}

Let $C^*_j$ be the completion times of $x^*(\theta^*,s^n)$. Then the completion times of the agent are $f(C^*_j):=C^*_j+[C^*_j\ge\tau^n_1]+[C^*_j\ge\tau^n_2-1]$. Since $f(c+1)-f(c)\ge1$, $f$ preserves the processing order and the release-time constraints, so $x^n$ is feasible and $x^*(\theta^*,s^n)$ is a feasible solution of the dominance-constrained problem.

\paragraph{Weighted multi-cover problem}
The word ranking common to both experiments is given in \Cref{alg:interleave}, the word scores for Cora and ogbn-arxiv in \Cref{alg:cora_words,alg:arxiv_words}, and the instance generation in \Cref{alg:gen_cora,alg:gen_arxiv}.
Here, $\mathrm{argsort}^{\downarrow}$ denotes a stable sort in descending order in which ties are ordered by increasing index.
The stop words for ogbn-arxiv are 164 words based on the lists of NLTK and scikit-learn.
The forward problem is solved with SCIP with $x\in\{0,\ldots,r\}^d$ ($x^n\le x\le r$ in the dominance-constrained version), and returns the componentwise $\min(r,A^\top z)$ of the selected papers $z$.

\begin{algorithm}[tbp]
    \caption{Word ranking that alternates between fields}
    \label{alg:interleave}
    \begin{algorithmic}[1]
        \REQUIRE word scores $S_c(j)$ for each field ($j=0,\ldots,n_{\mathrm{w}}-1$), an order of fields $(c_1,\ldots,c_{n_{\mathrm{f}}})$, upper limit on the number of words $L=200$
        \STATE For each field $c$, $\pi_c\gets\mathrm{argsort}^{\downarrow}_jS_c(j)$, $p_c\gets0$; $\mathcal{R}\gets()$
        \WHILE{$|\mathcal{R}|<L$ and a word was added to $\mathcal{R}$ in the previous round}
            \FOR{$l=1,\ldots,n_{\mathrm{f}}$}
                \STATE $c\gets c_l$, and $p_c\gets p_c+1$ while $\pi_c(p_c)\in\mathcal{R}$
                \IF{$p_c<n_{\mathrm{w}}$ and $|\mathcal{R}|<L$}
                    \STATE append $\pi_c(p_c)$ to the end of $\mathcal{R}$, and $p_c\gets p_c+1$
                \ENDIF
            \ENDFOR
        \ENDWHILE
        \RETURN $\mathcal{R}$ (the words for dimension $d$ are the first $d$ words)
    \end{algorithmic}
\end{algorithm}

\begin{algorithm}[tbp]
    \caption{Word scores for Cora}
    \label{alg:cora_words}
    \begin{algorithmic}[1]
        \REQUIRE Bag-of-Words $B^{\mathrm{BoW}}\in\{0,1\}^{2708\times1433}$, fields $c(i)\in\{0,\ldots,6\}$
        \STATE For each field $c$, $n_c\gets|\{i:c(i)=c\}|$, $S_c(j)\gets\frac{1}{n_c}\sum_{i:c(i)=c}B^{\mathrm{BoW}}_{ij}$ (occurrence rate)
        \STATE Sort the fields in descending order of $n_c$ (ties by increasing field index), apply \Cref{alg:interleave}, and let $\mathcal{R}$ be its output
        \RETURN $\mathcal{R}$
    \end{algorithmic}
\end{algorithm}

\begin{algorithm}[tbp]
    \caption{Word scores for ogbn-arxiv}
    \label{alg:arxiv_words}
    \begin{algorithmic}[1]
        \REQUIRE titles and abstracts of all papers $i=0,\ldots,n_{\mathrm{tot}}-1$ ($n_{\mathrm{tot}}=169343$, including papers from 2020), fields $c(i)\in\{0,\ldots,39\}$, a set of stop words $\mathcal{W}_{\mathrm{stop}}$
        \STATE For each paper $i$, concatenate the title and the abstract and convert them to lowercase; let $W_i$ be the set of words, extracted as maximal runs of alphabetic characters, that have length at least 3 and are not in $\mathcal{W}_{\mathrm{stop}}$
        \STATE Sort the words that appear in at least 20 papers in lexicographic order, and denote them by $w_0,\ldots,w_{n_{\mathrm{w}}-1}$
        \STATE For each field $c$, $n_c\gets|\{i:c(i)=c\}|$ (all years), $S_c(j)\gets\frac{|\{i:c(i)=c,\ w_j\in W_i\}|}{n_c}-\frac{|\{i:w_j\in W_i\}|}{n_{\mathrm{tot}}}$ (distinctiveness)
        \STATE Sort the fields in descending order of $n_c$ (ties by increasing field index), apply \Cref{alg:interleave}, and let $\mathcal{R}$ be its output
        \RETURN $\mathcal{R}$
    \end{algorithmic}
\end{algorithm}

\begin{algorithm}[tbp]
    \caption{Instance generation for Cora (episode $e$, dimension $d$)}
    \label{alg:gen_cora}
    \begin{algorithmic}[1]
        \REQUIRE neighborhoods $\mathcal{N}(\cdot)$ of the citation graph made undirected with self-loops removed, word ranking $\mathcal{R}$, $N=20$, $r=3$
        \STATE Let $(C_0,\ldots,C_{n_{\mathrm{cand}}-1})$ be the sequence of papers $v$ with $3\le|\mathcal{N}(v)|\le10$, sorted in increasing order of paper index
        \STATE Choose a subset from the uniform distribution over all $N$-element subsets of $\{0,\ldots,n_{\mathrm{cand}}-1\}$, and denote its elements in increasing order by $I_1<\cdots<I_N$
        \FOR{$n=1,\ldots,N$}
            \STATE $v\gets C_{I_n}$, $E_{s^n}\gets\mathcal{N}(v)$, $P_{s^n}\gets\bigl(E_{s^n}\cup\bigcup_{u\in E_{s^n}}\mathcal{N}(u)\bigr)\setminus\{v\}$
            \STATE $A_{ij}\gets B^{\mathrm{BoW}}_{i,\mathcal{R}_j}$ ($i\in P_{s^n}$, $j=1,\ldots,d$), $k_{s^n}\gets|E_{s^n}|$, $x^n_j\gets\min\bigl(r,\sum_{i\in E_{s^n}}A_{ij}\bigr)$
        \ENDFOR
        \RETURN $\{(s^n,x^n)\}_{n=1}^N$
    \end{algorithmic}
\end{algorithm}

\begin{algorithm}[tbp]
    \caption{Instance generation for ogbn-arxiv (episode $e$, dimension $d$)}
    \label{alg:gen_arxiv}
    \begin{algorithmic}[1]
        \REQUIRE set of directed edges $\mathcal{E}$ ($(u,w)\in\mathcal{E}$ means that $u$ cites $w$), publication year $t(\cdot)$, field $c(\cdot)$, word ranking $\mathcal{R}$, $N=20$, $r=5$
        \STATE $\mathrm{Ref}(u):=\{w:(u,w)\in\mathcal{E},\ t(w)\le2019,\ w\ne u\}$
        \STATE For each field $c$, let $G_c$ be the sequence of papers $v$ with $t(v)=2019$, $c(v)=c$, and $5\le|\mathrm{Ref}(v)|\le8$, sorted in increasing order of paper index, and let $m_c\gets|\{v:t(v)=2019,\ c(v)=c\}|$
        \STATE Sort the fields in descending order of $m_c$ (ties by increasing field index), keep only those with $|G_c|\ge N$, let $c$ be the $e$-th field of this sequence (0-indexed), and set $(G_0,\ldots,G_{n_{\mathrm{cand}}-1})\gets G_c$
        \STATE Choose a subset from the uniform distribution over all $N$-element subsets of $\{0,\ldots,n_{\mathrm{cand}}-1\}$, and denote its elements in increasing order by $I_1<\cdots<I_N$
        \FOR{$n=1,\ldots,N$}
            \STATE $v\gets G_{I_n}$, $E_{s^n}\gets\mathrm{Ref}(v)$, $P_{s^n}\gets\bigl(E_{s^n}\cup\bigcup_{w\in E_{s^n}}\mathrm{Ref}(w)\bigr)\setminus\{v\}$
            \STATE $A_{ij}\gets[\text{word }\mathcal{R}_j\text{ appears in paper }i]$ ($i\in P_{s^n}$, $j=1,\ldots,d$), $k_{s^n}\gets|E_{s^n}|$, $x^n_j\gets\min\bigl(r,\sum_{i\in E_{s^n}}A_{ij}\bigr)$
        \ENDFOR
        \RETURN $\{(s^n,x^n)\}_{n=1}^N$
    \end{algorithmic}
\end{algorithm}
} 

\subsection{Results for each dimension on the knapsack problem and the weighted multi-cover problem}
\label{subsec:appendix_main_tables}

For the knapsack problem and the weighted multi-cover problem, \Cref{tab:knapsack_median,tab:multicover_main} show, for each dimension, the median of $\ell^{\mathrm{OP}}$ of the seven methods at $B$ oracle calls.

\Cref{fig:opl_vs_calls_knapsack} shows the evolution of the worst 10\% of $\ell^{\mathrm{OP}}$ (best-so-far) on the knapsack problem.

\begin{figure}[tbp]
    \centering
    \includegraphics[width=\linewidth]{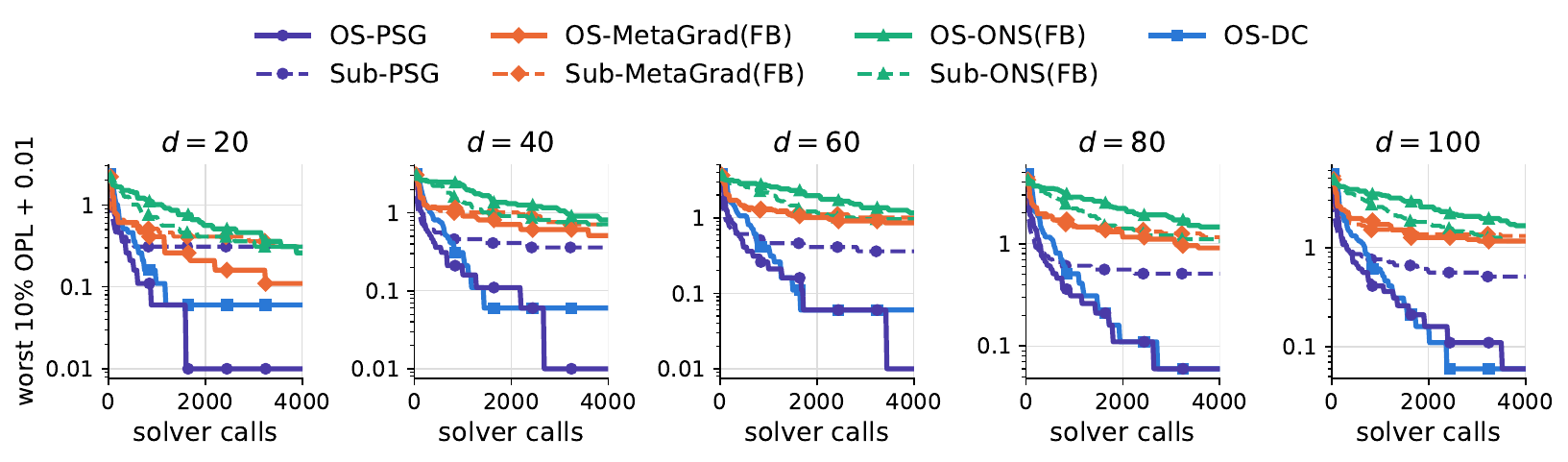}
    \caption{Relation between the worst 10\% of $\ell^{\mathrm{OP}}$ (best-so-far) and the number of oracle calls on the knapsack problem. The vertical axis and the lines are drawn as in \Cref{fig:opl_vs_calls}.}
    \label{fig:opl_vs_calls_knapsack}
\end{figure}

\begin{table}[tbp]
    \centering
    \caption{Knapsack problem: median of $\ell^{\mathrm{OP}}$ at $B=4000$ oracle calls (the attainment rate [\%] of $\ell^{\mathrm{OP}}=0$ in parentheses). 50 episodes for each $d$. For each optimization method, the pair of the OS loss (proposed) and the suboptimality loss (existing) is shown, and the one with the smaller median in each pair (with the higher attainment rate in case of a tie) is shown in bold.}
    \label{tab:knapsack_median}
    \setlength{\tabcolsep}{4pt}
    \begin{tabular}{l|ccccc}
        \hline
        Method & $d=20$ & $d=40$ & $d=60$ & $d=80$ & $d=100$ \\
        \hline
        OS-PSG & \textbf{0.00 (98)} & \textbf{0.00 (96)} & \textbf{0.00 (92)} & \textbf{0.00 (90)} & \textbf{0.00 (76)} \\
        Sub-PSG & 0.15 (16) & 0.20 (0) & 0.23 (0) & 0.28 (0) & 0.35 (0) \\
        \hline
        OS-MetaGrad(FB) & \textbf{0.00 (86)} & \textbf{0.00 (52)} & \textbf{0.08 (46)} & \textbf{0.35 (8)} & \textbf{0.38 (12)} \\
        Sub-MetaGrad(FB) & 0.15 (6) & 0.40 (0) & 0.65 (0) & 0.80 (0) & 0.95 (0) \\
        \hline
        OS-ONS(FB) & \textbf{0.10 (34)} & \textbf{0.40 (2)} & 0.70 (0) & 0.95 (0) & 1.15 (0) \\
        Sub-ONS(FB) & 0.15 (0) & 0.40 (0) & \textbf{0.62 (0)} & \textbf{0.70 (0)} & \textbf{0.88 (0)} \\
        \hline
        OS-DC & 0.00 (90) & 0.00 (88) & 0.00 (86) & 0.00 (82) & 0.00 (84) \\
        \hline
    \end{tabular}
\end{table}

\begin{table}[tbp]
    \centering
    \caption{Weighted multi-cover problem (left: Cora, right: ogbn-arxiv): median of $\ell^{\mathrm{OP}}$ at $B=4000$ oracle calls. 20 episodes for each $d$. For each optimization method, the pair of the OS loss (proposed) and the suboptimality loss (existing) is shown, and the one with the smaller median in each pair is shown in bold.}
    \label{tab:multicover_main}
    \setlength{\tabcolsep}{1pt}
    \begin{tabular}{l|ccccc|ccccc}
        \hline
        & \multicolumn{5}{c|}{Cora ($d$)} & \multicolumn{5}{c}{ogbn-arxiv ($d$)} \\
        Method & 20 & 40 & 60 & 80 & 100 & 20 & 40 & 60 & 80 & 100 \\
        \hline
        OS-PSG & \textbf{0.78} & \textbf{2.58} & \textbf{4.40} & \textbf{5.72} & \textbf{6.82} & \textbf{0.57} & \textbf{2.40} & \textbf{4.30} & \textbf{6.50} & \textbf{9.30} \\
        Sub-PSG & 1.25 & 3.00 & 4.95 & 6.25 & 7.47 & 1.12 & 2.95 & 4.85 & 7.22 & 9.40 \\
        \hline
        OS-MetaGrad(FB) & \textbf{0.97} & \textbf{2.92} & 5.25 & \textbf{6.75} & \textbf{7.80} & \textbf{0.78} & \textbf{2.73} & \textbf{4.65} & \textbf{7.18} & 9.55 \\
        Sub-MetaGrad(FB) & 1.27 & 3.23 & \textbf{5.18} & 7.05 & 8.15 & 1.15 & 2.97 & 4.85 & 7.20 & \textbf{9.47} \\
        \hline
        OS-ONS(FB) & \textbf{1.00} & \textbf{3.05} & 5.22 & 6.88 & \textbf{8.00} & \textbf{0.82} & \textbf{2.67} & \textbf{4.68} & \textbf{6.92} & 9.43 \\
        Sub-ONS(FB) & 1.33 & 3.23 & \textbf{5.20} & \textbf{6.82} & 8.38 & 1.15 & 2.97 & 4.85 & 7.00 & 9.43 \\
        \hline
        OS-DC & 0.82 & 2.60 & 4.40 & 5.83 & 7.30 & 0.65 & 2.52 & 4.53 & 7.00 & 9.35 \\
        \hline
    \end{tabular}
\end{table}

\subsection{Results on the single-machine scheduling problem and the movie selection problem with viewing costs}
\label{subsec:appendix_other_results}

\Cref{tab:onemachine_main,tab:timecost_main} show, for each dimension, $\ell^{\mathrm{OP}}$ at $B$ oracle calls for the three pairs that differ only in the loss function and for OS-DC, and \Cref{fig:opl_vs_calls_other} shows its trajectory on the worst 10\% of the episodes.

\begin{table}[tbp]
    \centering
    \caption{Single-machine scheduling problem: median of $\ell^{\mathrm{OP}}$ at $B=2000$ oracle calls (the attainment rate [\%] of $\ell^{\mathrm{OP}}=0$ in parentheses). 25 episodes for each $d$. For each optimization method, the pair of the OS loss (proposed) and the suboptimality loss (existing) is shown, and the one with the smaller median in each pair (with the higher attainment rate in case of a tie) is shown in bold.}
    \label{tab:onemachine_main}
    \begin{tabular}{l|ccc}
        \hline
        Method & $d=4$ & $d=6$ & $d=8$ \\
        \hline
        OS-PSG & \textbf{0.00 (100)} & \textbf{0.00 (100)} & \textbf{0.00 (96)} \\
        Sub-PSG & 0.00 (88) & 0.00 (72) & 0.10 (44) \\
        \hline
        OS-MetaGrad(FB) & \textbf{0.00 (100)} & \textbf{0.00 (100)} & \textbf{0.00 (96)} \\
        Sub-MetaGrad(FB) & 0.00 (84) & 0.00 (76) & 0.30 (28) \\
        \hline
        OS-ONS(FB) & \textbf{0.00 (100)} & \textbf{0.00 (100)} & \textbf{0.00 (84)} \\
        Sub-ONS(FB) & 0.00 (84) & 0.00 (52) & 0.30 (40) \\
        \hline
        OS-DC & 0.00 (100) & 0.00 (96) & 0.00 (96) \\
        \hline
    \end{tabular}
\end{table}

\begin{table}[tbp]
    \centering
    \caption{Movie selection problem with viewing costs: median of $\ell^{\mathrm{OP}}$ at $B=4000$ oracle calls. 20 episodes for each $d$. For each optimization method, the pair of the OS loss (proposed) and the suboptimality loss (existing) is shown, and the one with the smaller median in each pair is shown in bold.}
    \label{tab:timecost_main}
    \setlength{\tabcolsep}{1pt}
    \begin{tabular}{l|ccccc}
        \hline
        Method & $d=20$ & $d=40$ & $d=60$ & $d=80$ & $d=100$ \\
        \hline
        OS-PSG & \textbf{\twoline{0.72}{[0.59,1.01]}} & \textbf{\twoline{2.52}{[2.02,2.90]}} & \textbf{\twoline{4.05}{[3.65,4.22]}} & \textbf{\twoline{5.18}{[4.57,5.71]}} & \textbf{\twoline{6.50}{[5.65,7.09]}} \\
        Sub-PSG & \twoline{0.95}{[0.70,1.11]} & \twoline{2.85}{[2.44,3.17]} & \twoline{4.60}{[4.40,5.11]} & \twoline{5.80}{[5.39,6.53]} & \twoline{7.57}{[6.84,7.98]} \\
        \hline
        OS-MetaGrad(FB) & \textbf{\twoline{0.95}{[0.79,1.10]}} & \textbf{\twoline{2.70}{[2.60,3.10]}} & \textbf{\twoline{4.50}{[4.09,4.79]}} & \textbf{\twoline{5.65}{[5.42,6.19]}} & \textbf{\twoline{7.10}{[6.50,7.75]}} \\
        Sub-MetaGrad(FB) & \twoline{1.30}{[1.12,1.52]} & \twoline{3.25}{[2.91,3.42]} & \twoline{4.88}{[4.67,5.20]} & \twoline{6.18}{[5.85,6.53]} & \twoline{7.90}{[6.66,8.49]} \\
        \hline
        OS-ONS(FB) & \textbf{\twoline{1.10}{[0.85,1.30]}} & \textbf{\twoline{3.15}{[2.91,3.31]}} & \twoline{4.95}{[4.49,5.16]} & \twoline{6.45}{[6.01,6.76]} & \twoline{8.00}{[7.08,8.49]} \\
        Sub-ONS(FB) & \twoline{1.35}{[1.19,1.51]} & \twoline{3.27}{[2.95,3.40]} & \twoline{4.95}{[4.67,5.15]} & \textbf{\twoline{6.17}{[5.85,6.65]}} & \textbf{\twoline{7.72}{[6.96,8.26]}} \\
        \hline
        OS-DC & \twoline{0.62}{[0.50,0.85]} & \twoline{2.30}{[2.05,2.96]} & \twoline{4.20}{[3.89,4.36]} & \twoline{5.28}{[4.99,5.81]} & \twoline{6.83}{[6.12,7.15]} \\
        \hline
    \end{tabular}
\end{table}

\begin{figure}[tbp]
    \centering
    \includegraphics[width=0.75\linewidth]{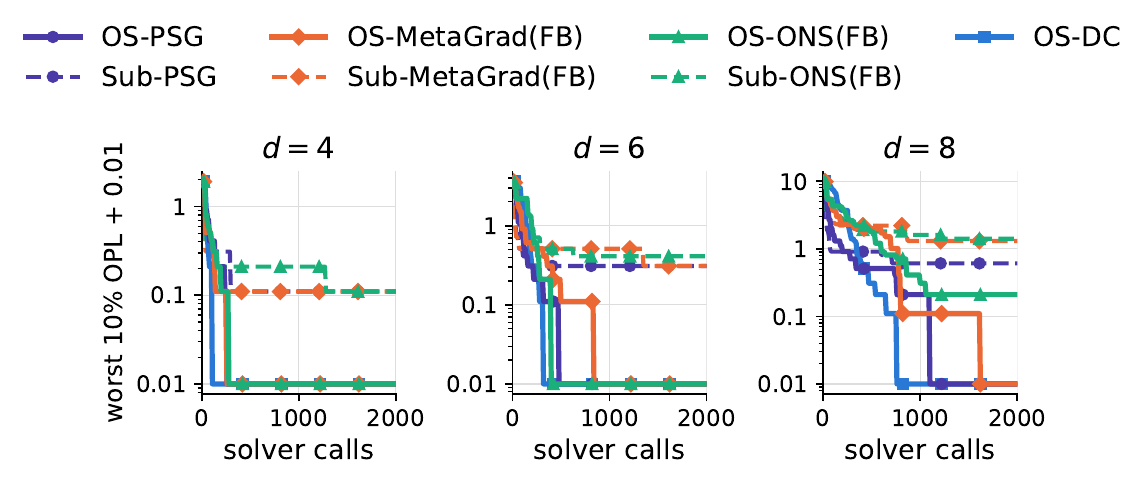}\\
    (a) Single-machine scheduling problem\\[2mm]
    \includegraphics[width=\linewidth]{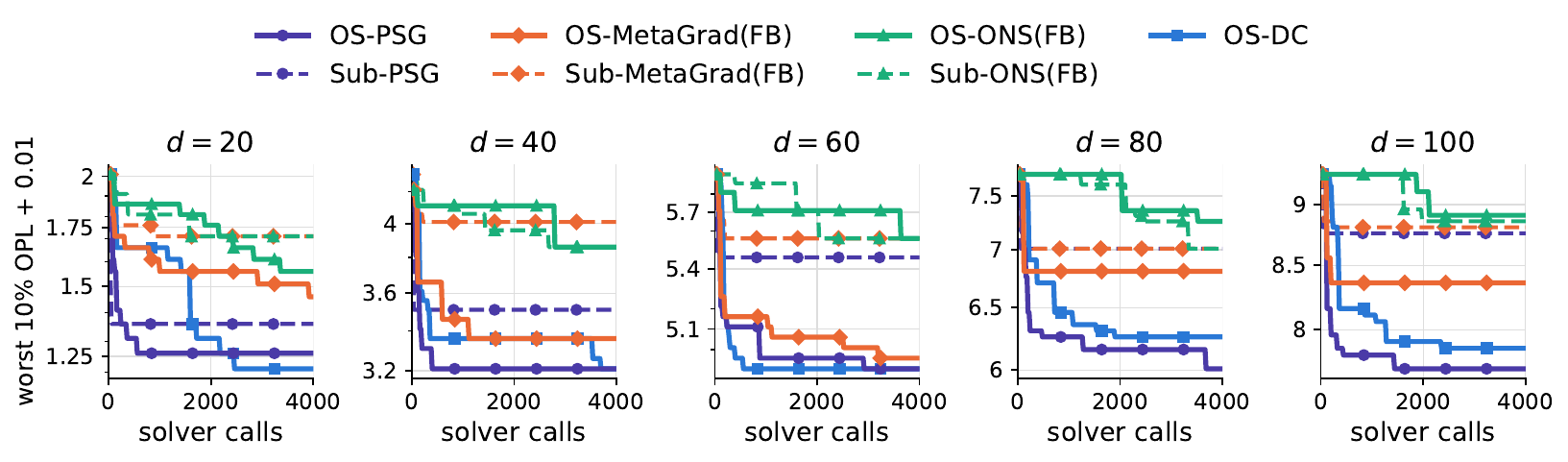}\\
    (b) Movie selection problem with viewing costs
    \caption{Relation between the worst 10\% of $\ell^{\mathrm{OP}}$ (best-so-far) and the number of oracle calls on the single-machine scheduling problem and the movie selection problem with viewing costs. The vertical axis and the lines are drawn as in \Cref{fig:opl_vs_calls}.}
    \label{fig:opl_vs_calls_other}
\end{figure}

\paragraph{Single-machine scheduling problem}
In all three pairs and in all dimensions, the OS-loss variant achieved a higher attainment rate (\Cref{tab:onemachine_main}).
The attainment rates of the four methods with the OS loss are 84--100\%, whereas those of the suboptimality-loss variant decrease with the dimension; at $d=8$, Sub-PSG attains 44\%, Sub-MetaGrad(FB) 28\%, and Sub-ONS(FB) 40\% (the OS-loss variants of the same pairs attain 96\%, 96\%, and 84\%, respectively).
Aggregated over all dimensions, the attainment rates of the eight methods with the OS loss are 91--99\%, which exceed those of the five existing methods including the online versions (23--68\%) (\Cref{tab:full_onemachine}).
Comparing on the same instances, among the 75 episodes, the numbers of episodes in which $\ell^{\mathrm{OP}}$ of OS-PSG is smaller than that of Sub-PSG, Sub-ONS, and Sub-MetaGrad are 24, 58, and 36, respectively, and the numbers of episodes in which it is larger are 0, 0, and 0 (for OS-DC, the numbers of episodes with smaller $\ell^{\mathrm{OP}}$ are 24, 58, and 36, and those with larger $\ell^{\mathrm{OP}}$ are 2, 0, and 0).
In the three pairs that differ only in the loss function, the OS-loss variant loses in only 3 episodes in total (\Cref{tab:loss_only_pairs}).
The remaining episodes are ties.
As in the knapsack problem, in this setting where the actions of the agent are degraded optimal solutions, the advantage of the proposed method becomes more pronounced in higher dimensions.

\paragraph{Movie selection problem with viewing costs}
In this problem, the actions of the agent are real data, and since we do not assume the existence of a weight satisfying \cref{eq:outperformance_IOP}, we compare the values of $\ell^{\mathrm{OP}}$ (\Cref{tab:full_timecost}).
In the comparison that replaces only the loss function, the OS-loss variant won in all three pairs: 97 wins and 1 loss in the PSG pair, 90 wins and 1 loss in the MetaGrad(FB) pair, and 45 wins and 32 losses in the ONS(FB) pair (\Cref{tab:loss_only_pairs}).
Looking at each dimension, in the PSG pair and the MetaGrad(FB) pair, the median of the OS-loss variant is smaller in all dimensions (\Cref{tab:timecost_main}).
In particular, in the PSG pair the difference widens with the dimension; at $d=100$, OS-PSG attains $6.50$ whereas Sub-PSG attains $7.57$.
In the ONS(FB) pair, the median of OS-ONS(FB) is smaller for $d\le40$, the two are equal at $d=60$, and the median of Sub-ONS(FB) is smaller for $d\ge80$.
Among the seven methods in the table, the smallest median is attained by OS-DC ($0.62$, $2.30$) for $d=20,40$ and by OS-PSG ($4.05$, $5.18$, $6.50$) for $d=60,80,100$; that is, it is attained by a method with the OS loss in all dimensions.
In other words, under the same optimization method, merely replacing the suboptimality loss with the OS loss improves the prediction of solutions that outperform the agent.
Note that the online versions Sub-ONS and Sub-MetaGrad, which update state by state, have medians of $3.550$ and $3.425$ over all dimensions, respectively, which are smaller than that of OS-PSG ($3.975$) (\Cref{tab:full_timecost}).
Since the medians of the full-batch versions of the same learners, Sub-ONS(FB) and Sub-MetaGrad(FB), are $4.950$ and $4.875$, respectively, this difference is presumably due not to the use of the suboptimality loss but to the optimization method, namely online learning that updates state by state.
Among the proposed methods, the median of OS-DC[ONS] over all dimensions ($3.550$) equals that of Sub-ONS ($3.550$), and its median computation time is $31$ seconds, about one third of that of Sub-ONS ($91$ seconds).
Combining the OS loss with methods that minimize it state by state via online learning is left for future work.

\subsection{Results of all methods on the five experiments}
\label{subsec:appendix_full_results}

\paragraph{Results of the methods not included in the main text}
On the knapsack problem, comparing on the same instances, among the 250 episodes, the numbers of episodes in which $\ell^{\mathrm{OP}}$ of OS-PSG is smaller than that of Sub-PSG, Sub-ONS, and Sub-MetaGrad are 241, 227, and 243, respectively, and there is no episode in which it is larger (for OS-DC, the numbers of episodes with smaller $\ell^{\mathrm{OP}}$ are 240, 222, and 242, and those with larger $\ell^{\mathrm{OP}}$ are 1, 3, and 0).
Here, Sub-ONS and Sub-MetaGrad are the online versions that update state by state, and their attainment rates are also 9\% and 3\%, respectively, over all dimensions (\Cref{tab:full_knapsack}).
The empirical mean of the OS loss measured at the same weights is also $0$ in median for OS-PSG and OS-DC, whereas it remains at $3.9\times10^{-4}$--$2.8\times10^{-3}$ for Sub-PSG, Sub-ONS, and Sub-MetaGrad (\Cref{tab:full_knapsack}).
On the weighted multi-cover problem, unlike the movie selection problem with viewing costs (\Cref{subsec:appendix_other_results}), the medians of $\ell^{\mathrm{OP}}$ of the online versions Sub-ONS and Sub-MetaGrad also fall short of the best method with the OS loss.
Among all 13 methods, the smallest median of $\ell^{\mathrm{OP}}$ over all dimensions is attained by OS-DC[LV+CPF]+SGF ($4.000$) on Cora and by OS-DC[ONS] ($3.575$) on ogbn-arxiv, which are smaller than the best of the existing methods ($4.500$ of Sub-ONS on Cora and $3.950$ of Sub-MetaGrad on ogbn-arxiv) (\Cref{tab:full_cora_multicover,tab:full_arxiv_multicover}).
Comparing on the same instances the pairs of the proposed method and the existing method with the smallest mean of $\ell^{\mathrm{OP}}$ over all episodes (OS-DC[LV+CPF]+SGF vs.\ Sub-MetaGrad on Cora, and OS-DC[MetaGrad] vs.\ Sub-MetaGrad on ogbn-arxiv), the proposed method wins 65 times and loses 30 times on Cora, and wins 85 times and loses 11 times on ogbn-arxiv.
The median computation time of OS-PSG is $8.7$ seconds on Cora (second to $8.2$ seconds of OS-ONS(FB)) and $7.4$ seconds on ogbn-arxiv (the shortest among all methods), about one quarter of that of Sub-MetaGrad ($35.4$ seconds and $31.1$ seconds); note that the time of Sub-MetaGrad includes the solves used only for recording, which OS-PSG does not need (see the definition of ``Time'' below).
Among the three pairs differing only in the loss function, in the two full-batch pairs, although the OS-loss variant won overall, it lost in 17 episodes (MetaGrad(FB)) and 22 episodes (ONS(FB)) on Cora, so the margin is smaller than in the PSG pair (\Cref{tab:loss_only_pairs}).

\Cref{tab:full_knapsack,tab:full_onemachine,tab:full_timecost,tab:full_cora_multicover,tab:full_arxiv_multicover} show the results of all 13 methods on the five experiments.
In each table, the upper block shows existing inverse optimization (the suboptimality loss), and the lower block shows outperformance inverse optimization (the OS loss).
In the method names, the bracket of OS-DC[$\cdot$] indicates the inner solver, and +SGF indicates that subgradient finishing is added.
Each value is the median, over all dimensions and episodes, of the value at $B$ oracle calls; for $\ell^{\mathrm{OP}}$, the interquartile range (the first and third quartiles) is also shown.
$\bar{\ell}^{\optdex}$ is measured at the weight that gives the best-so-far $\ell^{\mathrm{OP}}$ (\Cref{subsec:experimental_setup}).
``Time'' is the computation time until the run stops, that is, until the number of oracle calls reaches $B$ or $\ell^{\mathrm{OP}}=0$ is attained, whichever comes first (the part after that is not included), and its mean and median are shown. It also includes the time for the solves used only for recording and not counted as oracle calls, that is, the solves of $x^*(\theta;x^n,s^n)$ for recording $\bar{\ell}^{\optdex}$ in the Sub methods and OS-DC. Hence, in comparisons with the OS methods other than OS-DC, which do not need these solves, the times of the Sub methods are overestimated.

\begin{table}[tbp]
    \centering
    \setlength{\tabcolsep}{1pt}
    \caption{Knapsack problem (oracle-call budget $B=4000$). Each value is the median, over the 250 runs ($d\in\{20,40,60,80,100\}$, 50 episodes for each $d$), of the value at $B$ oracle calls, and the attainment rate is the fraction of these 250 runs that attained $\ell^{\mathrm{OP}}=0$. For the time (in seconds), the mean is also shown.}
    \label{tab:full_knapsack}
    \begin{tabular}{l|cccccc}
        \hline
        Method & \twoline{$\ell^{\mathrm{OP}}$}{median} & \twoline{interquartile}{range} & \twoline{attainment}{rate [\%]} & \twoline{$\bar{\ell}^{\optdex}$}{median} & \twoline{time}{mean} & \twoline{time}{median} \\
        \hline
        Sub-PSG & 0.250 & [0.150,0.300] & 3 & $4.9\times10^{-4}$ & 1.3 & 1.3 \\
        Sub-ONS & 0.150 & [0.100,0.250] & 9 & $3.9\times10^{-4}$ & 2.2 & 2.2 \\
        Sub-MetaGrad & 0.350 & [0.200,0.550] & 3 & $2.8\times10^{-3}$ & 8.4 & 8.3 \\
        Sub-ONS(FB) & 0.600 & [0.300,0.800] & 0 & $6.2\times10^{-3}$ & 1.4 & 1.4 \\
        Sub-MetaGrad(FB) & 0.550 & [0.300,0.850] & 1 & $5.5\times10^{-3}$ & 2.4 & 2.4 \\
        \hline
        OS-PSG & 0.000 & [0.000,0.000] & 90 & $0$ & 0.3 & 0.2 \\
        OS-ONS(FB) & 0.700 & [0.300,1.000] & 7 & $8.9\times10^{-3}$ & 0.7 & 0.7 \\
        OS-MetaGrad(FB) & 0.100 & [0.000,0.450] & 41 & $4.8\times10^{-4}$ & 0.9 & 0.8 \\
        OS-DC{[}ONS{]} & 0.100 & [0.000,0.250] & 26 & $5.2\times10^{-4}$ & 1.3 & 1.3 \\
        OS-DC{[}MetaGrad{]} & 0.050 & [0.000,0.150] & 44 & $1.1\times10^{-4}$ & 1.8 & 1.5 \\
        OS-DC{[}CP{]}+SGF & 0.650 & [0.213,1.050] & 12 & $7.2\times10^{-3}$ & 1.7 & 1.7 \\
        OS-DC{[}LV+CPF{]}+SGF & 0.000 & [0.000,0.000] & 85 & $0$ & 0.9 & 0.7 \\
        OS-DC{[}PB+CPF{]}+SGF & 0.000 & [0.000,0.000] & 86 & $0$ & 0.7 & 0.6 \\
        \hline
    \end{tabular}
\end{table}

\begin{table}[tbp]
    \centering
    \setlength{\tabcolsep}{1pt}
    \caption{Single-machine scheduling problem (oracle-call budget $B=2000$). Each value is the median, over the 75 runs ($d\in\{4,6,8\}$, 25 episodes for each $d$), of the value at $B$ oracle calls, and the attainment rate is the fraction of these 75 runs that attained $\ell^{\mathrm{OP}}=0$. For the time (in seconds), the mean is also shown.}
    \label{tab:full_onemachine}
    \begin{tabular}{l|cccccc}
        \hline
        Method & \twoline{$\ell^{\mathrm{OP}}$}{median} & \twoline{interquartile}{range} & \twoline{attainment}{rate [\%]} & \twoline{$\bar{\ell}^{\optdex}$}{median} & \twoline{time}{mean} & \twoline{time}{median} \\
        \hline
        Sub-PSG & 0.000 & [0.000,0.150] & 68 & $0$ & 17.7 & 2.8 \\
        Sub-ONS & 0.500 & [0.100,1.350] & 23 & $1.5\times10^{-2}$ & 32.6 & 17.3 \\
        Sub-MetaGrad & 0.000 & [0.000,0.300] & 52 & $0$ & 24.9 & 11.8 \\
        Sub-ONS(FB) & 0.000 & [0.000,0.300] & 59 & $0$ & 24.6 & 13.7 \\
        Sub-MetaGrad(FB) & 0.000 & [0.000,0.200] & 63 & $0$ & 19.7 & 1.2 \\
        \hline
        OS-PSG & 0.000 & [0.000,0.000] & 99 & $0$ & 6.3 & 1.1 \\
        OS-ONS(FB) & 0.000 & [0.000,0.000] & 95 & $0$ & 10.3 & 1.2 \\
        OS-MetaGrad(FB) & 0.000 & [0.000,0.000] & 99 & $0$ & 6.3 & 1.5 \\
        OS-DC{[}ONS{]} & 0.000 & [0.000,0.000] & 91 & $0$ & 19.6 & 1.9 \\
        OS-DC{[}MetaGrad{]} & 0.000 & [0.000,0.000] & 93 & $0$ & 18.7 & 3.3 \\
        OS-DC{[}CP{]}+SGF & 0.000 & [0.000,0.000] & 97 & $0$ & 14.4 & 3.4 \\
        OS-DC{[}LV+CPF{]}+SGF & 0.000 & [0.000,0.000] & 95 & $0$ & 9.9 & 1.7 \\
        OS-DC{[}PB+CPF{]}+SGF & 0.000 & [0.000,0.000] & 97 & $0$ & 10.9 & 2.2 \\
        \hline
    \end{tabular}
\end{table}

\begin{table}[tbp]
    \centering
    \setlength{\tabcolsep}{1pt}
    \caption{Movie selection problem with viewing costs (oracle-call budget $B=4000$). Each value is the median, over the 100 runs ($d\in\{20,40,60,80,100\}$, 20 episodes for each $d$), of the value at $B$ oracle calls, and the attainment rate is the fraction of these 100 runs that attained $\ell^{\mathrm{OP}}=0$. For the time (in seconds), the mean is also shown.}
    \label{tab:full_timecost}
    \begin{tabular}{l|cccccc}
        \hline
        Method & \twoline{$\ell^{\mathrm{OP}}$}{median} & \twoline{interquartile}{range} & \twoline{attainment}{rate [\%]} & \twoline{$\bar{\ell}^{\optdex}$}{median} & \twoline{time}{mean} & \twoline{time}{median} \\
        \hline
        Sub-PSG & 4.600 & [2.438,6.000] & 0 & $3.1\times10^{-2}$ & 106.7 & 95.3 \\
        Sub-ONS & 3.550 & [1.950,4.612] & 0 & $2.0\times10^{-2}$ & 98.9 & 91.2 \\
        Sub-MetaGrad & 3.425 & [1.900,4.600] & 0 & $2.4\times10^{-2}$ & 107.4 & 97.5 \\
        Sub-ONS(FB) & 4.950 & [2.950,6.312] & 0 & $5.0\times10^{-2}$ & 39.6 & 40.4 \\
        Sub-MetaGrad(FB) & 4.875 & [2.913,6.300] & 0 & $5.4\times10^{-2}$ & 76.1 & 62.5 \\
        \hline
        OS-PSG & 3.975 & [2.025,5.362] & 0 & $1.2\times10^{-2}$ & 75.7 & 81.1 \\
        OS-ONS(FB) & 4.950 & [2.913,6.487] & 0 & $4.9\times10^{-2}$ & 7.0 & 7.1 \\
        OS-MetaGrad(FB) & 4.500 & [2.600,5.850] & 0 & $7.5\times10^{-2}$ & 18.0 & 17.4 \\
        OS-DC{[}ONS{]} & 3.550 & [2.000,5.013] & 0 & $3.9\times10^{-2}$ & 21.6 & 23.0 \\
        OS-DC{[}MetaGrad{]} & 3.900 & [2.188,5.312] & 0 & $5.5\times10^{-2}$ & 23.2 & 23.6 \\
        OS-DC{[}CP{]}+SGF & 4.750 & [2.700,6.475] & 0 & $4.3\times10^{-2}$ & 16.2 & 16.0 \\
        OS-DC{[}LV+CPF{]}+SGF & 4.150 & [1.938,5.463] & 0 & $3.4\times10^{-3}$ & 115.0 & 110.3 \\
        OS-DC{[}PB+CPF{]}+SGF & 4.200 & [2.050,5.450] & 0 & $1.0\times10^{-2}$ & 105.2 & 102.9 \\
        \hline
    \end{tabular}
\end{table}

\begin{table}[tbp]
    \centering
    \setlength{\tabcolsep}{1pt}
    \caption{Weighted multi-cover problem (Cora) (oracle-call budget $B=4000$). Each value is the median, over the 100 runs ($d\in\{20,40,60,80,100\}$, 20 episodes for each $d$), of the value at $B$ oracle calls, and the attainment rate is the fraction of these 100 runs that attained $\ell^{\mathrm{OP}}=0$. For the time (in seconds), the mean is also shown.}
    \label{tab:full_cora_multicover}
    \begin{tabular}{l|cccccc}
        \hline
        Method & \twoline{$\ell^{\mathrm{OP}}$}{median} & \twoline{interquartile}{range} & \twoline{attainment}{rate [\%]} & \twoline{$\bar{\ell}^{\optdex}$}{median} & \twoline{time}{mean} & \twoline{time}{median} \\
        \hline
        Sub-PSG & 4.800 & [2.787,6.338] & 0 & $7.0\times10^{-2}$ & 18.9 & 16.9 \\
        Sub-ONS & 4.500 & [2.688,6.100] & 0 & $7.9\times10^{-2}$ & 17.9 & 16.9 \\
        Sub-MetaGrad & 4.600 & [2.537,6.075] & 0 & $9.1\times10^{-2}$ & 40.8 & 35.4 \\
        Sub-ONS(FB) & 5.150 & [2.888,6.812] & 0 & $1.1\times10^{-1}$ & 17.7 & 16.2 \\
        Sub-MetaGrad(FB) & 5.150 & [2.975,7.075] & 0 & $1.1\times10^{-1}$ & 21.8 & 20.3 \\
        \hline
        OS-PSG & 4.375 & [2.362,5.900] & 0 & $6.3\times10^{-2}$ & 8.9 & 8.7 \\
        OS-ONS(FB) & 5.150 & [2.650,6.963] & 0 & $1.4\times10^{-1}$ & 8.8 & 8.2 \\
        OS-MetaGrad(FB) & 5.125 & [2.712,6.975] & 0 & $1.1\times10^{-1}$ & 10.5 & 10.1 \\
        OS-DC{[}ONS{]} & 4.350 & [2.337,5.663] & 0 & $7.9\times10^{-2}$ & 17.2 & 16.4 \\
        OS-DC{[}MetaGrad{]} & 4.100 & [2.375,5.713] & 0 & $7.1\times10^{-2}$ & 20.1 & 19.4 \\
        OS-DC{[}CP{]}+SGF & 4.775 & [2.300,7.075] & 0 & $8.3\times10^{-2}$ & 13.5 & 12.3 \\
        OS-DC{[}LV+CPF{]}+SGF & 4.000 & [1.550,6.100] & 0 & $4.6\times10^{-2}$ & 15.3 & 13.7 \\
        OS-DC{[}PB+CPF{]}+SGF & 4.400 & [2.287,6.150] & 0 & $5.1\times10^{-2}$ & 15.1 & 13.7 \\
        \hline
    \end{tabular}
\end{table}

\begin{table}[tbp]
    \centering
    \setlength{\tabcolsep}{1pt}
    \caption{Weighted multi-cover problem (ogbn-arxiv) (oracle-call budget $B=4000$). Each value is the median, over the 100 runs ($d\in\{20,40,60,80,100\}$, 20 episodes for each $d$), of the value at $B$ oracle calls, and the attainment rate is the fraction of these 100 runs that attained $\ell^{\mathrm{OP}}=0$. For the time (in seconds), the mean is also shown.}
    \label{tab:full_arxiv_multicover}
    \begin{tabular}{l|cccccc}
        \hline
        Method & \twoline{$\ell^{\mathrm{OP}}$}{median} & \twoline{interquartile}{range} & \twoline{attainment}{rate [\%]} & \twoline{$\bar{\ell}^{\optdex}$}{median} & \twoline{time}{mean} & \twoline{time}{median} \\
        \hline
        Sub-PSG & 4.650 & [2.612,7.375] & 0 & $8.4\times10^{-2}$ & 17.2 & 14.2 \\
        Sub-ONS & 4.475 & [2.663,6.650] & 0 & $3.4\times10^{-2}$ & 16.3 & 14.6 \\
        Sub-MetaGrad & 3.950 & [2.012,6.013] & 0 & $7.2\times10^{-2}$ & 36.5 & 31.1 \\
        Sub-ONS(FB) & 4.650 & [2.650,7.312] & 0 & $1.1\times10^{-1}$ & 18.0 & 14.1 \\
        Sub-MetaGrad(FB) & 4.650 & [2.650,7.537] & 0 & $1.1\times10^{-1}$ & 21.4 & 17.3 \\
        \hline
        OS-PSG & 4.200 & [1.738,7.088] & 0 & $4.7\times10^{-2}$ & 8.5 & 7.4 \\
        OS-ONS(FB) & 4.625 & [2.300,7.188] & 0 & $1.2\times10^{-1}$ & 9.1 & 7.5 \\
        OS-MetaGrad(FB) & 4.625 & [2.325,7.263] & 0 & $1.2\times10^{-1}$ & 10.4 & 8.6 \\
        OS-DC{[}ONS{]} & 3.575 & [1.700,5.812] & 0 & $5.9\times10^{-2}$ & 17.5 & 14.4 \\
        OS-DC{[}MetaGrad{]} & 3.600 & [1.688,5.562] & 0 & $6.9\times10^{-2}$ & 19.4 & 16.8 \\
        OS-DC{[}CP{]}+SGF & 4.550 & [2.438,7.500] & 0 & $8.2\times10^{-2}$ & 10.0 & 8.6 \\
        OS-DC{[}LV+CPF{]}+SGF & 4.325 & [1.887,7.050] & 0 & $1.2\times10^{-4}$ & 11.7 & 9.8 \\
        OS-DC{[}PB+CPF{]}+SGF & 4.225 & [1.938,7.125] & 0 & $1.4\times10^{-2}$ & 10.8 & 9.2 \\
        \hline
    \end{tabular}
\end{table}

{\sloppy
On the knapsack problem, OS-PSG (attainment rate 90\%), OS-DC[PB+CPF]+SGF (86\%), and OS-DC[LV+CPF]+SGF (85\%) attained a median of $\ell^{\mathrm{OP}}$ equal to $0$ and outperform the five existing methods (attainment rates 0--9\%).
The attainment rates of OS-DC[CP]+SGF and OS-ONS(FB) are 12\% and 7\%, respectively.
On the single-machine scheduling problem, the attainment rates of the eight methods with the OS loss are 91--99\%, which exceed those of the five existing methods (23--68\%).
On the movie selection problem with viewing costs, the attainment rates of all methods are 0\%; the median of $\ell^{\mathrm{OP}}$ is the smallest for Sub-MetaGrad ($3.425$), followed by Sub-ONS and OS-DC[ONS] (both $3.550$).
On this problem, OS-DC[LV+CPF]+SGF, which has the smallest empirical mean of the OS loss (median $3.4\times10^{-3}$), has a median of $\ell^{\mathrm{OP}}$ equal to $4.150$, which is larger than that of Sub-ONS (median empirical mean of the OS loss $2.0\times10^{-2}$).
This is why we use $\ell^{\mathrm{OP}}$ as the main evaluation metric in addition to the empirical mean of the OS loss (\Cref{sec:performance_metric}).
On the weighted multi-cover problem, the attainment rates of all methods are 0\%, and the smallest median of $\ell^{\mathrm{OP}}$ is attained by OS-DC[LV+CPF]+SGF ($4.000$) on Cora and by OS-DC[ONS] ($3.575$) on ogbn-arxiv.
On ogbn-arxiv, OS-DC[LV+CPF]+SGF has the smallest median empirical mean of the OS loss among all methods, $1.2\times10^{-4}$, but its median of $\ell^{\mathrm{OP}}$ is $4.325$, which is larger than that of OS-DC[ONS].\par}

} 

\end{document}